\documentclass[1p]{elsarticle}

\usepackage{hyperref}
\usepackage{microtype}
\usepackage{graphicx}
\usepackage{subfigure}
\usepackage{booktabs} 
\usepackage{amsmath}
\usepackage{amssymb}
\usepackage{amsthm}
\usepackage{colortbl}
\usepackage{dsfont}
\newtheorem{theorem}{Proposition}
\newtheorem{assumption}{Assumption}

\journal{Neural Networks (accepted: April 2021)}

\begin{document}

\begin{frontmatter}

\title{FastGAE: Scalable Graph Autoencoders with Stochastic Subgraph Decoding}

\author[1,2]{Guillaume Salha}
\author[1]{Romain Hennequin}
\author[2]{Jean-Baptiste Remy}

\author[1]{\\Manuel Moussallam}
\author[2]{Michalis Vazirgiannis}

\address[1]{Deezer Research\footnote{Contact: research@deezer.com}, Paris, France}
\address[2]{LIX, \'{E}cole Polytechnique, Palaiseau, France }

\begin{abstract}
Graph autoencoders (AE) and variational autoencoders (VAE) are powerful node embedding methods, but suffer from scalability issues. In this paper, we introduce FastGAE, a general framework to scale graph AE and VAE to large graphs with millions of nodes and edges. Our strategy, based on an effective stochastic subgraph decoding scheme, significantly speeds up the training of graph AE and VAE while preserving or even improving performances. We demonstrate the effectiveness of FastGAE on various real-world graphs, outperforming the few existing approaches to scale graph AE and VAE by a wide margin.
\end{abstract}


\end{frontmatter}

\section{Introduction}
\label{introduction}

Graph structures efficiently represent relationships and interactions among entities. Social networks, molecules, citations of scientific publications and web pages constitute some of the most famous examples of data usually represented as graphs, i.e. as \textit{nodes} connected via \textit{edges}. Extracting information from these connections is essential to address numerous graph-based learning problems, ranging from link prediction to influence maximization and node clustering. In this direction, several significant improvements were recently achieved by methods leveraging \textit{node embeddings} \cite{hamilton2017representation,wu2019comprehensive}. Instead of relying on hand made features, these methods aim at automatically \textit{learning} low-dimensional vector space representations of nodes capturing relevant information from the graph, such as structural proximity, notably by using graph neural networks \cite{kipf2016-1,hamilton2017inductive}, matrix factorization \cite{cao2015grarep} or random walk processes \cite{perozzi2014deepwalk,tang2015line,grover2016node2vec}. 

In particular, during the last few years, \textit{graph autoencoders} (graph AE)  and \textit{graph variational autoencoders} (graph VAE) emerged as two of the most promising and powerful \textit{unsupervised} node embedding methods \cite{tian2014learning,wang2016structural,kipf2016-2}. Introduced as extensions of standard AE \cite{Rumelhart1986,baldi2012autoencoders} and VAE \cite{kingma2013vae} to graph structures, they involve the combination of two stacked models. First, an \textit{encoder}, typically based on a graph neural network (GNN), maps the nodes into an embedding space; then, a \textit{decoder} tries to reconstruct the original graph structure from the vector representations. Both models are jointly trained to optimize the quality of the reconstruction from the embedding space, in an unsupervised fashion with (for VAE) or without (for AE) a probabilistic approach. Recently, graph AE and VAE have been widely adopted to tackle challenging graph-based problems, such as node clustering \cite{wang2017mgae,pan2018arga,semiimplicit2019,salha2019-1,aaai20}, graph generation \cite{molecule3,molecule1,molecule2,simonovsky2018graphvae,samanta2019nevae} and link prediction \cite{kipf2016-2,berg2018matrixcomp,tran2018multi,grover2019graphite,salha2019-2}, often reaching competitive results w.r.t. popular unsupervised baselines such as DeepWalk \cite{perozzi2014deepwalk} and node2vec \cite{grover2016node2vec}.

Nonetheless, graph AE, VAE and their extensions suffer from \textit{scalability} issues. As we explain in Section 2, they mainly result from the costly decoding operations involved in the graph reconstruction. While several recent works provided strategies to scale GNN models\footnote{We nonetheless point out that these works were done \textit{out} of the graph AE and VAE unsupervised frameworks that we consider in this paper. They aimed at scaling GNN models that were trained in a supervised or semi-supervised manner.} i.e. \textit{encoders} \cite{hamilton2017inductive,chen2018stochastic,chen2018fastgcn,chiang2019cluster}, the question of \textit{how to overcome complex decoders} in graph AE and VAE remains open, preventing them from scaling. As a consequence, existing graph AE and VAE have been mainly applied to relatively small graphs, with up to a few thousand nodes. As larger graphs are ubiquitous, we propose to address these scalability concerns in this paper, making the following contributions:

\begin{itemize}
    \item We introduce FastGAE, a general framework to scale graph AE and VAE models to large graphs with millions of nodes and edges. We leverage graph mining-based sampling schemes and an effective subgraph decoding strategy to significantly lower the computational complexity of graph AE and VAE models, while preserving or even improving their performances.
    \item We propose an in-depth theoretical and experimental analysis of our method. We demonstrate its empirical effectiveness on seven graphs, with various characteristics and natures, and with up to millions of nodes and edges.
    \item We publicly release the code of FastGAE\footnote{Our code is available at: \href{https://github.com/deezer/fastgae}{https://github.com/deezer/fastgae}}, to ensure reproducibility and future usages.

\end{itemize}
The remainder of this paper is organized as follows. After reviewing key notions on standard graph AE and VAE and on their complexity in Section 2, we present and analyze FastGAE, our scalable framework, in Section 3. We report our experimental evaluation and a discussion of our results in Section 4, and we conclude in Section~5.

\section{Preliminaries}

Throughout most of this paper, we consider an undirected graph $\mathcal{G} = (\mathcal{V},\mathcal{E})$ with $|\mathcal{V}| = n$ nodes and $|\mathcal{E}| = m$ edges ; extensions of our work to directed graphs are nonetheless discussed in Section 5. We denote by $A$ the binary and symmetric adjacency matrix of $\mathcal{G}$, and by $X$ an $n \times f$ matrix stacking up $f$-dimensional node-level features vectors, one for each node of $\mathcal{G}$. For featureless graphs, we simply assume $X = I_n$, where $I_n$ is the $n\times n$ identity matrix.

\subsection{Graph Autoencoders}

Graph autoencoders (AE) \cite{tian2014learning,wang2016structural,kipf2016-2} involve the combination of two stacked models: an \textit{encoder} and a \textit{decoder}.
\subsubsection{Encoder}

First, the \textit{encoder model} aims at learning an $n \times d$ matrix $Z$, whose rows $z_i$ are the $d$-dimensional embedding vectors of each node $i \in \mathcal{V}$, with $d \ll n$. This matrix is usually obtained through a \textit{graph neural network} (GNN) \cite{kipf2016-1,bruna2013spectral,defferrard2016} processing $A$ and $X$. More precisely, most recent variants of graph AE implement multi-layer \textit{graph convolutional networks} (GCN) encoders \cite{kipf2016-1,kipf2016-2,micheli2009neural}. In a $L$-layer GCN (with $L \geq 2$), with input layer $H^{(0)} = X$ and output layer $H^{(L)} = Z$ i.e. the embedding vectors, we have:
\begin{align}
&H^{(l)} = \text{ReLU} (\tilde{A} H^{(l-1)} W^{(l-1)}), \hspace{5pt} \text{for } l \in \{1,...,L-1\} \\
&H^{(L)} = \tilde{A} H^{(L-1)} W^{(L-1)}.
\nonumber
\end{align}
In the above equations, $\tilde{A} = D^{-1/2}(A + I_n) D^{-1/2}$ is the symmetric normalization of $A$, with $D$ denoting the diagonal degree matrix of $A + I_n$. In a nutshell, at each layer $l$ we compute a representation for each node, by averaging the representations from layer $l-1$ of its direct neighbors (that, from layer 2, have already aggregated information from \textit{their own} neighbors) and of itself (thus the $A + I_n$ in the normalization), together with a ReLU activation: $\text{ReLU}(x) = \max(x,0)$. $W^{(0)},...,W^{(L-1)}$ are weight matrices, to tune as subsequently detailed in Section 2.1.3.
\subsubsection{Decoder}

Then, the \textit{decoder model} aims at reconstructing the graph from the embedding. \citet{kipf2016-2} and most subsequent works on graph AE models implement a simple \textit{inner-product} decoder. The reconstructed adjacency matrix is then: 
\begin{align}
\hat{A} = \sigma(ZZ^T),
\end{align}
with $Z = \text{GCN}(A,X)$, and with $\sigma(\cdot)$ the sigmoid function: $\sigma(x) = 1/(1 + e^{-x})$. In other words, we have $\hat{A}_{ij} = \sigma (z^T_i z_j)$ for all $(i,j) \in \mathcal{V}\times\mathcal{V}$ i.e. nodes with large inner-products in the embedding are more likely to be connected in the graph according to the model. While we will also consider this decoder in our experiments for simplicity and consistency with previous works, we  point out the existence of more sophisticated decoders in recent research, such as the asymmetric
decoder of \citet{salha2019-2}, the
reverse message passing schemes of \citet{grover2019graphite}
and the decoder of \citet{aaai20} reconstructing node triads.

\subsubsection{Learning}

The objective of graph AE is to learn low-dimensional vector representations that ensure a good reconstruction $\hat{A}$ from the decoder. To achieve this, GCN weights are usually tuned by iteratively minimizing, by \textit{gradient descent} \cite{goodfellow2016deep}, a \textit{reconstruction loss} capturing the similarity between $A$ and $\hat{A}$. In the graph AE framework, this loss is usually formulated as a cross entropy loss \cite{kipf2016-2} i.e.:
\begin{align}
\mathcal{L} = \frac{-1}{n^2}\sum_{(i,j) \in \mathcal{V}^2} \Big[A_{ij}\log(\hat{A}_{ij}) + (1-A_{ij})\log(1 - \hat{A}_{ij})\Big].
\end{align}
Also, in practice, the pairs with $A_{ij} = 1$ are often re-weighted in the loss, when dealing with a sparse adjacency matrix $A$ \cite{kipf2016-2,salha2020simple}.

\subsection{Graph Variational Autoencoders}

\citet{kipf2016-2} also introduced graph extensions of \textit{variational autoencoders}~(VAE), a model originally introduced for general data by \citet{kingma2013vae}.

\subsubsection{Encoder}
In a graph VAE, we establish a probabilistic model on $A$, involving a $d$-dimensional latent variable $z_i$ for each node $i \in \mathcal{V}$, corresponding to its embedding vector. \citet{kipf2016-2} propose the following mean-field inference model as \textit{encoder}:
\begin{align}
q(Z|A,X) = \prod_{i=1}^n q(z_i|A,X),
\end{align}
with $q(z_i|A,X)$ corresponding to a $\mathcal{N}(\mu_i, \text{diag}(\sigma_i^2))$ distribution. They leverage two GCNs to learn the $d$-dimensional Gaussian mean and variance vectors $\mu_i$ and $\sigma_i$ for each node. In a nutshell, $\mu = \text{GCN}_{\mu}(A,X)$, with $\mu$ the matrix of mean vectors $\mu_i$ ; also, $\log \sigma = \text{GCN}_{\sigma}(A,X)$. The embedding vectors $z_i$ are samples subsequently drawn from these distributions.
\subsubsection{Decoder}
Then, from these embedding vectors, a \textit{generative model} aims at \textit{decoding} $A$ using, as for graph AE, inner-products together with sigmoid activations. We have:
\begin{align}\hat{A}_{ij} = p(A_{ij} =1|z_i, z_j) = \sigma(z_i^Tz_j).
\end{align}
Then:
\begin{align}p(A|Z) = \prod\limits_{i,j=1}^n  p(A_{ij}|z_i, z_j).\end{align}
\subsubsection{Learning}
\citet{kipf2016-2} propose to iteratively maximize a lower bound  of the model's likelihood (ELBO) \cite{kingma2013vae} by gradient descent w.r.t. the GCNs' weights:
\begin{align}\mathcal{L}^{\text{ELBO}} = \mathbb{E}_{q(Z|A,X)} \Big[\log
p(A|Z)\Big] - \mathcal{D}_{KL}\Big(q(Z|A,X)||p(Z)\Big).\end{align}
$\mathcal{D}_{KL}(\cdot||\cdot)$ denotes the Kullback-Leibler divergence \cite{kullback1951information}, and $p(Z)$ corresponds to an initial standard Gaussian prior on the distribution of latent vectors. We refer to \cite{kipf2016-2,kingma2013vae} for more technical details on computations.

\subsection{On Complexity and Scalability}

GCN models have become popular \textit{encoders} for graph AE and VAE, thanks to their relative simplicity w.r.t. other GNN architectures \cite{bruna2013spectral,defferrard2016}. The cost of evaluating each layer of a GCN encoder is \textit{linear} w.r.t. the number of edges $m$ \cite{kipf2016-1}, which can also be improved by instead encoding nodes with a FastGCN \cite{chen2018fastgcn}, with a Cluster-GCN \cite{chiang2019cluster}, or by using simple graph convolutions \cite{wu2019simplifying} or other stochastic strategies \cite{hamilton2017inductive,chen2018stochastic,ying2018graph,zeng2020graphsaint}. 

However, the inner-product \textit{decoders} of graph AE and VAE both involve the multiplication of the dense matrices $Z$ and $Z^T$ at each training iteration. It suffers from a quadratic $O(n^2)$ complexity, as the aforementioned alternative decoders that all require inner-products or Euclidean distances computations. Storing $n \times n$ dense matrices $\hat{A}$ can also lead to memory issues for large $n$. As a consequence, the recent aforementioned efforts to scale GCNs (that were achieved in a supervised setting, and out of the wider graph AE and VAE frameworks studied in this paper, where GCNs are only a building block) are \textit{not} sufficient to scale graph AE and VAE. These models still suffer from a quadratic time complexity due to their costly decoding operations, and therefore from scalability issues.

As a result, graph AE, VAE and their extensions were usually applied to relatively small graphs with up to a few thousand nodes and edges. We acknowledge that \citet{kipf2016-2,grover2019graphite} and \citet{salha2020simple} all very briefly mentioned \textit{sampling} ideas as possible extensions, but without further investigation. We will later observe that a direct uniform sampling of nodes is often sub-optimal. Recently, \citet{aaai20} incorporated more elaborated mini-batch sampling ideas in their work on graph AE and VAE. More specifically, authors proposed to reconstruct \textit{triads} of nodes in their decoder instead of simple inner-products. The total number of triads in a graph is ${n}\choose{3}$ which is very large, but authors proposed to sample a smaller number of triads to make their model tractable. Nonetheless, as the objective of their work was more to improve the performances of graph AE and VAE than to provide scalable models, they did not report running times nor experiments on large graphs.

In a wider analysis on scalable graph AE and VAE, \citet{salha2019-1} proposed to speed up computations by training the AE/VAE only on a smaller version of the graph, namely on the $k$-core of the graph \cite{malliaros2019}, then by propagating representations to other nodes via simpler but faster heuristics. While they did provide experiments on larger graphs, their performances tend to significantly deteriorate for smaller cores i.e. for faster models. To sum up, the question of how to effectively scale graph AE and VAE remains unsatisfactorily addressed.
\section{Scaling Graph AE and VAE with FastGAE}

In this Section, we introduce our proposed framework to scale graph AE and VAE. We refer to it as \textit{FastGAE}, and as \textit{variational FastGAE} when applied to graph VAE.

\subsection{Encoding the Entire Graph...}

As explained in Section 2.3, GCN models \cite{kipf2016-1} and their scalable extensions \cite{chen2018fastgcn,chiang2019cluster,wu2019simplifying,zeng2020graphsaint} can effectively process large graphs. Therefore, in our FastGAE framework, we rely on these models to \textit{encode all the nodes} into the embedding space. More precisely, in the following experiments, we implement \textit{standard GCN encoders} for the sake of simplicity and for an easier comparison to existing graph AE and VAE architectures. This design choice is made \textit{without loss of generality}, the FastGAE framework being valid for \textit{any} other encoder producing an embedding matrix $Z$.

\subsection{...But Decoding Stochastic Subgraphs}

However, while computing node embedding vectors through a \textit{forward} GCN pass is fast, \textit{tuning the weights} of this encoder in the graph AE and VAE settings requires the reconstruction of the entire matrix $\hat{A}$ at each training iteration which, as detailed in Section 2.3, suffers from a quadratic complexity and is intractable for large graphs.

\subsubsection{Subgraph Decoding}

 To overcome this issue, we propose to \textit{approximate reconstruction losses}, by computing their values only from \textit{wisely selected random subparts} of the original graph. More precisely, at each training iteration, we aim at decoding a different sampled subgraph of $\mathcal{G}$ with $n_{(S)}$ nodes, with $n_{(S)} < n$ being a fixed parameter. Let $\mathcal{G}_{(S)} =
(\mathcal{V}_{(S)},\mathcal{E}_{(S)})$ be such sampled subgraph, with $\mathcal{V}_{(S)} \subset\mathcal{V}$, $|\mathcal{V}_{(S)}| = n_{(S)}$, and with $\mathcal{E}_{(S)}$ denoting the subset of edges connecting the nodes in $\mathcal{V}_{(S)}$. Instead of reconstructing the $n \times n$ matrix $\hat{A}$, we propose to reconstruct the smaller $n_{(S)} \times n_{(S)}$ matrix $\hat{A}_{(S)}$ with:
\begin{align}\hat{A}_{(S)~ij} = \sigma(z^T_i z_j), \hspace{5pt} \forall (i,j) \in \mathcal{V}^2_{(S)},
\end{align}
and to only learn from the quality of $\hat{A}_{(S)}$ w.r.t. its ground truth counterpart $A_{(S)}$, as measured by a cross entropy loss for AE, or an ELBO loss for VAE. We propose to use the resulting approximate loss for gradients computations and GCN weights updates by gradient descent. We draw \textit{a different subgraph $\mathcal{G}_{(S)}$ at each training iteration}, using the sampling methods detailed next.

\subsubsection{(Naive) Uniform Node Sampling}

A very simple way to obtain such subgraphs would consist in \textit{uniformly sampling} $n_{(S)}$ nodes from the set $\mathcal{V}$ at each training iteration. However, with such strategy, there is no guarantee that the most important links (or absence of links) from the original graph structure will be preserved in the drawn subgraphs to reconstruct during the training phase. As we will experimentally show in Section 4, this usually significantly impacts the quality of the final node embedding, leading to underperforming performances on downstream evaluation tasks. As a consequence, in the following sections, we propose and study more refined strategies, aiming at \textit{leveraging the graph structure} to obtain a more effective sampling.

\subsubsection{Node Sampling with Graph Mining}

We propose to consider alternative sampling methods, that increase the probability of including some particular nodes in the drawn subgraph w.r.t. some others. Let $f: \mathcal{V} \rightarrow \mathbb{R}^+$ denote some measure of the relative \textit{importance} of nodes in the graph, obtained through graph mining methods. Assuming such function is available, we draw inspiration from word sampling in natural language processing \cite{w2v1,w2v2} and propose to set the probability to pick each node $i \in \mathcal{V}$ as the first element of $\mathcal{V}_{(S)}$ as:
\begin{align}
p_i = \frac{f(i)^{\alpha}}{\sum\limits_{j \in \mathcal{V} } (f(j)^{\alpha})},\end{align}
with $\alpha \in \mathbb{R}^+$. Then, assuming we sample $n_{(S)}$ \textit{distinct} nodes \textit{without replacement}, each remaining node $i \in \mathcal{V} \setminus \mathcal{V}_{(S)}$ has a probability $p_i / \sum_{j \notin \mathcal{V}_{(S)}} p_j$ to be picked as the second element of $\mathcal{V}_{(S)}$, and so on until $|\mathcal{V}_{(S)}| = n_{(S)}$. The previous division is a simple normalization to ensure $\sum_{j \notin \mathcal{V}_{(S)}} p_j = 1$ at each sampling step. Alternatively, one could also sample $n_{(S)}$ nodes \textit{with replacement}: it simplifies computations, as sampling probabilities are then independent of previous draws and remain fixed to $p_i$, but a node could then be drawn several times. We stress out that, in our implementation, both variants return very similar results. We later adopt the former.

In a nutshell, \textit{important nodes according to $f$ are more likely to be selected for decoding}, and the hyperparameter $\alpha$ helps sharpening (for $\alpha > 1$) or smoothing (for $\alpha < 1$) the distribution. Setting $\alpha = 0$ leads to the aforementioned uniform node sampling. In our experiments, we will consider and evaluate two importance measures $f$ from graph mining:
\begin{itemize}
    \item the \textit{degree} of each node, which is simply the number of direct connections of each node: $f(i) = \sum_{j \in \mathcal{V}} A_{ij}$.
    \item the \textit{core number} of each node: $f(i) = C(i)$. The $k$-core version of a graph is its largest subgraph for which every node has a degree higher or equal to $k$ \textit{within this subgraph}. The core number $C(i)$ of a node $i$ corresponds to the largest value of $k$ for which $i$ is in the $k$-core. Core decomposition has been widely used over the past years to quantify the significance of nodes and extract representative subgraphs (see e.g. \citet{malliaros2019} for a review). They constitute a more \textit{global} importance measure than the \textit{local} node degree.
\end{itemize}
Besides their popularity and their complementarity, we also chose to focus on these two metrics for \textit{computational efficiency}. Indeed, contrary to other potential importance metrics based on influence maximization \cite{kempe2003maximizing}, random walks \cite{leskovec2006sampling} or centrality measures \cite{newman}, both can be evaluated in a linear $O(m)$ running time \cite{batagelj2003m}. As we will empirically check in Section 4, this leads to \textit{fast and scalable computations of probability distributions}, which is crucial for our FastGAE framework whose primary objective is scalability. We refer the interested reader to the work of \citet{leskovec2006sampling}, \citet{hu2013survey} and \citet{chiericetti2016sampling} for a broader overview of other existing graph sampling methods.

\subsection{Theoretical Considerations}

In the Section 3.3.1, we provide a global overview  of some theoretical analyses that we subsequently further develop in Sections 3.3.2 and 3.3.3. For readability reasons, we will report the proofs of all propositions in the appendices.

\subsubsection{Overview of Theoretical Considerations}

\paragraph{On Approximate Losses} In the case of degree and core-based sampling strategies, some node pairs from the graph are more likely to appear in subgraphs than others. The probability to draw a node $i$, or an edge incident to $i$, increases with $p_i$ and with $f(i)$ for $\alpha > 0$. As a consequence, at each gradient descent iteration, the approximate loss (say $\mathcal{L}^{\text{\tiny FastGAE}}$) is \textit{biased} w.r.t. the standard graph AE or VAE loss that would have been computed on $\mathcal{G}$ (say $\mathcal{L}$), i.e. $\mathbb{E}(\mathcal{L}^{\text{\tiny FastGAE}}) \neq \mathcal{L}$ in general. For completeness, in Propositions 1, 2 and 3 of the upcoming Section 3.3.2, \textit{we provide a theoretical analysis, in which we fully explicit the expected loss} $\mathbb{E}(\mathcal{L}^{\text{\tiny FastGAE}})$ that we actually stochastically optimize in FastGAE, as well as the formal probabilities to sample a given node or node pair at each training iteration.  Moreover, we will show in Section 4 that, despite such bias, optimizing this alternative loss does not deteriorate the quality of node embeddings. On the contrary, we will provide insights exhibiting the fact that re-weighting node pairs from high degree/core nodes can actually be beneficial.

\paragraph{On the Selection of $n_{(S)}$} When selecting $n_{(S)}$, one faces a performance/speed trade-off: reconstructing very small subgraphs will speed up training but, as we later verify, this might also deteriorate performances. While we claimed in the previous paragraph that stochastically minimizing $\mathbb{E}(\mathcal{L}^{\text{\tiny FastGAE}})$ instead of $\mathcal{L}$ might be beneficial, we also acknowledge that, for small values of $n_{(S)}$, the actual loss $\mathcal{L}^{\text{\tiny FastGAE}}$ computed at a given training iteration can significantly deviate from its expectation. In this paper, we propose to use these deviations as a criterion to select a relevant subgraph size. In Propositions 4 and 5 of the upcoming Section 3.3.3, we leverage concentration inequalities to derive a \textit{theoretically-grounded threshold size}, denoted  $n^*_{(S)}$ in the following, for which, at each training iteration, the deviation between the evaluation of $\mathcal{L}^{\text{\tiny FastGAE}}$ for each node and its expectation  is proven to be bounded with a high probability. This proposed subgraph size is of the form:
\begin{align}
n^*_{(S)} = C \sqrt{n}
\end{align}
where constant $C >0$ depends on the deviation magnitude and probability, and is explicitly presented thereafter. Our experiments will confirm the relevance of this choice.

\subsubsection{On Approximate Losses}

Let us recall that, in our FastGAE framework, at each training iteration we run a full GCN forward pass and sample a subgraph $\mathcal{G}_{(S)} = (\mathcal{V}_{(S)},\mathcal{E}_{(S)})$. Then, we evaluate reconstruction losses only on this subgraph, which involves fewer operations w.r.t. standard decoders, and we use the resulting approximate loss for GCN weights updates via gradient descent. 
More precisely, in standard implementations of graph AE/VAE, the cross entropy loss (from Section 2.1 on AE) and the negative of the ELBO's expectation part (from Section 2.2 on VAE) are empirically derived by computing the following node pairs average at each training iteration:
\begin{equation}\mathcal{L} = \frac{1}{n^2}\sum_{(i,j) \in \mathcal{V}^2} \mathcal{L}_{ij}(A_{ij},\hat{A}_{ij}),\end{equation}
with\footnote{In most graph AE and VAE implementations (see e.g. \cite{kipf2016-2}), the terms with $A_{ij} = 1$ are often re-weighted in the loss, in case of sparse $A$. They are multiplied by $w \geq 1$, a positive links re-weighting scalar parameter which is usually inversely proportional to the graph sparsity. In our analyses, to clarify notations, we omit this scalar multiplication, which is equivalent to implicitly assuming that $w=1$. This simplification is made without loss of generality and all results remain valid for any $w>1$.}:
$$\mathcal{L}_{ij}(A_{ij},\hat{A}_{ij}) = - [A_{ij}\log(\hat{A}_{ij}) + (1-A_{ij})\log(1 - \hat{A}_{ij})].$$
In the FastGAE framework, we instead compute:
\begin{equation}\mathcal{L}^{\text{\tiny FastGAE}} = \frac{1}{n_{(S)}^2}\sum_{(i,j) \in \mathcal{V}^{2}} \mathds{1}_{((i,j) \in \mathcal{V}_{(S)}^2)}\mathcal{L}_{ij}(A_{ij},\hat{A}_{ij}),\end{equation}
with $\mathds{1}_{((i,j) \in \mathcal{V}_{(S)}^2)} = 1$ if $(i,j) \in \mathcal{V}_{(S)}^2$ and $0$ otherwise.

We recall that, for variational FastGAE, we need to substract the Kullback-Leibler (KL) divergence, as in the ELBO of standard graph VAE, to obtain our actual loss evaluation. At this stage, two options are possible:
\begin{itemize}
    \item Computing the KL term only on the nodes in the subgraph. 
    \item Or, computing the KL term on all nodes.
\end{itemize}
We consider that the two options are valid. The first one ensures that the resulting loss is a proper lower bound of the likelihood computed on this subgraph. The second one, despite violating this property, can nonetheless be empirically convenient and interpreted as the addition of a \textit{regularization term} on all node embedding vectors (penalizing large deviations w.r.t. a $\mathcal{N}(0,I_d)$ prior distribution on these vectors) to the \textit{performance} term $\mathcal{L}^{\text{\tiny FastGAE}}$. In our implementations, both options returned similar results. In the following propositions, we assume that the KL term is computed on \textit{all} nodes for simplicity, and we therefore only approximate the performance term $\mathcal{L}$, both in the AE and in the VAE settings.

Propositions 1 and 2 detail the formal probabilities to sample a given node or a given node pair at each training iteration. We consider both sampling variants \textit{with} and \textit{without} replacement (see Section 3.2) for this analysis, as the former significantly simplifies results w.r.t. the later. 

\begin{theorem}
Let $\mathcal{G}_{(S)} = (\mathcal{V}_{(S)},\mathcal{E}_{(S)})$ be a subgraph of $\mathcal{G}$ obtained from sampling $n_{(S)}$ nodes \textbf{with} replacement using the node sampling strategy of FastGAE. Let $i$ and $j$ denote two distinct nodes from the original graph $\mathcal{G}$: $(i,j) \in \mathcal{V}^2$. Then:
\begin{equation}\mathbb{P}\Big(i \in \mathcal{V}_{(S)}\Big) = 1 - (1 - p_i)^{n_{(S)}}.\end{equation}
Also:
\begin{align}
\mathbb{P}\Big((i,j) \in \mathcal{V}_{(S)}^2\Big) &= 1 - \Big[(1 - p_i)^{n_{(S)}} + (1 - p_j)^{n_{(S)}} \nonumber \\ 
&- (1 - p_i - p_j)^{n_{(S)}}\Big].
\end{align}
\end{theorem}

\begin{theorem}
Let $\mathcal{G}_{(S)} = (\mathcal{V}_{(S)},\mathcal{E}_{(S)})$ be a subgraph of $\mathcal{G}$ obtained from sampling $n_{(S)}$ nodes \textbf{without} replacement using the node sampling strategy of FastGAE. Let $i$ and $j$ denote two distinct nodes from $\mathcal{G}$: $(i,j) \in \mathcal{V}^2$. Then:
\begin{equation}\mathbb{P}\Big(i \in \mathcal{V}_{(S)}\Big) = \sum_{\mathcal{U} \in \mathcal{\textbf{U}}(i)} p_{u_1} \prod_{k=2}^{n_{(S)}} \frac{p_{u_k}}{1 - \sum_{k'=1}^{k-1} p_{u_{k'}}},\end{equation}
where $\mathcal{\textbf{U}}(i) = \{\mathcal{U} \subset \mathcal{V}, |\mathcal{U}| = n_{(S)} \text{ and } i \in \mathcal{U}\}$ is the set of all ordered subsets of $n_{(S)}$ distinct nodes including node $i$. For a given set $\mathcal{U} \in \mathcal{\textbf{U}}(i)$, we denote by $(u_1, u_2,...,u_{n_{(S)}})$ its ordered elements. Also,
\begin{equation}\mathbb{P}\Big((i,j) \in \mathcal{V}_{(S)}^2\Big) = 
\sum_{\mathcal{U} \in \mathcal{\textbf{U}}(i) \cap \textbf{U}(j)} p_{u_1} \prod_{k=2}^{n_{(S)}} \frac{p_{u_k}}{1 - \sum_{k'=1}^{k-1} p_{u_{k'}}}.\end{equation}
\end{theorem}

Despite different formulations, both variants share a similar behaviour in practice on most real-world graphs. In this paper, as explained in Section 3.2, we sample nodes \textit{without replacement}. One can derive from the above expressions that the probability to draw a node $i$, or an edge incident to $i$, increases with $n_{(S)}$, with $p_i$ and with $f(i)$ for $\alpha > 0$. This also leads to the following formulation of the expected (re-weighted) loss that FastGAE stochastically optimizes.

\begin{theorem}
Using the expressions of Proposition 1 (with replacement) or Proposition 2 (without replacement):
\begin{equation}\mathbb{E}\Big[\mathcal{L}^{\text{\tiny FastGAE}}\Big] = \frac{1}{n_{(S)}^2}\sum_{(i,j) \in \mathcal{V}^2} \mathbb{P}\Big((i,j) \in \mathcal{V}_{(S)}^2\Big) \mathcal{L}_{ij}(A_{ij},\hat{A}_{ij}).\end{equation}
\end{theorem}

\subsubsection{On the Selection of $n_{(S)}$}

While our experiments will tend to show that stochastically minimizing $\mathbb{E}(\mathcal{L}^{\text{\tiny FastGAE}})$ (equation 17) instead of $\mathcal{L}$ (equation 11) might be beneficial, we also acknowledge that, for small values of $n_{(S)}$, the actual loss $\mathcal{L}^{\text{\tiny FastGAE}}$ computed at a given training iteration (equation 12) might significantly deviate from its expectation.

We propose to use these deviations as a criterion to automatically select a relevant subgraph size. More precisely, let us rewrite $\mathcal{L}^{\text{\tiny FastGAE}}$ from equation 12 as follows:
$$\mathcal{L}^{\text{\tiny FastGAE}} = \frac{1}{n_{(S)}} \sum_{i \in \mathcal{V}} \mathds{1}_{(i \in \mathcal{V}_{(S)})} \mathcal{L}^{\text{\tiny FastGAE}}(i),$$
where the node-level terms $\mathcal{L}^{\text{\tiny FastGAE}}(i)$ are defined as:
$$\mathcal{L}^{\text{\tiny FastGAE}}(i) =  \frac{1}{n_{(S)}} \sum_{j \in \mathcal{V}} \mathds{1}_{(j \in \mathcal{V}_{(S)})} \mathcal{L}_{ij}(A_{ij},\hat{A}_{ij})$$
and where $\mathcal{L}_{ij}$ denotes the cross entropy loss as in equation 11.
In the following, we leverage concentration inequalities \cite{hoeffding1963probability} to derive a theoretically-grounded threshold size, denoted  $n^*_{(S)}$ in the following, for which, under mild asssumptions, the (random) node-level deviation
$$|\mathcal{L}^{\text{\tiny FastGAE}}(i) - \mathbb{E}[\mathcal{L}^{\text{\tiny FastGAE}}(i)]|$$
at each training iteration is proven to be bounded with a high probability, for any node $i$. This proposed subgraph size is of the form:
$$n^*_{(S)} = C \sqrt{n},$$
where constant $C >0$ depends on the deviation magnitude and probability, and is explicitly presented in Proposition 5. In our empirical analysis, this criterion will allow us to significantly improve the scalability and training speed of graph AE and VAE models (see discussion on complexity in Section 3.4), while reaching fairly competitive performances in a majority of experiments (see Section 4).

To prove our bounds, we require the following technical assumption on the reconstructed matrix $\hat{A}$:

\begin{assumption}
Let $(i,j) \in \mathcal{V}^2$. We thereafter assume that $\hat{A}_{ij} = \sigma(z^T_i z_j)$ can actually be \textbf{capped}, and that:
$$\hat{A}_{ij} \in [\varepsilon,1-\varepsilon]$$
where $0 < \varepsilon < 1$ is a constant that can be arbitrarily close to 0.
\end{assumption}
Under this assumption, we derive Propositions 4 and 5.
\begin{theorem}
Let us consider a training iteration of the FastGAE framework, a sampled subgraph $\mathcal{G}_{(S)} = (\mathcal{V}_{(S)},\mathcal{E}_{(S)})$, with $|\mathcal{V}_{(S)}| = n_{(S)} < n$ nodes sampled without replacement, and the corresponding node-level approximate reconstruction computed for a given node $i$:
$$\mathcal{L}^{\text{\tiny FastGAE}}(i) =  \frac{1}{n_{(S)}} \sum_{j \in \mathcal{V}} \mathds{1}_{(j \in \mathcal{V}_{(S)})} \mathcal{L}_{ij}(A_{ij},\hat{A}_{ij}),$$
with the random variable $\mathds{1}_{(j \in \mathcal{V}_{(S)})} = 1$ if node $j \in \mathcal{V}_{(S)}$ and $0$ otherwise, with $A_{ij} \in \{0,1\}$ for all $(i,j) \in \mathcal{V}^2$ and with:
$$\mathcal{L}_{ij}(A_{ij},\hat{A}_{ij}) = - [A_{ij}\log(\hat{A}_{ij}) + (1-A_{ij})\log(1 - \hat{A}_{ij})].$$
Then, under Assumption 1, for any $\gamma \geq 0$, we have:
\begin{small}
\begin{equation}
\mathbb{P}(|\mathcal{L}^{\text{\tiny FastGAE}}(i) - \mathbb{E}[\mathcal{L}^{\text{\tiny FastGAE}}(i)]| \geq \gamma) \leq 2~\text{exp}\Big( - 2(\frac{\gamma}{\log(\varepsilon)})^2 \frac{n_{(S)}^2}{n}\Big).
\end{equation}
\end{small}
\end{theorem}

We note that the right hand side term tends to 0 exponentially fast w.r.t. the deviation magnitude $\gamma$ and w.r.t. the subgraph size $n_{(S)}$.

\begin{theorem}
For any confidence level $\alpha \in ]0,1[$ and node $i \in \mathcal{V}$, selecting a subgraph size $n_{(S)}$ such that
\begin{equation}
n_{(S)} \geq n^*_{(S)} = \sqrt{n} \underbrace{\sqrt{\frac{-\log (\frac{\alpha}{2})\log(\varepsilon)^2}{ 2\gamma^2}}}_{\text{denoted $C$ in eq. 10}}
\end{equation}
guarantees that
$$\mathbb{P}(|\mathcal{L}^{\text{\tiny FastGAE}}(i) - \mathbb{E}[\mathcal{L}^{\text{\tiny FastGAE}}(i)]| \geq \gamma) \leq \alpha.$$
\end{theorem}

As an opening, we note that, while the current bounds are empirically effective (see Section 4), future research will aim at directly bounding the deviation of $\mathcal{L}^{\text{\tiny FastGAE}}$ instead of the node-level terms $\mathcal{L}^{\text{\tiny FastGAE}}(i)$, which would be more ambitious and challenging due to the inherent dependencies among sampled node pairs in FastGAE. Also, while Proposition 4 and 5 focus on the case of the \textit{cross entropy loss} for consistency w.r.t. the main paper, a similar analysis (omitted here) could be performed to obtain comparable bounds for other \textit{bounded} reconstruction losses. For instance, in the case of the Frobenius loss, where $\mathcal{L}_{ij}(A_{ij},\hat{A}_{ij}) = (A_{ij} - \hat{A}_{ij})^2$, and without Assumption 1, one can obtain similar concentration guarantees as Proposition 5, with $C$ being replaced by the constant $C' = \sqrt{\frac{-\log(\alpha/2)}{2 \gamma^2}}$.

\paragraph{Numerical Application}

In our experiments, all $n^*_{(S)}$ threshold subgraph sizes will be computed by evaluating equation 19, setting the following values for hyperparameters: $\gamma = 1$, $\alpha = 0.1$ and $\varepsilon = 0.001$.

\subsection{Complexity of FastGAE}

As previously detailed, both the encoder and the sampling step of FastGAE have a linear time complexity w.r.t. $m$. Moreover, our decoder runs in $O(n_{(S)}^2)$ time, with $n_{(S)}$ being significantly smaller than $n$ in practice. In particular, setting  $n_{(S)} = n^*_{(S)}$ ensures a $O(n)$ time complexity for decoding (as $n^{*2}_{(S)} = (C \sqrt{n})^2 = C^2 n$) and an overall $O(m+n)$ linear time complexity for a complete FastGAE training iteration. Faster bounds can also be achieved by lowering $n_{(S)}$ or by replacing GCNs by another encoder. Therefore, as we will empirically verify in Section 4, our framework is significantly faster and more scalable than standard graph AE/VAE. 

\subsection{Differences with Related Work}

Before diving into experiments, we would like to emphasize some key differences between FastGAE and existing works. Foremost, FastGAE is \textit{not} directly comparable to the existing research cited in Section 2.3 to scale GCN models, e.g. to FastGCN \cite{chen2018fastgcn} that also sample nodes. Indeed, FastGCN is a GCN-like model, optimized to classify node labels in a (semi) supervised fashion. It samples the \textit{neighborhood} of each node when averaging vector representations in forward passes. 
On the contrary, in this paper, after \textit{full} GCN forward passes, we instead sample \textit{subgraphs to reconstruct}, in order to approximate the \textit{reconstruction losses} of two unsupervised models, in which GCNs are only a building part (the encoder) of a larger framework (the AE or the VAE).
Both settings therefore address different problems; as explained in Section 3.1, FastGCN, GraphSAGE or Cluster-GCN could actually be used \textit{in conjuction with} FastGAE, as encoders.

Futhermore, FastGAE is also more elaborated than data cleaning methods that simply consist in removing some nodes from a graph, e.g. the low-degree ones, to reduce its size. Indeed, in the case of FastGAE with degree sampling, low-degree nodes are still 1) \textit{fully} used in the GCN encoder, and 2) might also appear in \textit{some} subgraphs that we decode (but less often than high-degree nodes). As we leverage new different subgraphs at each iteration, we explore different parts of the \textit{entire} graph during training.

Last, we note that \textit{effective subset selection} for faster learning has already provided promising results in the machine learning community \cite{act1,act2,act3}; however, contrary to these works, we focus on an unsupervised graph-based problem, and our sampling methods remain fixed throughout learning as we rely on graph mining to select $\mathcal{G}_{(S)}$.

\section{Empirical Analysis}

In this section, we present an in-depth experimental evaluation of our proposed framework to scale graph AE and VAE models.

\subsection{Experimental Setting}

\subsubsection{Datasets}

\begin{table}[t]
\begin{center}
\begin{small}
\centering
\begin{tabular}{c|c|c}
\toprule
\textbf{Dataset} & \textbf{Number of nodes} & \textbf{Number of edges}  \\
\midrule
\midrule
\textbf{Cora} & 2708 & 5429  \\
\textbf{Citeseer} & 3327 & 4732  \\
\textbf{Pubmed} & 19717 & 44338 \\
\textbf{SBM} & 100000 & 1498844 \\
\textbf{Google} & 875713  & 4322051  \\
\textbf{Youtube} & 3223589  & 9375374 \\
\textbf{Patent} &  3774768  & 16518948 \\
\bottomrule
\end{tabular}
\caption{Datasets Statistics}
\end{small}
\end{center}
\end{table}
We provide experimental results on seven graphs of increasing sizes. Their statistics are presented in Table 1. We first study the Cora, Citeseer and Pubmed citation networks\footnote{\label{linqs}\href{https://linqs.soe.ucsc.edu/data}{https://linqs.soe.ucsc.edu/data}}, with and without node features corresponding to $f$-dimensional bag-of-words vectors (with $f =$ 1433, 3703 and 500 respectively). Nodes are clustered in respectively 6, 7 and 3 topic classes, acting as ground truth communities. These datasets are common benchmarks for evaluating graph AE and VAE (see \citet{kipf2016-2} and a majority of recent works \cite{wang2017mgae,pan2018arga,semiimplicit2019,salha2019-1,aaai20,tran2018multi,grover2019graphite,salha2019-2,huang2019rwr}). For these \textit{medium-size graphs}, we can directly compare the performance of FastGAE to standard graph AE and VAE, as training standard models is still computationally affordable.

Then, we consider four significantly \textit{larger graphs}, with up to millions of nodes and edges, and for which training standard graph AE or VAE is \textit{intractable}. We consider the Google\footnote{\label{snap}\href{http://snap.stanford.edu/data/index.html}{http://snap.stanford.edu/data/index.html}} hyperlinks web graph, the Youtube\footnote{\label{konect}\href{http://konect.cc/networks/}{http://konect.cc/networks/}} social network of users (friendship connections), the US Patent\textsuperscript{\ref{snap}} citation network, and a synthetic graph, denoted SBM, generated from a \textit{stochastic block model} which is a generative model for random graphs \cite{abbe2017community}. In this last graph, by design, nodes are clustered in 100 groups of 1000 nodes, acting as ground truth communities. Two nodes from a same community (resp. from different communities) are connected by an edge with probability $2 \times 10^{-2}$ (resp. $2 \times 10^{-4}$).

Our evaluation therefore includes graphs with various characteristics, sizes, and from four different families (citation networks, social networks, web graphs and stochastic block model graphs).

\subsubsection{Evaluation Tasks}

We consider two downstream tasks for evaluation:
\begin{itemize}
    \item First, we consider a \textit{link prediction} task. We train all models on masked graphs for which $15\%$ of edges were randomly removed. Then, we create validation and test sets from the removed edges (resp. from $5\%$ and $10\%$ of edges) and from the same number of sampled unconnected node pairs. Using decoder predictions $\hat{A}_{ij}$, we evaluate our ability to classify edges from non-edges, using the mean \textit{Area Under the ROC Curve} (AUC) and \textit{Average Precision} (AP) scores\footnote{We computed scores via scikit-learn \cite{pedregosa2011scikit}. Formulas are provided in \href{https://scikit-learn.org/stable/modules/generated/sklearn.metrics.roc_auc_score.html}{https://scikit-learn.org/stable/modules/generated/sklearn.metrics.roc\_auc\_score.html} for AUC, and in \href{https://scikit-learn.org/stable/modules/generated/sklearn.metrics.average_precision_score.html}{https://scikit-learn.org/stable/modules/generated/sklearn.metrics.average\_precision\_score.html} for AP.} on test sets. Link prediction is the most common task to evaluate graph AE and VAE models since the seminal research of \citet{kipf2016-2} (see e.g. \citet{salha2020simple} and references therein for an overview), and we therefore found essential to consider it as well in our work.
    \item We also perform \textit{node clustering} experiments, on datasets with ground truth communities. For this task, after training models on \textit{complete} versions of the graphs, we run $k$-means algorithms in embedding spaces to cluster the $z_i$ vectors. We compare these clusters to the ground truth ones using the mean \textit{Adjusted Mutual Information} (AMI) scores\footnote{We computed AMI scores via scikit-learn \cite{pedregosa2011scikit}. The formula is provided in \href{https://scikit-learn.org/stable/modules/generated/sklearn.metrics.adjusted_mutual_info_score.html}{https://scikit-learn.org/stable/modules/generated/sklearn.metrics.adjusted\_mutual\_info\_score.html}} on test sets. 
\end{itemize} 

We emphasize that we chose to focus on these two tasks instead of a more direct reconstruction task (despite working with autoencoders), as AUC and AP scores from link prediction, as well as AMI scores from node clustering, are more insighful and understandable metrics than a direct reporting of some cross-entropy or ELBO reconstruction losses. Besides, we also aimed at providing a consistent experimental setting w.r.t. the existing literature on graph AE and and VAE that, for the most part, focused on link prediction and, to some extent, on node clustering tasks.

\subsubsection{Details on Models: Hyperparameters and Model Selection Procedure for Standard and FastGAE-based AE/VAE}

In the upcoming experiments, for the aforementioned graphs and evaluation tasks, we compare standard graph AE and VAE models (when they are tractable) to FastGAE-based versions of these models.

All AE and VAE models, with and without FastGAE, were optimized for the \textit{link prediction} task. More specifically, we selected the best sets of hyperparameters in terms of mean AUC scores \textit{on the validation sets} introduced in Section 4.1.2. Instructions to easily run a similar validation are provided in our source code.

We trained models for 200 iterations (resp. 300) for graphs with $n <$ 100000 (resp. $n \geq$ 100000), and thoroughly checked the convergence of all models for these values (in terms of loss stability in the validation set). Other hyperparameters for these models are described thereafter. 

Our encoders are 2-layer GCNs (we tested models with 1 to 3 layers). They include 32-dim hidden layers, and 16-dim output layer, which means that the dimension of embedding vectors (denoted $d$ in Section 2.1.1 on AE and in Section 2.2.1 on VAE) is equal to $d = 16$. We emphasize that we also tested models with $d \in \{32,64,128\}$, reaching similar conclusions w.r.t. $d =16$ (the impact of $d$ is further discussed in Section 4.2.4). 

Besides, for all models, we used the Adam optimizer \cite{kingma2014adam}, without dropout (we tested models with dropout values in $\{0,0.1,0.2,0.3,0.4,0.5\}$). Regarding learning rates for such optimizer, we tested values from the grid $\{0.0001,0.0005,0.001,0.005,0.01,0.05,0.1,0.2\}$. We eventually picked a learning rate of 0.1 for Patent with uniform sampling, and of 0.01 otherwise as, one again, these values returned the best mean AUC scores on validation sets. Last, as \citet{kipf2016-2}, we considered all graphs as \textit{undirected} and ignored edges directions when available.

We used TensorFlow \cite{abadi2016tensorflow}, training models on an NVIDIA GTX 1080 GPU, and running other operations on a double Intel Xeon Gold 6134 CPU.

\subsubsection{Details on Models: Other Baselines} For completeness, we also compare standard graph AE/VAE and FastGAE-based models to the few other existing methods to scale graph AE/VAE, using a similar sets of hyperparameters and similar embedding values ($d=16$) as previous section:
\begin{itemize}
    \item We consider a simple \textit{negative sampling} strategy, briefly mentioned by \citet{kipf2016-2} and suggested by \citet{pytorchgeometric}, where we reconstruct all edges but only $|\mathcal{E}|$ randomly picked unconnected node pairs to compute losses. We leveraged methods made available by \citet{pytorchgeometric} to estimate losses, with consistent dropout values, learning rates and architectures w.r.t. Section 4.1.3.
    \item We also consider the framework recently proposed by \citet{salha2019-1}, denoted as \textit{Core-GAE} in next tables. Authors train the AE/VAE only on the smaller graph $k$-core, then propagate embedding representations to other nodes out of the $k$-core via simple heuristics; $k$ is a parameter tuning the size of the input graph for learning (which ranges from 1 to the maximal $k$ for which the corresponding $k$-core is not empty). We used the author's implementation \cite{salha2019-1} with optimal values (regarding mean AUC scores on validation sets) for the hyperparameter $k$ detailed in next tables, and with consistent dropout values, learning rates and architectures w.r.t. Section 4.1.3. 
    \item Besides, the other aforementioned sampling ideas briefly mentioned (as possible extensions) in the recent literature \cite{grover2019graphite,salha2020simple} actually are particular cases of FastGAE, namely with \textit{uniform} sampling. 
\end{itemize}
Last, in addition to an extensive comparison between the different AE/VAE models, we also report results obtained with the following non AE/VAE-based baselines:
\begin{itemize}
    \item A \textit{spectral embedding}, which is a powerful but not scalable baseline. We used the implementation provided by \citet{pedregosa2011scikit}; embedding axes correspond the the $d$ first eigenvectors of $\mathcal{G}$'s Laplacian matrix, excluding non-informative vectors, and with $d$ denoting the embedding dimension \cite{von2007tutorial} as for all other models.
    \item \textit{node2vec} \cite{grover2016node2vec}, another very popular and scalable node embedding method. We trained models with hyperparameters $p = 1$ and $q = 1$, from 10 random walks of length 80 per node, with a window size of 5 and on a single epoch, and using the author's implementation \cite{grover2016node2vec}. We omit comparison to other random walk-based models \cite{perozzi2014deepwalk,tang2015line} due to very similar performances on some of our preliminary tests. 
    \item For node clustering, we also compare our approach to \textit{Louvain}'s scalable community detection algorithm, from a direct usage of the authors' implementation \cite{blondel2008louvain}.
\end{itemize}

\subsection{Results}

In the remainder of this section, we provide an empirical evaluation of FastGAE and of its variational FastGAE variant.

\subsubsection{Preliminary Insights on High Degree/Core Nodes}

Before studying FastGAE we report important insights from preliminary experiments on standard graph AE/VAE. They motivated the design of our framework and emphasize the relevance of sampling high-degree/core nodes. On the medium-size Cora, Citeseer and Pubmed graphs, we trained standard graph AE/VAE models, but \textit{tried to mask $k$ nodes and their edges} from the computation of reconstruction losses, for different values of $k$. Such masking procedure is expected to lower performances, as the model leverages less information about the quality of the reconstruction for learning.

Figure 1 shows that, when these $k$ removed nodes are \textit{the top-$k$ highest degrees/cores nodes}, performances on the link prediction task tumble down. On the contrary, removing \textit{the $k$ nodes with minimal degrees or core numbers} from the loss leads to almost no drop, and even slightly better results on Pubmed, which suggests that removing non-informative nodes might even be beneficial for learning.

\begin{figure*}[!ht]
\centering
  \subfigure[Cora - Degree masking]{
  \scalebox{0.32}{\includegraphics{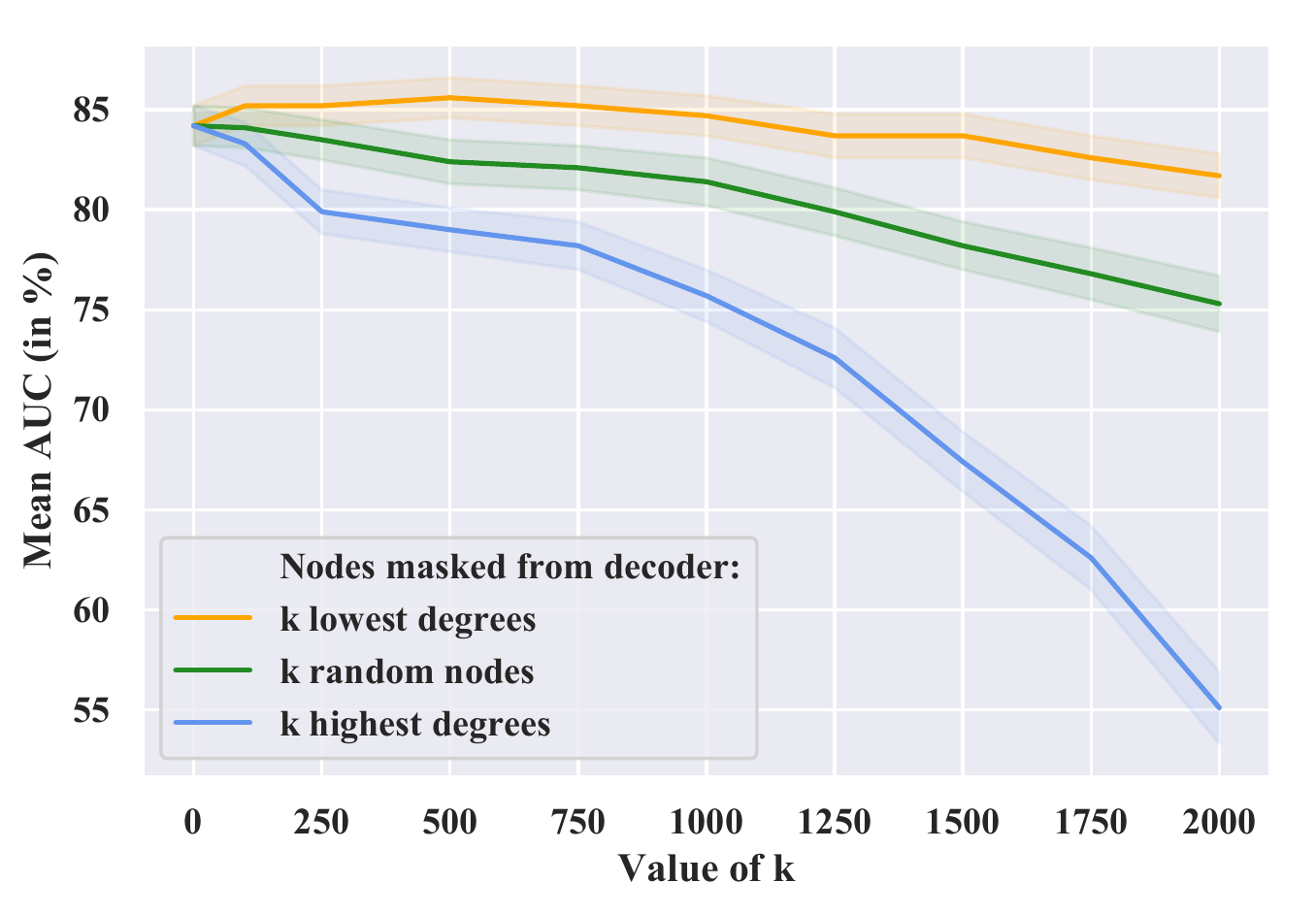}}}\subfigure[Citeseer - Degree masking]{
  \scalebox{0.32}{\includegraphics{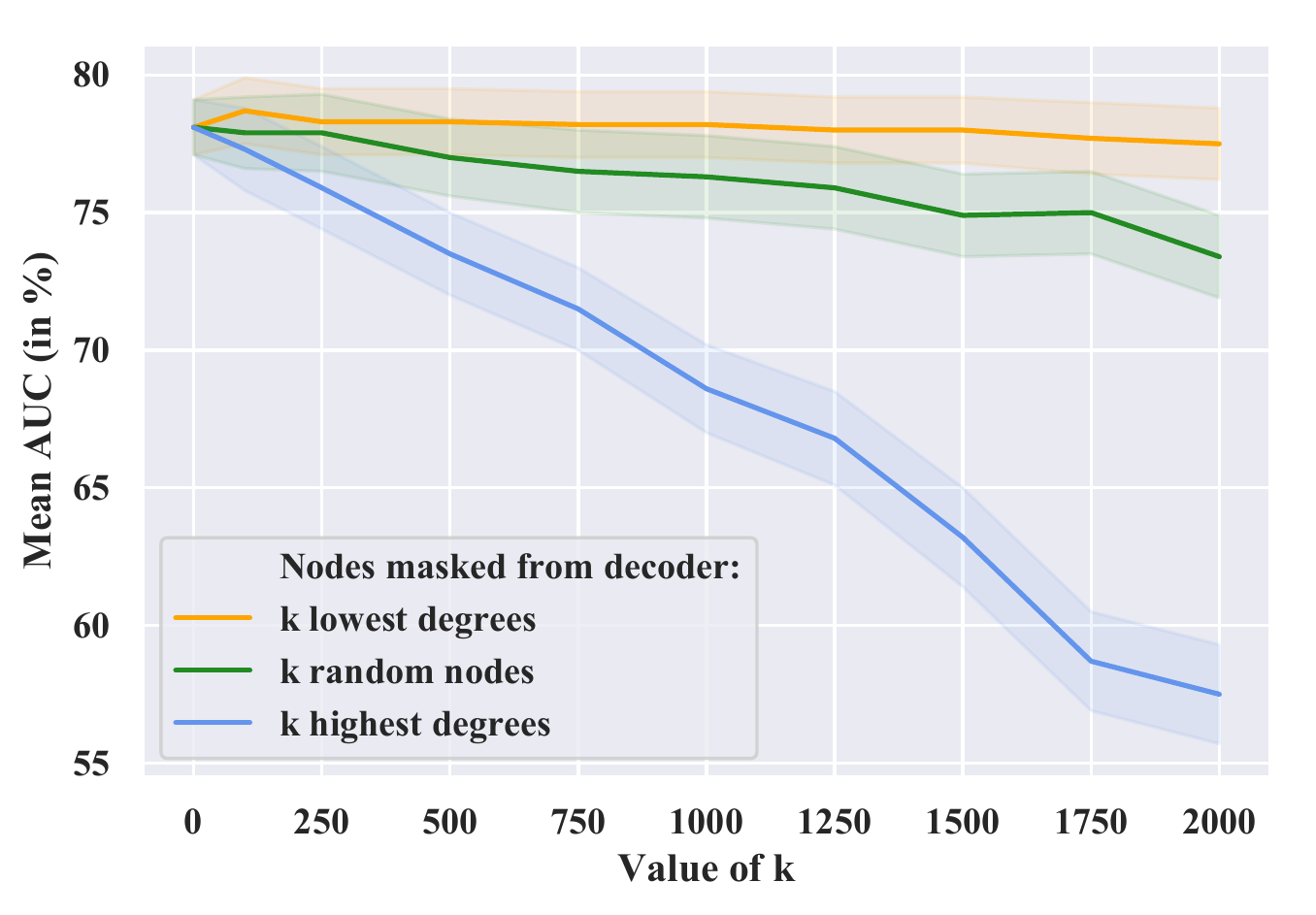}}}\subfigure[Pubmed - Degree masking]{
  \scalebox{0.32}{\includegraphics{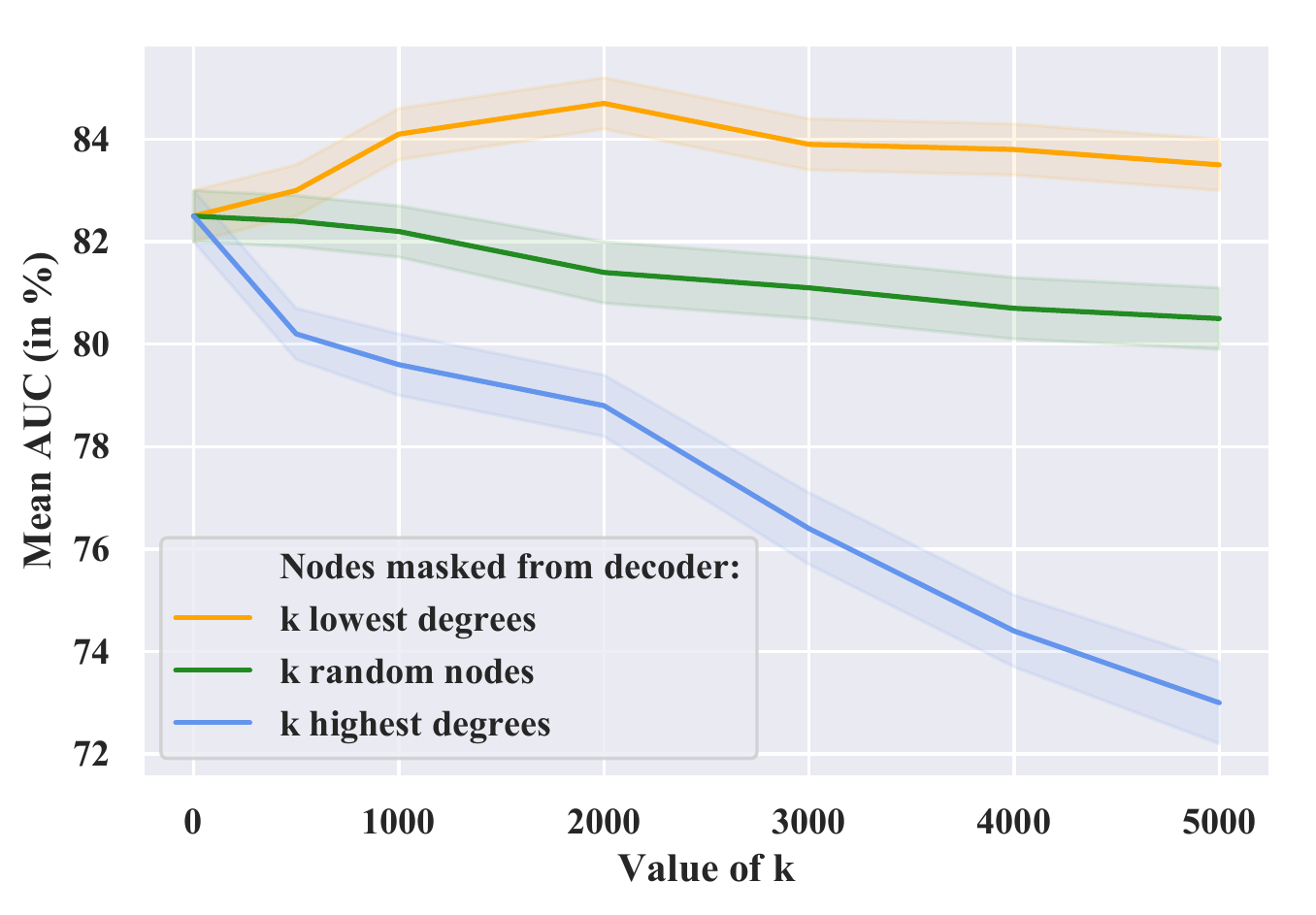}}}
    \subfigure[Cora - Core masking]{
  \scalebox{0.32}{\includegraphics{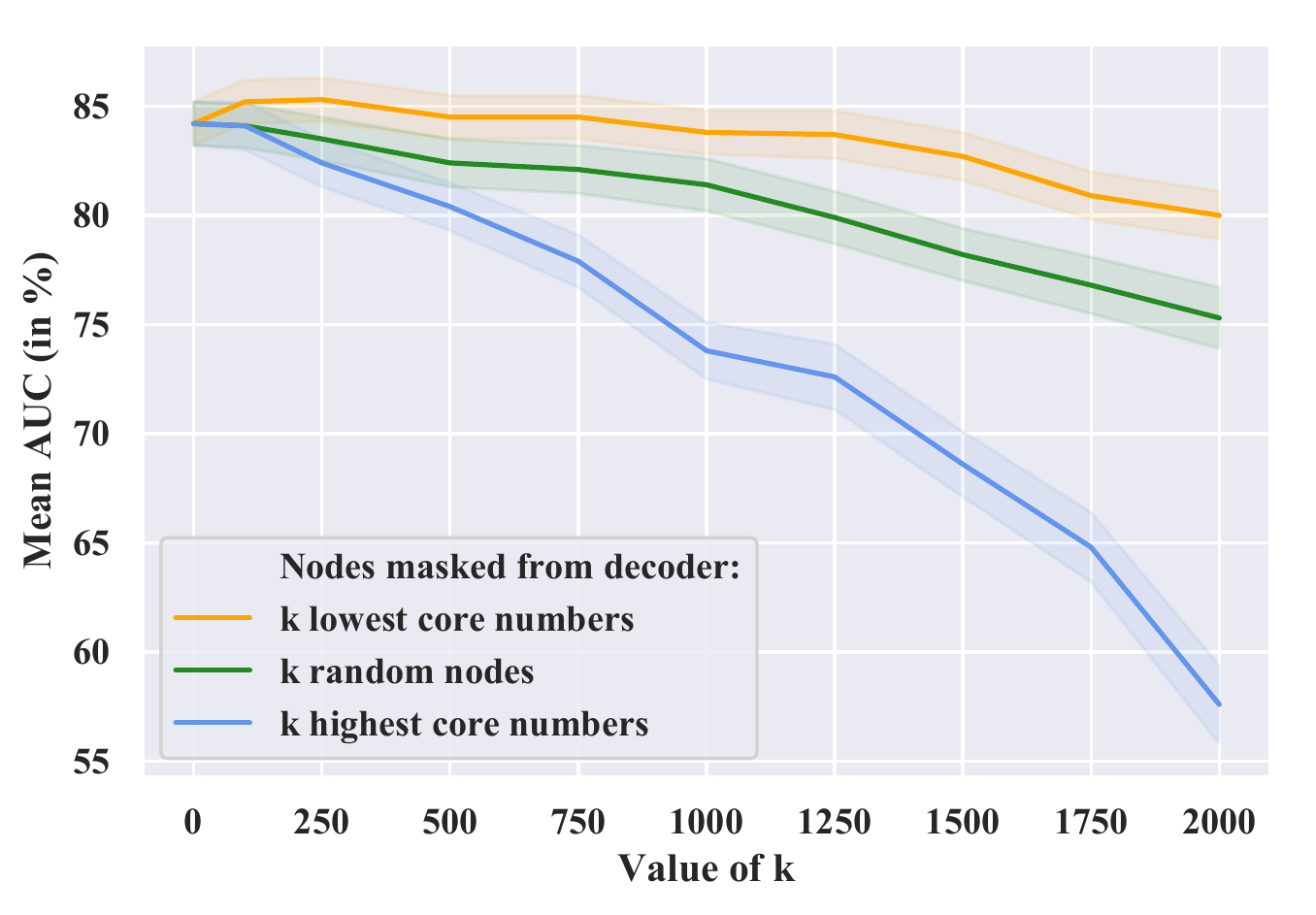}}}\subfigure[Citeseer - Core masking]{
  \scalebox{0.32}{\includegraphics{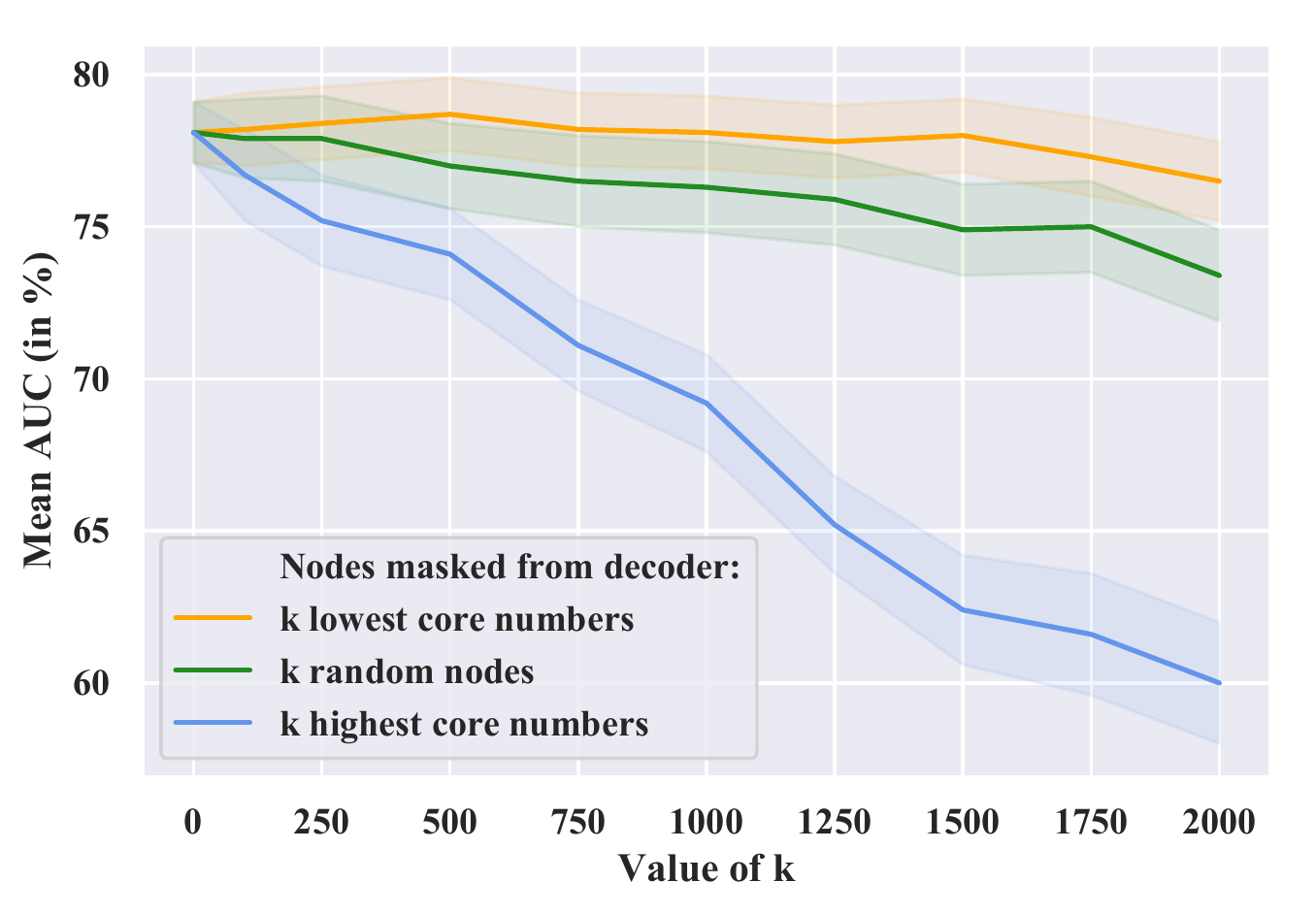}}}\subfigure[Pubmed - Core masking]{
  \scalebox{0.32}{\includegraphics{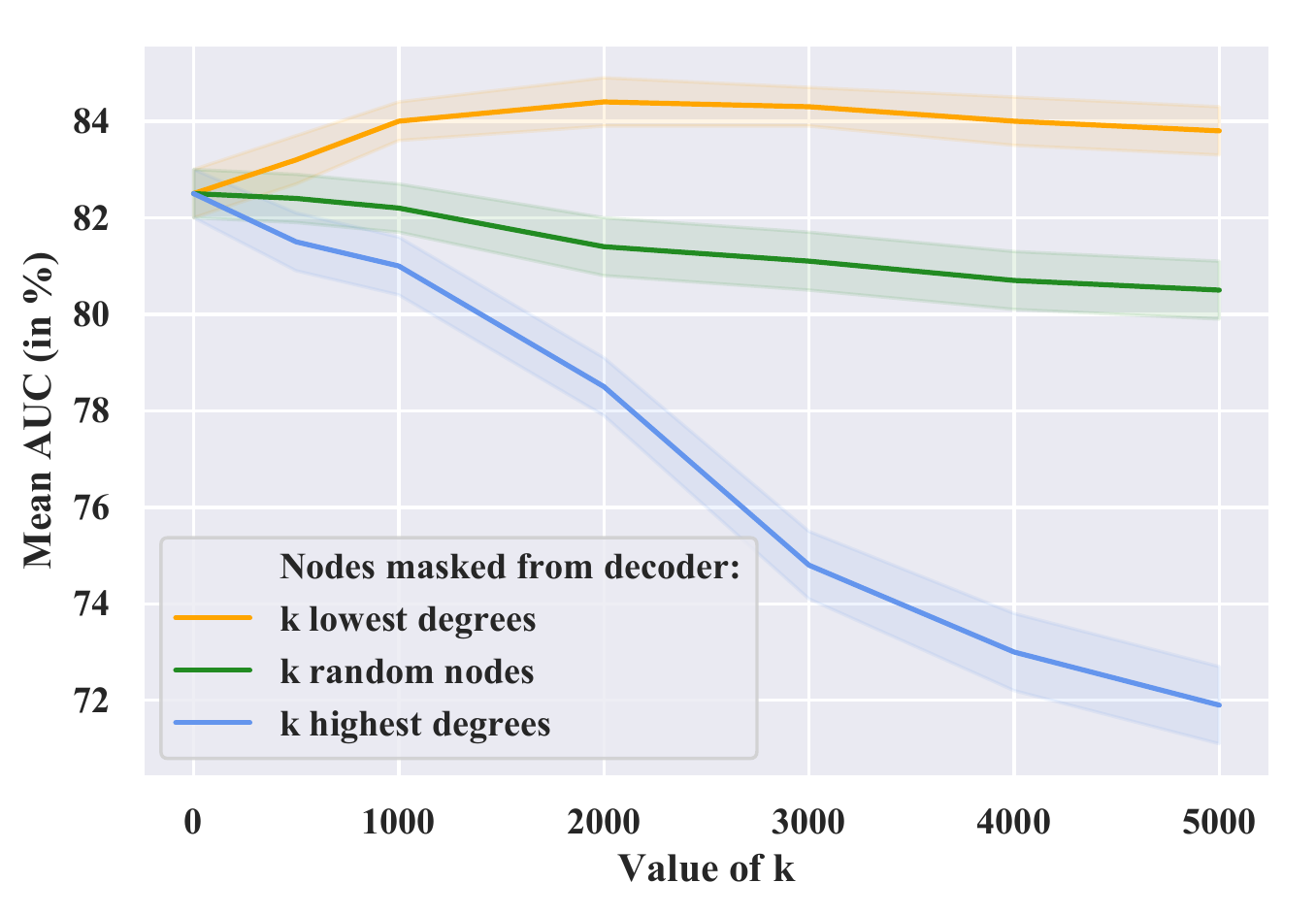}}}
      \caption{Link prediction on featureless Cora, Citesser and Pubmed using standard Graph VAE models, but trained while masking $k$ nodes and their connections from the decoder/reconstruction loss. AUC scores are averaged over 100 runs with random train/test splits.} 
      \vspace{5pt}
  \subfigure[Cora - Degree masking]{
  \scalebox{0.32}{\includegraphics{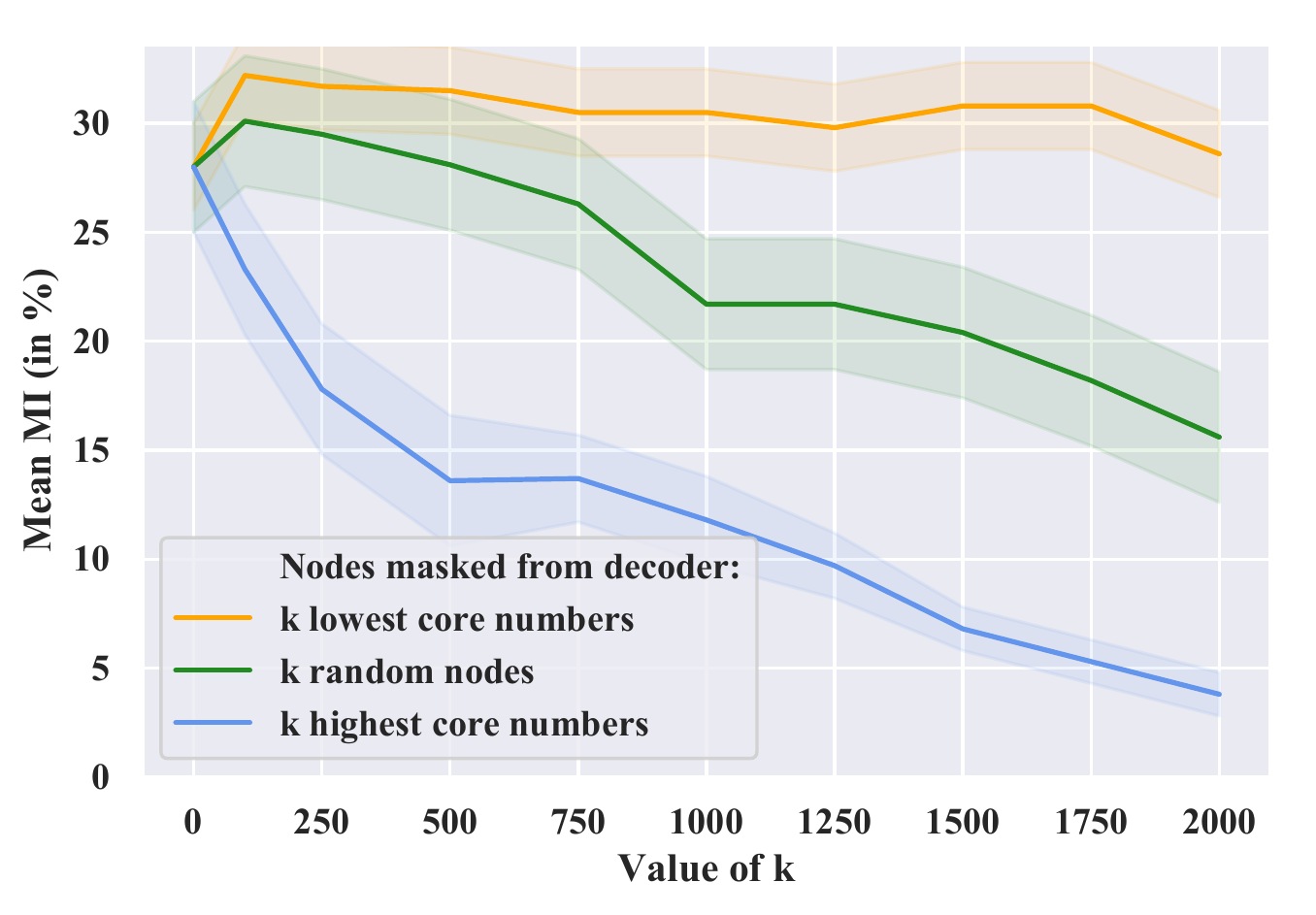}}}\subfigure[Citeseer - Degree masking]{
  \scalebox{0.32}{\includegraphics{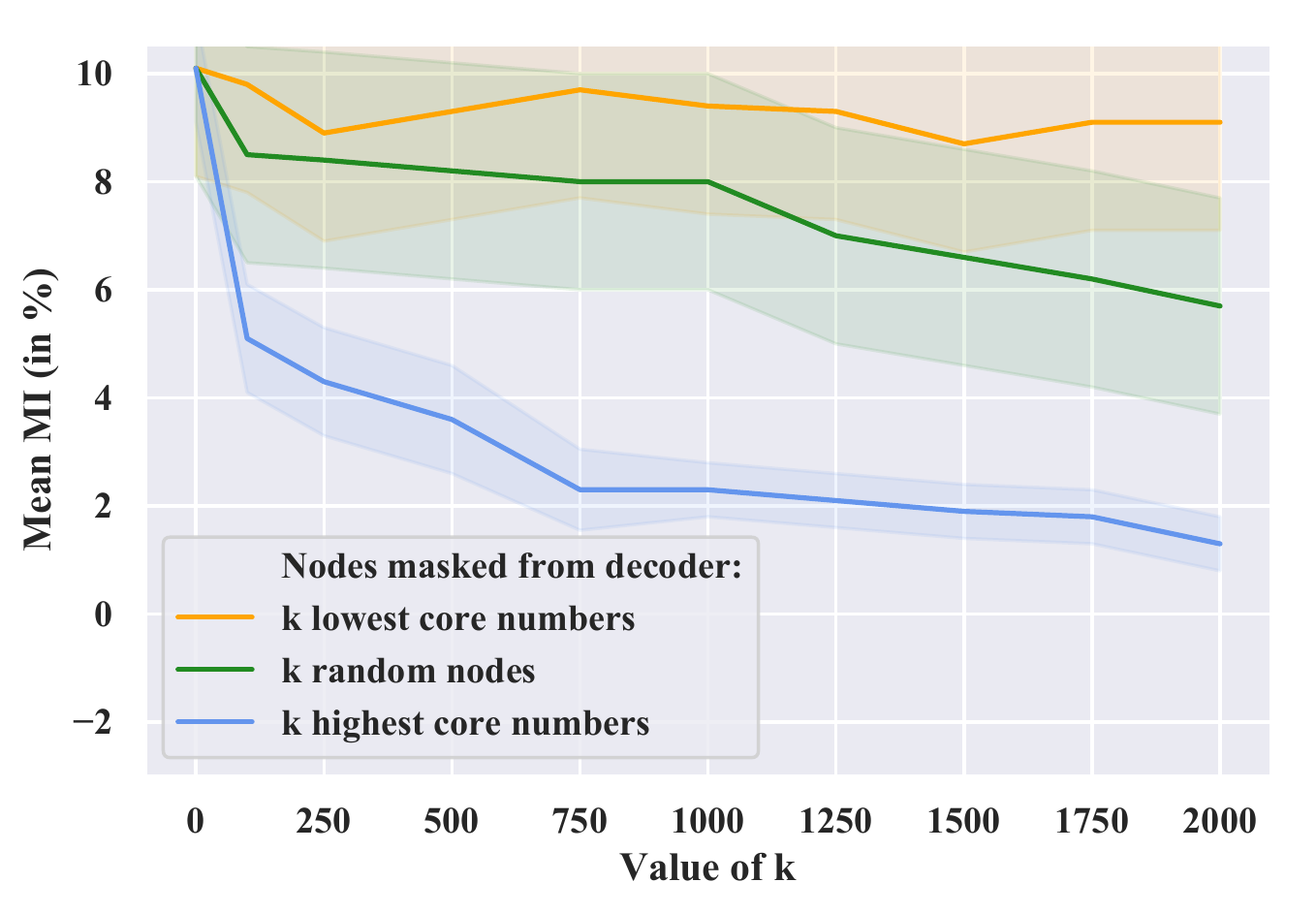}}}\subfigure[Pubmed - Degree masking]{
  \scalebox{0.32}{\includegraphics{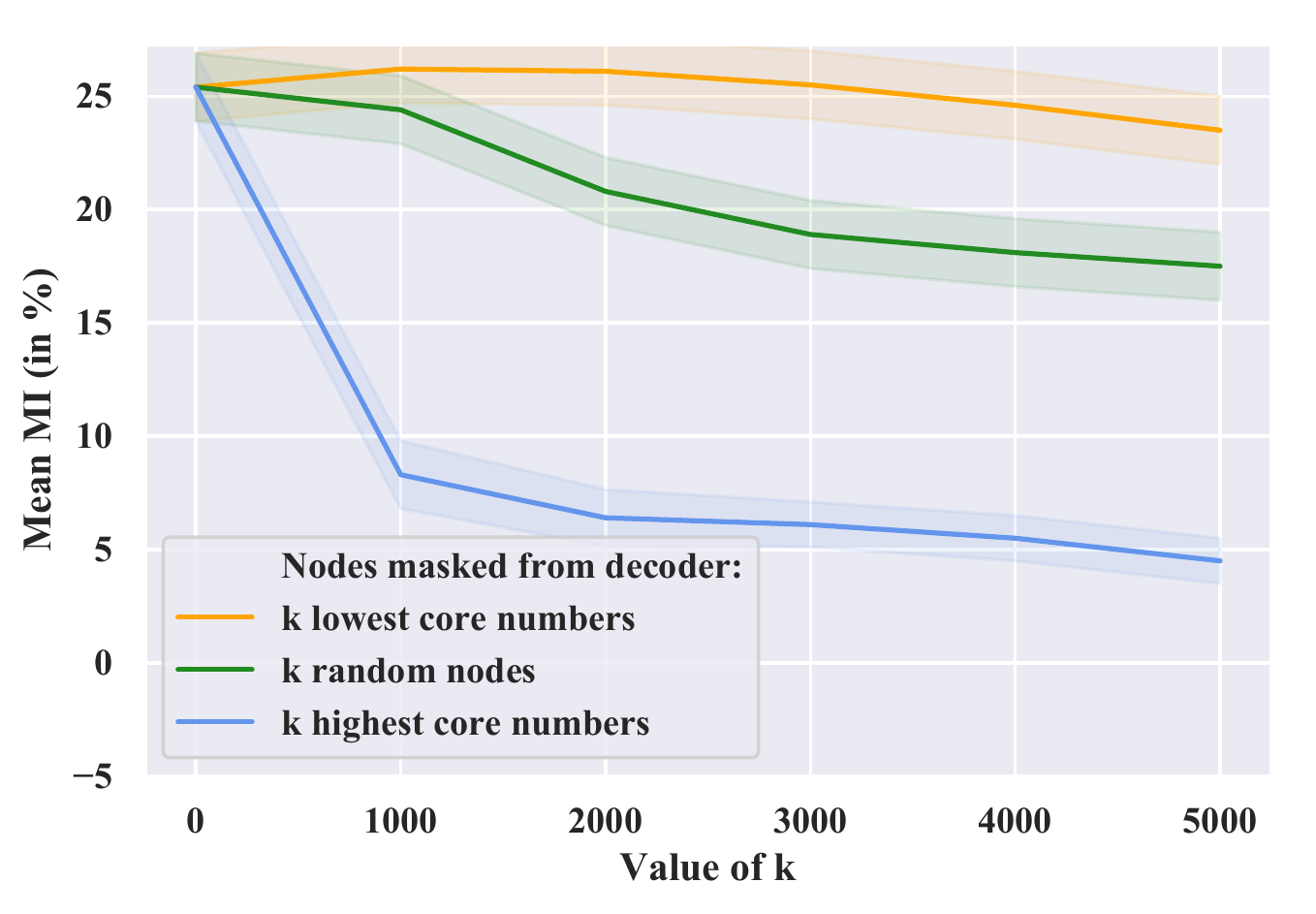}}}
    \subfigure[Cora - Core masking]{
  \scalebox{0.32}{\includegraphics{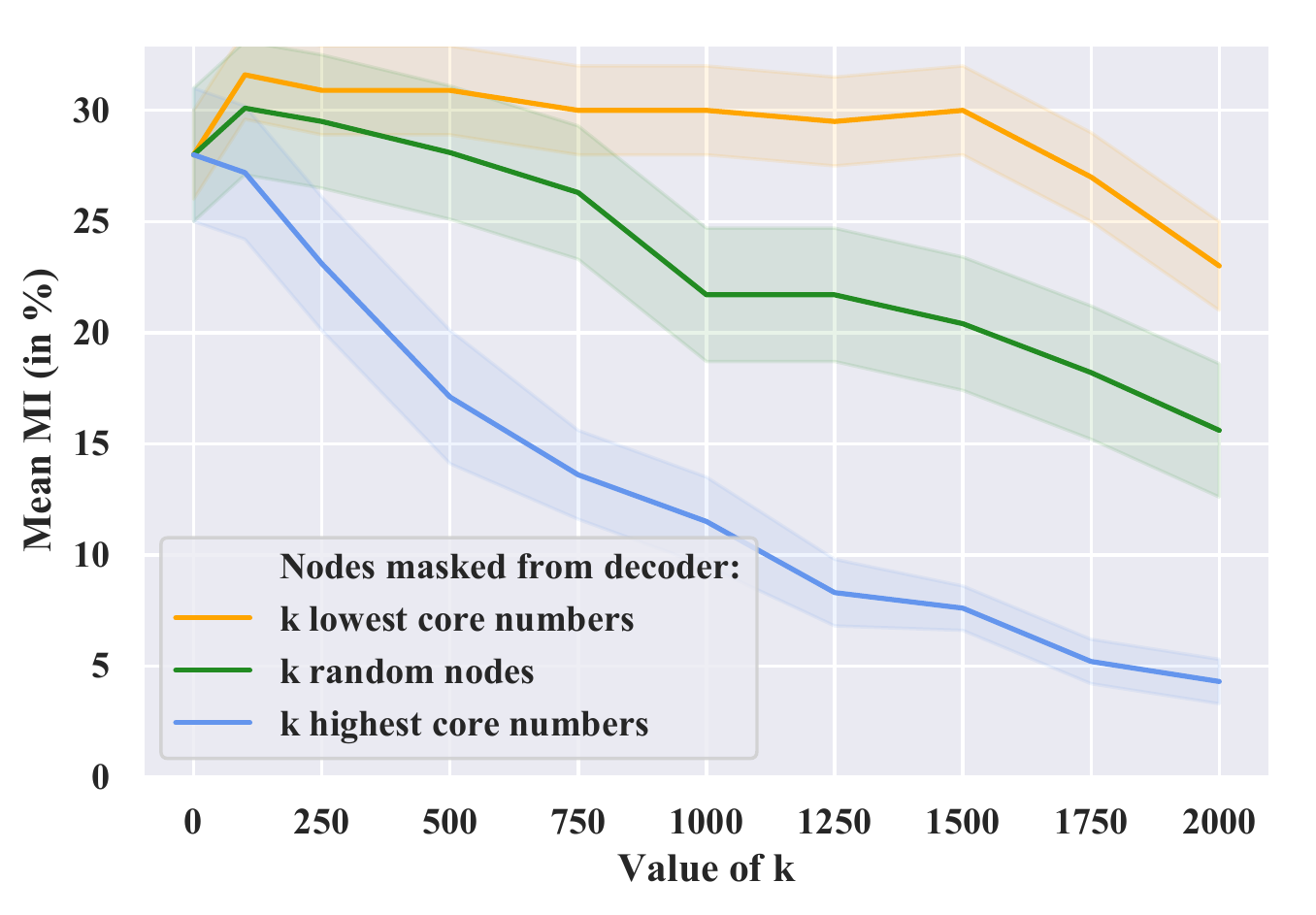}}}\subfigure[Citeseer - Core masking]{
  \scalebox{0.32}{\includegraphics{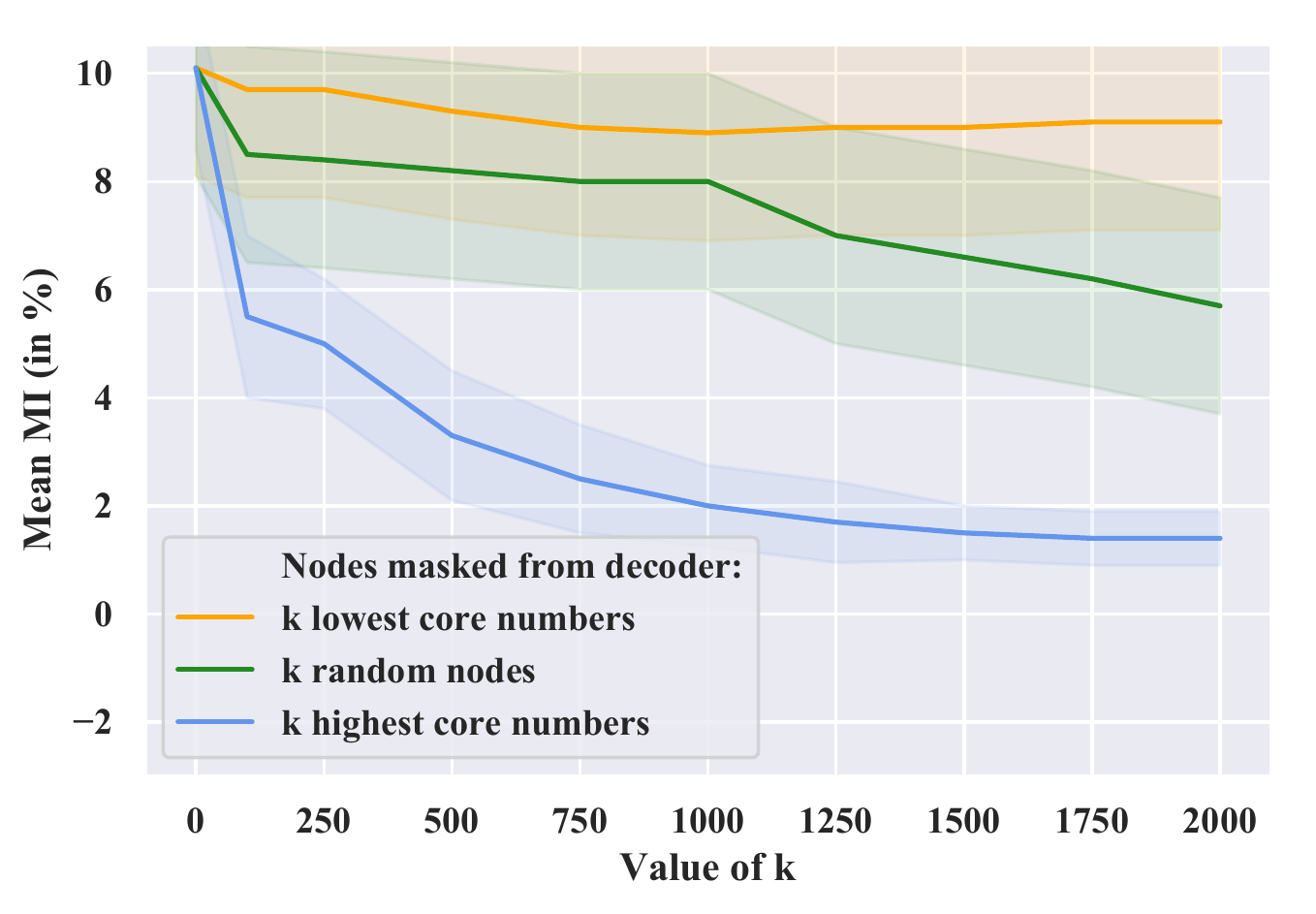}}}\subfigure[Pubmed - Core masking]{
  \scalebox{0.32}{\includegraphics{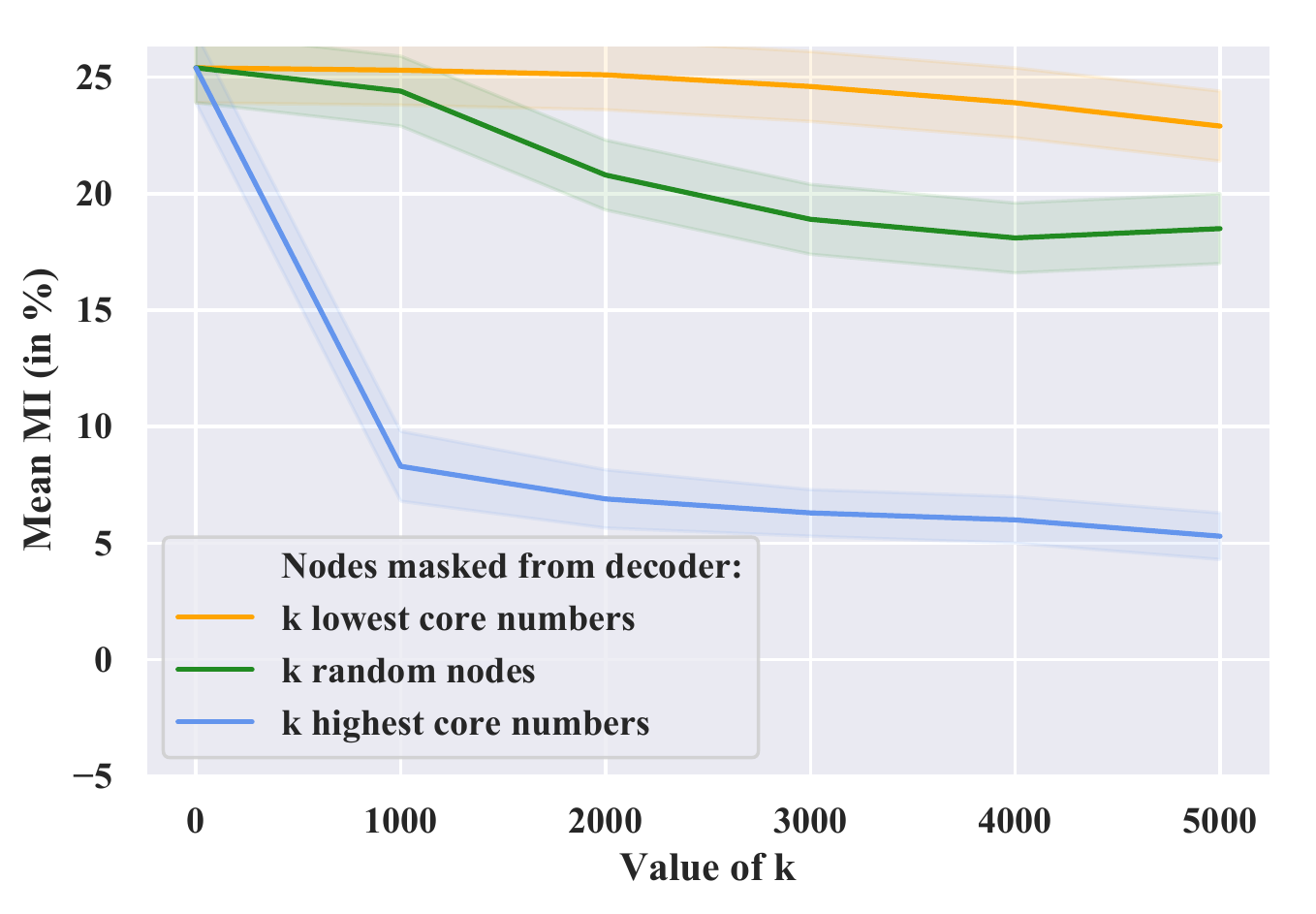}}}
  \caption{Node clustering on featureless Cora, Citesser and Pubmed using standard Graph VAE models, but trained while masking $k$ nodes and their connections from the decoder/reconstruction loss. Adjusted MI scores are averaged over 100 runs with random train/test splits.}
\end{figure*}

In Figure 2, we report similar results on Adjusted Mutual Information scores for node clustering. These ablation studies suggest that, when implementing stochastic subgraph decoding strategies for scalability, sampling high-degree/core nodes is indeed crucial to learn meaningful embeddings. FastGAE, which explicitly exploits these structural node properties, and optimizes a reconstruction loss that re-weights high degrees/cores node pairs, behaves consistently w.r.t. such important insights.

\clearpage

\begin{table*}[!ht]
\centering
\begin{tiny}
\begin{tabular}{c|c|cc|ccc|c}
\toprule
\textbf{Model}  & \textbf{Subgraphs} & \multicolumn{2}{c}{\textbf{Average Perf. on Test Set}} & \multicolumn{4}{c}{\textbf{Average Running Times (in seconds)}}\\
(Dimension $d=16$)& \textbf{size $n_{(S)}$} & \tiny \textbf{AUC (in \%)} & \tiny \textbf{AP (in \%)} & \tiny Compute & \tiny Train & \tiny \textbf{Total} & \tiny Speed gain \\ &  &  & & \tiny $p_i$ & \tiny model &  & \tiny w.r.t. GAE \\
\midrule
\midrule
Standard Graph AE  & - & 82.51 $\pm$ 0.64 & 87.42 $\pm$ 0.38 & - & 811.43 & 811.43 & - \\
\midrule
FastGAE with & 5000 & \textbf{84.82} $\pm$ \textbf{0.32} & \textbf{88.19} $\pm$ \textbf{0.23} & 0.01 & 14.41 & 14.42 & $\times$ 56.27 \\
\textbf{degree} sampling & 2500 & 84.12 $\pm$ 0.40 & 87.56 $\pm$ 0.30 & 0.01 & 5.72 & 5.73 & $\times$ 141.61 \\
($\alpha = 1$)& 1187$^*$ & 83.67 $\pm$ 0.42 & 87.01 $\pm$ 0.31 & 0.01 & 3.20 & 3.21 & $\times$ 252.78\\
& 500 & 82.68 $\pm$ 0.51 & 85.89 $\pm$ 0.47 & 0.01 & 2.98 & 2.99 & $\times$ 271.38 \\
& 250 & 80.77 $\pm$ 0.55 & 84.05 $\pm$ 0.51 & 0.01 & 2.83 & \textbf{2.84} & $\times$ \textbf{285.71} \\
\midrule
FastGAE with & 5000 & \textbf{84.62} $\pm$ \textbf{0.24} & \textbf{88.09} $\pm$ \textbf{0.16} & 1.75 & 15.98 & 17.73 & $\times$ 45.77 \\
\textbf{core} sampling & 2500 & 83.69 $\pm$ 0.34 & 87.28 $\pm$ 0.31 & 1.75 & 7.51 & 9.26 & $\times$ 87.63 \\
($\alpha = 2$) & 1187$^*$ & 82.53 $\pm$ 0.46 & 86.28 $\pm$ 0.37 & 1.75 & 4.81 & 6.56 & $\times$ 123.69 \\
& 500 & 80.96 $\pm$ 0.52 & 84.86 $\pm$ 0.46 & 1.75 & 4.57 & 6.32 & $\times$ 128.39 \\
& 250 & 79.53 $\pm$ 0.53 & 83.10 $\pm$ 0.50 & 1.75 & 4.44 & 6.19 & $\times$ 131.08 \\
\midrule
FastGAE with & 5000 & 81.08 $\pm$ 0.48 & 85.90 $\pm$ 0.60 & - & 13.90 & 13.90 & $\times$ 58.37 \\
\textbf{uniform} sampling & 2500 & 78.72 $\pm$ 0.74 & 83.50 $\pm$ 0.75 & - & 5.48 & 5.48 & $\times$ 148.07 \\
& 1187$^*$ & 77.28 $\pm$ 0.89 & 81.89 $\pm$ 0.91 & - & 3.10 & 3.10 & $\times$ 261.75 \\
& 500 & 75.09 $\pm$ 2.05 & 78.53 $\pm$ 2.04 & - & 2.98 & 2.98 & $\times$ 271.29\\
& 250 & 74.12 $\pm$ 2.07 & 77.72 $\pm$ 1.22 & - & 2.82 & \textbf{2.82} & $\times$ \textbf{287.74} \\
\midrule
\midrule
Core-GAE, $k=2$ (best choice) & - & 84.30 $\pm$ 0.27 & 86.11 $\pm$ 0.43 & - & 168.91 & 168.91 & $\times$ 4.80 \\
Core-GAE, $k=9$ (fastest choice) & - & 61.65 $\pm$ 0.94 & 64.82 $\pm$ 0.72 & - & 2.92 & 2.92 & $\times$ 277.89 \\
Negative Sampling GAE & - & 81.19 $\pm$ 0.68 & 83.21 $\pm$ 0.40 & - & 111.79 & 111.79  & $\times$ 7.28\\
node2vec & - & 81.25 $\pm$ 0.26 & 85.55 $\pm$ 0.26 & - & 48.91 & 48.91 & $\times$ 16.59\\
Spectral Embedding & - & 83.14 $\pm$ 0.42 & 86.55 $\pm$ 0.41 & - & 31.71 & 31.71 & $\times$ 25.59 \\
\bottomrule
\end{tabular}
\end{tiny}
\caption{Link prediction on the featureless Pubmed graph ($n=$ 19717, $m =$ 44338) using standard Graph AE, FastGAE with degree, core and uniform sampling, and baselines. For degree and core sampling, values of the hyperparameter $\alpha$ (as defined in equation 9) were tuned, as described in Figure C.6. All models learn embedding vectors of dimension $d=16$. Scores are averaged over 100 runs with different and random train/validation/test sets. Bold numbers correspond to the best performance (several numbers are bold when scores are comparable, in a $\pm 1$ standard deviation range) and best running time. Subgraphs sizes annotated with $^*$ correspond to the $n^*_{(S)}$ threshold, as introduced in equation 19.}
\end{table*}

\subsubsection{FastGAE for Medium-Size Graphs}

We now evaluate FastGAE and its variational FastGAE variant. First, we focus on \textit{medium-size graphs}. For Cora, Citeseer and Pubmed, we can compare FastGAE to standard graph AE/VAE. The above Table 2 details mean AUC and AP scores and standard errors over 100 runs with different train/test splits for link prediction on (featureless) Pubmed with AE models. For the sake of brevity, we report more summarized results for other medium-size graphs, for VAE and for node clustering, in Table 3 and Figure 3 (for link prediction) as well as in Table 4 and Figure 4 (for node clustering).

\paragraph{FastGAE vs Standard Graph AE/VAE} Let us first compare FastGAE to standard graph AE/VAE models. In Table 2, we observe that, for sample sizes roughly 20 times smaller than $n$, FastGAE models with \textit{degree} and \textit{core} sampling both achieve competitive or even outperforming\footnote{At first glance, the fact that FastGAE sometimes even slightly \textit{outperforms} standard graph AE/VAE models might be surprising. This improvement is actually consistent with recent research on the benefits of mini-batch-based GNNs \cite{rong2020dropedge,hu2020open}. It comes from the relevance of the two sampling schemes that we consider (core-based and degree-based) and from the stochastic nature of the training, that might tend to avoid local minima more easily \cite{kleinberg2018alternative}.} results w.r.t. standard graph AE on Pubmed (e.g. +2.31 AUC points for FastGAE with degree sampling and $n_{(S)} =$ 5000).

Futhermore, FastGAE models are also significantly \textit{faster}: in Table 2 for instance, our approach with degree sampling is up to $\times$~252.78 faster without performance degradation. The additional operation required by our framework, i.e. computing the $p_i$ distribution, is efficient in practice, especially for degree sampling. By further reducing the subgraph size $n_{(S)}$, one can achieve even faster results, while only losing a few AUC/AP points in performance.

In Table 3, Table 4, Figure 3 and Figure 4, we consolidate our results by reaching similar conclusions on VAE, on other medium-size graphs (with and without features), and on node clustering. On Figures 3 and 4, we also confirm that, even for relatively low $n_{(S)}/n$ proportions, our proposed method achieves comparable performances w.r.t. standard models.

\begin{table*}[!t]
\centering
\begin{tiny}
\begin{tabular}{c|l|cc|ccc|c}
\toprule
\textbf{Dataset}  & \textbf{Model}  & \multicolumn{2}{c}{\textbf{Average Perf. on Test Set}} & \multicolumn{3}{c}{\textbf{Avg. Run. Times (in sec.)}} & \textbf{Speed}\\
& (Dimension $d=16$) & \tiny \textbf{AUC (in \%)} & \tiny \textbf{AP (in \%)} & \tiny Comp. & \tiny Train & \tiny \textbf{Total} & \tiny \textbf{Gain} \\
&  &  &  & \tiny $p_i$ & \tiny model & \tiny  &  \\ 
\midrule
\midrule
 & Standard Graph AE & 84.79 $\pm$ 1.10 & \textbf{88.45} $\pm$ \textbf{0.82} & - & 3.87 & 3.87 & - \\
 & \underline{FastGAE (degree, $\alpha = 2$)} &  &  &  & & \\
\textbf{Cora} & - with $n_{(S)} =$ 250  & 84.13 $\pm$ 1.20 & 86.65 $\pm$ 1.23 & \textbf{0.002} & \textbf{1.46} & \textbf{1.462} & $\times$ \textbf{2.65} \\
 & -  with $n_{(S)} = n^*_{(S)} =$ 440  & 84.74 $\pm$ 0.81 & 87.42 $\pm$ 0.75 & 0.002 & 1.56 & 1.562 & $\times$ 2.48 \\
 &  - with  $n_{(S)} =$ 1000& 84.75 $\pm$ 0.84 & \textbf{87.77} $\pm$ \textbf{0.81} & 0.002 & 1.65 & 1.652 & $\times$ 2.34 \\
& \underline{Best baseline}  &  &  &  & & \\
& Spectral Embedding & \textbf{86.49} $\pm$ \textbf{0.98} & \textbf{87.42} $\pm$ \textbf{1.04} & - & 2.49 & 2.49 & $\times$ 1.55 \\
\midrule

 & Standard Graph VAE & \textbf{91.64} $\pm$ \textbf{0.92} & \textbf{92.66} $\pm$ \textbf{0.91} & - & 4.25 & 4.25 & - \\
 & \underline{Var. FastGAE (degree, $\alpha = 2$)} &  &  &  & & \\
\textbf{Cora} & - with $n_{(S)} =$ 250 & 90.50 $\pm$ 1.10 & 91.10 $\pm$ 1.08 & \textbf{0.002} & \textbf{2.30} & \textbf{2.302} & $\times$ \textbf{1.85} \\
\textbf{with} & - with $n_{(S)} = n^*_{(S)} =$ 440  & 90.82 $\pm$ 1.07 & 91.44 $\pm$ 1.13 & 0.002 & 2.52 & 2.522 & $\times$ 1.69 \\
\textbf{features}& - with $n_{(S)} =$ 1000  & \textbf{91.72} $\pm$ \textbf{0.98} & \textbf{92.36} $\pm$ \textbf{1.11} & 0.002 & 2.87 & 2.872 & $\times$ 1.48 \\
& \underline{Best baseline}  &   &  &  & & \\
& Core-Graph VAE, $k=2$ & 87.94 $\pm$ 1.12 & 89.00 $\pm$ 1.11 & - & 3.09 & 3.09 & $\times$ 1.38 \\
\midrule

 & Standard Graph AE & 78.25 $\pm$ 1.69 & \textbf{83.79} $\pm$ \textbf{1.24} & - & 5.25 & 5.25 & - \\
 & \underline{FastGAE (degree, $\alpha = 1$)}  &  &  &  & & \\
\textbf{Citeseer} & - with $n_{(S)} =$ 250 & 77.28 $\pm$ 1.11 & 81.29 $\pm$ 0.92 & \textbf{0.002} & \textbf{1.47} & \textbf{1.472} & $\times$ \textbf{3.57}\\
 & - with $n_{(S)} = n^*_{(S)} =$ 488 & 78.30 $\pm$ 1.30 & \textbf{82.42} $\pm$ \textbf{1.09} & 0.002 & 1.58 & 1.582 & $\times$ 3.32 \\
 & - with $n_{(S)} =$ 1000 & 78.31 $\pm$ 1.25 & \textbf{82.40} $\pm$ \textbf{0.99} & 0.002 & 1.61 & 1.612 & $\times$ 3.26  \\
& \underline{Best baseline} &  &  &  & & \\
& Spectral Embedding  & \textbf{80.42} $\pm$ \textbf{1.38} & \textbf{83.75} $\pm$ \textbf{1.12} & - & 3.50 & 3.50 & $\times$ 1.50 \\
\midrule

 & Standard Graph VAE & \textbf{90.72} $\pm$ \textbf{1.01} & \textbf{92.05} $\pm$ \textbf{0.97} & - & 6.28 & 6.28 & - \\
 & \underline{Var. FastGAE (degree, $\alpha = 1$)}   &   &  &  & & \\
\textbf{Citeseer}  & - with $n_{(S)} =$ 250  & 89.37 $\pm$ 1.69 & 89.63 $\pm$ 1.83 & \textbf{0.002} & \textbf{2.32} & \textbf{2.322} & $\times$ \textbf{2.70} \\
\textbf{with} &  - with $n_{(S)} = n^*_{(S)} =$ 488  & \textbf{90.10} $\pm$ \textbf{1.33} & 90.15 $\pm$ 1.50 & 0.002 & 2.62 & 2.622 & $\times$ 2.40\\
\textbf{features} & - with $n_{(S)} =$ 1000 & \textbf{90.22} $\pm$ \textbf{1.14} & 90.16 $\pm$ 1.20 & 0.002 & 2.89 & 2.892 & $\times$ 2.17 \\
& \underline{Best baseline}   &   &  &  & & \\
& Core-Graph VAE, $k=2$ & 81.85 $\pm$ 1.72 & 83.65 $\pm$ 1.64 & - & 2.55 & 2.55 & $\times$ 2.46 \\
\midrule

 & Standard Graph AE & 82.51 $\pm$ 0.64 & 87.42 $\pm$ 0.38 & - & 811.43 & 811.43& - \\
& \underline{FastGAE (degree, $\alpha = 1$)} &  &   &  &  & & \\ 
\textbf{Pubmed} & - with $n_{(S)} =$ 500  & 82.68 $\pm$ 0.51 & 85.89 $\pm$ 0.47 & \textbf{0.01} & \textbf{2.98} & \textbf{2.99} & $\times$ \textbf{271.38} \\ 
 & - with $n_{(S)} = n^*_{(S)} =$ 1187 & 83.67 $\pm$ 0.42 & 87.01 $\pm$ 0.31 & 0.01 & 3.20& 3.21 & $\times$ 252.78 \\ 
& - with $n_{(S)} =$ 5000 & \textbf{84.82} $\pm$ \textbf{0.32} & \textbf{88.19} $\pm$ \textbf{0.23} & 0.01 & 14.41 & 14.42 & $\times$ 56.27 \\
& \underline{Best baseline} &  &   &  &  & & \\ 
& Core-Graph AE, $k=2$  & 84.30 $\pm$ 0.27 & 86.11 $\pm$ 0.43 & - & 168.91 & 168.91 & $\times$ 4.80 \\
\midrule

 & Standard Graph AE & \textbf{96.28} $\pm$ \textbf{0.36} & \textbf{96.29} $\pm$ \textbf{0.25} & - & 952.63 & 952.63 & - \\
 & \underline{FastGAE (degree, $\alpha = 1$)} &  &   &  &  & \\ 
\textbf{Pubmed} & - with $n_{(S)} =$ 500 & 95.08 $\pm$ 0.45 & 95.24 $\pm$ 0.46 & \textbf{0.01} & \textbf{3.53} & \textbf{3.54} & $\times$ \textbf{269.10}\\ 
\textbf{with} & - with $n_{(S)} = n^*_{(S)} =$ 1187 & 95.45 $\pm$ 0.26 & 95.70 $\pm$ 0.30 & 0.01 & 4.01 & 4.02 & $\times$ 237.56\\ 
\textbf{features} & - with $n_{(S)} =$ 5000 & \textbf{96.12} $\pm$ \textbf{0.20} & \textbf{96.35} $\pm$ \textbf{0.19} & 0.01 & 19.74 & 19.75 & $\times$ 48.23 \\
& \underline{Best baseline} &   &   &  &  & & \\ 
& Core-Graph AE, $k=2$ & 85.34 $\pm$ 0.33 & 86.06 $\pm$ 0.24 & - & 40.22 & 40.22 & $\times$ 23.69 \\
\bottomrule
\end{tabular}
\end{tiny}
\caption{Link prediction on all medium-size graphs. For each graph, for brevity, we only report the \textbf{best} graph AE \textbf{or} VAE model in terms of AUC and AP scores, a few representative degree-based FastGAE versions of this model, and the best baseline (among Core-Graph AE/VAE, Negative Sampling Graph AE/VAE, node2vec and the spectral embedding). Scores are averaged over 100 runs with different and random train/validation/test sets. For degree sampling, values of the hyperparameter $\alpha$ (as defined in equation 9) were tuned, as described in Figure C.6. All models learn embedding vectors of dimension $d=16$. Bold numbers correspond to the best performance (several numbers are bold when scores are comparable, in a $\pm 1$ standard deviation range) and best running time.}
\end{table*}


\begin{table*}[!t]
\centering
\begin{tiny}
\begin{tabular}{c|l|c|ccc|c}
\toprule
\textbf{Dataset}  & \textbf{Model}& \textbf{Average Performance} & \multicolumn{3}{c}{\textbf{Average Running Times (in sec.)}} & \textbf{Speed}\\
&  (Dimension $d=16$) & \tiny \textbf{AMI (in \%)} & \tiny Compute & \tiny Train & \tiny \textbf{Total} & \tiny \textbf{Gain} \\
& &  & \tiny $p_i$ & \tiny model &  &  \\ 
\midrule
\midrule

 & Standard Graph AE & 30.88 $\pm$ 2.56 & - & 3.90 & 3.90 & - \\
 & \underline{FastGAE (degree, $\alpha = 2$)} &  &  &  & & \\
\textbf{Cora}  & - with $n_{(S)} =$ 250 & 33.32 $\pm$ 2.61 & \textbf{0.002} & \textbf{1.51} & \textbf{1.512} & $\times$ \textbf{2.58}\\
 & - with $n_{(S)} = n^*_{(S)} =$ 440 & 34.64 $\pm$ 2.45 & 0.002 & 1.59 & 1.592 & $\times$ 2.45\\
 & - with $n_{(S)} =$ 1000 & 35.56 $\pm$ 2.80 & 0.002 & 1.67 & 1.672 & $\times$ 2.33\\
& \underline{Best baseline} &  &  &  & & \\
& Louvain & \textbf{46.72} $\pm$ \textbf{0.85} & - & 1.79 & 1.79 & $\times$ 2.18 \\
\midrule

  & Standard Graph VAE & \textbf{44.84} $\pm$ \textbf{2.63} & - & 4.32 & 4.32 & - \\
 & \underline{Var. FastGAE (degree, $\alpha = 2$)} &  &  &  & & \\
\textbf{Cora} & - with $n_{(S)} =$ 250 & 41.35 $\pm$ 3.49 & 0.002 & 2.40 & 2.402 & $\times$ 1.80\\
\textbf{with} & - with $n_{(S)} = n^*_{(S)} =$ 440 & 42.89 $\pm$ 2.72 & 0.002 & 2.67 & 2.672 & $\times$ 1.62\\
\textbf{features}& - with $n_{(S)} =$ 1000 & \textbf{45.02} $\pm$ \textbf{2.81} & 0.002 & 2.92 & 2.922 & $\times$ 1.48 \\
& \underline{Best baseline} &  &  &  & & \\
& Louvain & \textbf{46.72} $\pm$ \textbf{0.85} & - & \textbf{1.79} & \textbf{1.79} & $\times$ \textbf{2.41} \\
\midrule

 & Standard Graph VAE & 9.85 $\pm$ 1.24  & - & 5.44 & 5.44 & - \\
 & \underline{Var. FastGAE (degree, $\alpha = 1$)} &  &  &  & & \\
\textbf{Citeseer} & - with $n_{(S)} =$ 250 & 9.34 $\pm$ 1.48 & \textbf{0.002} & \textbf{1.77} & \textbf{1.772} & $\times$ \textbf{3.07}\\
 & - with $n_{(S)} = n^*_{(S)} =$ 488 & 10.02 $\pm$ 1.42 & 0.002 &  2.02 & 2.022 & $\times$ 2.69\\
 & - with $n_{(S)} =$ 1000 & 10.16 $\pm$ 1.41 & 0.002 & 2.19 & 2.192 & $\times$ 2.48\\
& \underline{Best baseline} &  &  &  & & \\
& Louvain & \textbf{16.39} $\pm$ \textbf{1.45} & - & 2.41 & 2.41 & $\times$ 2.26 \\
\midrule

 & Standard Graph VAE & \textbf{20.17} $\pm$ \textbf{3.07} & - & 6.45 & 6.45 & - \\
 & \underline{Var. FastGAE (degree, $\alpha = 1$)} & &   &  &  & \\
\textbf{Citeseer} & - with $n_{(S)} =$ 250  & \textbf{20.49} $\pm$ \textbf{3.74} & \textbf{0.002} & \textbf{2.80} & \textbf{2.802} & $\times$ \textbf{2.30} \\
\textbf{with} & - with $n_{(S)} = n^*_{(S)} =$ 488  & \textbf{20.53} $\pm$ \textbf{3.45}  & 0.002 & 2.88 & 2.882 & $\times$ 2.24\\
\textbf{features}&  - with $n_{(S)} =$ 1000 & \textbf{20.94} $\pm$ \textbf{3.21} & 0.002 & 3.11 & 3.112 & $\times$ 2.07 \\
 & \underline{Best baseline} &   &   &  &  & \\
 & Cora-Graph VAE, $k=2$  & 16.53 $\pm$ 1.95 & - & \textbf{2.76} & \textbf{2.76} & $\times$ \textbf{2.33} \\
\midrule

 & Standard Graph VAE & 20.52 $\pm$  2.97  & - & 856.05 & 856.05 & - \\
 & \underline{Var. FastGAE (degree, $\alpha = 1$)} &  &   &   &  &  \\
\textbf{Pubmed} & - with $n_{(S)} =$ 500 & 16.86 $\pm$ 4.84 & \textbf{0.01}  & \textbf{3.17} & \textbf{3.18} & $\times$ \textbf{269.20}\\
 & - with $n_{(S)} = n^*_{(S)} =$ 1187 &  18.84 $\pm$ 4.78 & 0.01  & 3.61  & 3.62 & $\times$ 236.49 \\
& - with $n_{(S)} =$ 5000 & \textbf{22.81} $\pm$ \textbf{4.80} & 0.01 & 14.95 & 14.96 & $\times$ 57.22 \\
& \underline{Best baseline} &   &   &  &  & \\
& Core-Graph VAE, $k=2$  & \textbf{23.56} $\pm$ \textbf{3.12} & - & 50.11 & 50.11 & $\times$ 17.08 \\
\midrule

 & Standard Graph VAE  & 25.43 $\pm$ 1.47  & - & 970.67 & 970.67 & - \\
 & \underline{Var. FastGAE (degree, $\alpha = 1$)} &  &  &  & & \\
\textbf{Pubmed} & - with $n_{(S)} =$ 500 & 29.04 $\pm$ 4.17 & \textbf{0.01} & \textbf{4.03} & \textbf{4.04} & $\times$ \textbf{240.26}\\
\textbf{with} & - with $n_{(S)} = n^*_{(S)} =$ 1187 & \textbf{31.11} $\pm$ \textbf{3.27} & 0.01 & 4.65 & 4.66 & $\times$ 208.30\\
\textbf{features}& - with $n_{(S)} =$ 5000 & \textbf{30.89} $\pm$  \textbf{3.01} & 0.01 & 20.01 & 20.02 & $\times$ 48.49\\
 & \underline{Best baseline}&  &  &  & & \\
 & Core-Graph VAE, $k=2$ & 24.35 $\pm$ 1.55 & - & 57.09 & 57.09 & $\times$ 17.00 \\
 \midrule

 & Standard Graph VAE & \textit{(intractable)}  & \multicolumn{3}{c|}{\textit{(intractable)}} & - \\
& \underline{Var. FastGAE (degree, $\alpha = 2$)} &  &  &  & & \\
\textbf{SBM} & - with $n_{(S)} =$ 2500 & 30.77 $\pm$ 0.32 & \textbf{0.03} & \textbf{52.01} & \textbf{52.04} & - \\
& - with $n_{(S)} = n^*_{(S)} =$ 2673 & 30.89 $\pm$ 0.30 & 0.03 & 53.98 & 54.01 & - \\
& -with $n_{(S)} =$ 5000 & 32.28 $\pm$ 0.26 & 0.03 & 61.96 & 61.69 & - \\
& \underline{Best baseline} &  &  &  & & \\
& Louvain & \textbf{35.90} $\pm$ \textbf{0.14} & - & 464.11 & 464.11 & - \\
\bottomrule
\end{tabular}
\end{tiny}
\caption{Node clustering on all graphs with communities. For each graph, for brevity, we only report the \textbf{best} graph AE \textbf{or} VAE model in terms of mean AMI, a few representative degree-based FastGAE versions of this model, and the best baseline (among Core-Graph AE/VAE, Negative Sampling Graph AE/VAE, node2vec, Louvain and the spectral embedding). Scores are averaged over 100 runs (resp. 10 runs) for medium-size graphs (resp. for the large graph SBM).  For degree sampling, values of the hyperparameter $\alpha$ (as defined in equation 9) were tuned, as described in Figure C.6. All models learn embedding vectors of dimension $d=16$. Bold numbers correspond to the best performance (several numbers are bold when scores are comparable, in a $\pm 1$ standard deviation range) and best running time.}
\end{table*}

\clearpage

\begin{figure*}[!ht]
\centering
  \subfigure[Cora]{
  \scalebox{0.35}{\includegraphics{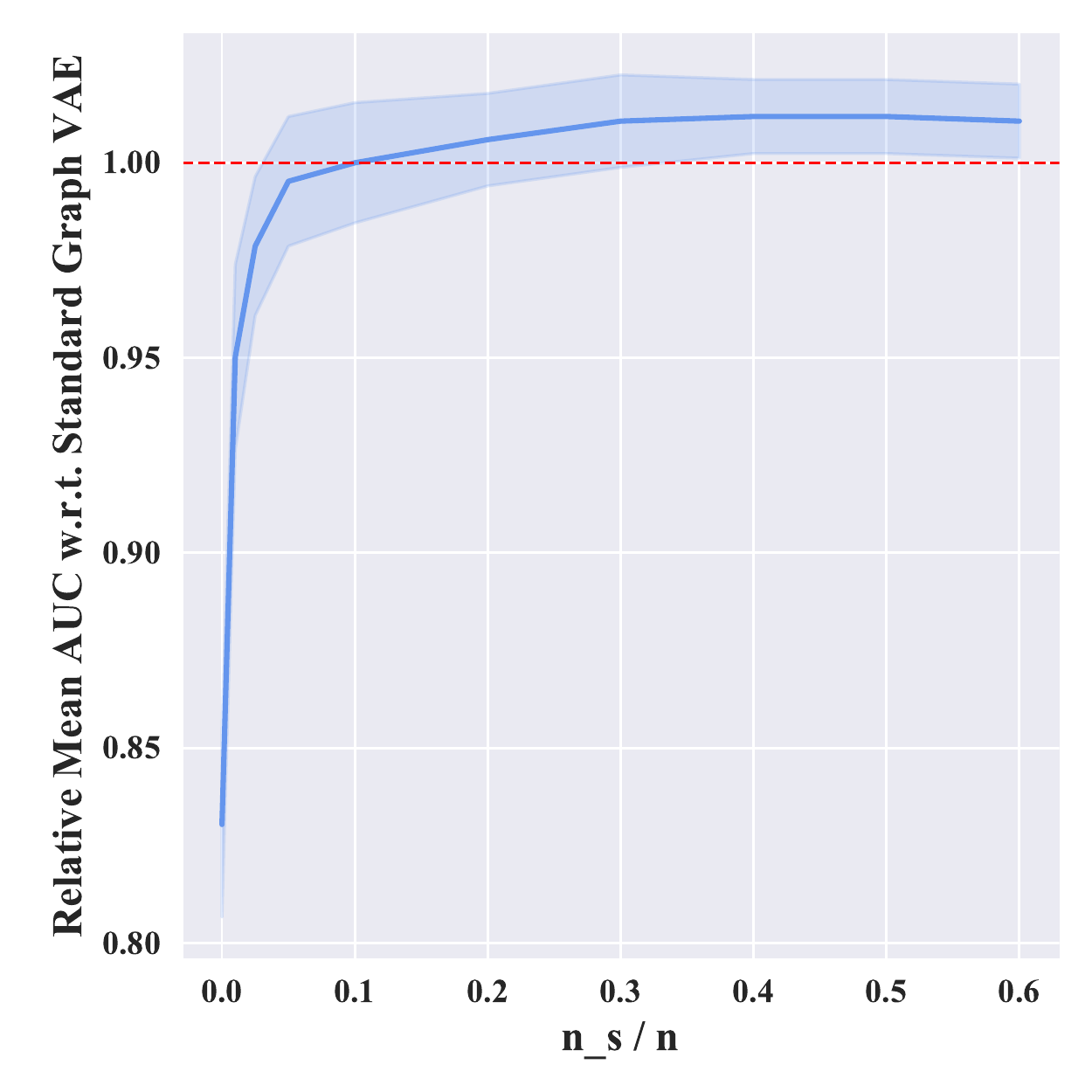}}}\subfigure[Citeseer]{
  \scalebox{0.35}{\includegraphics{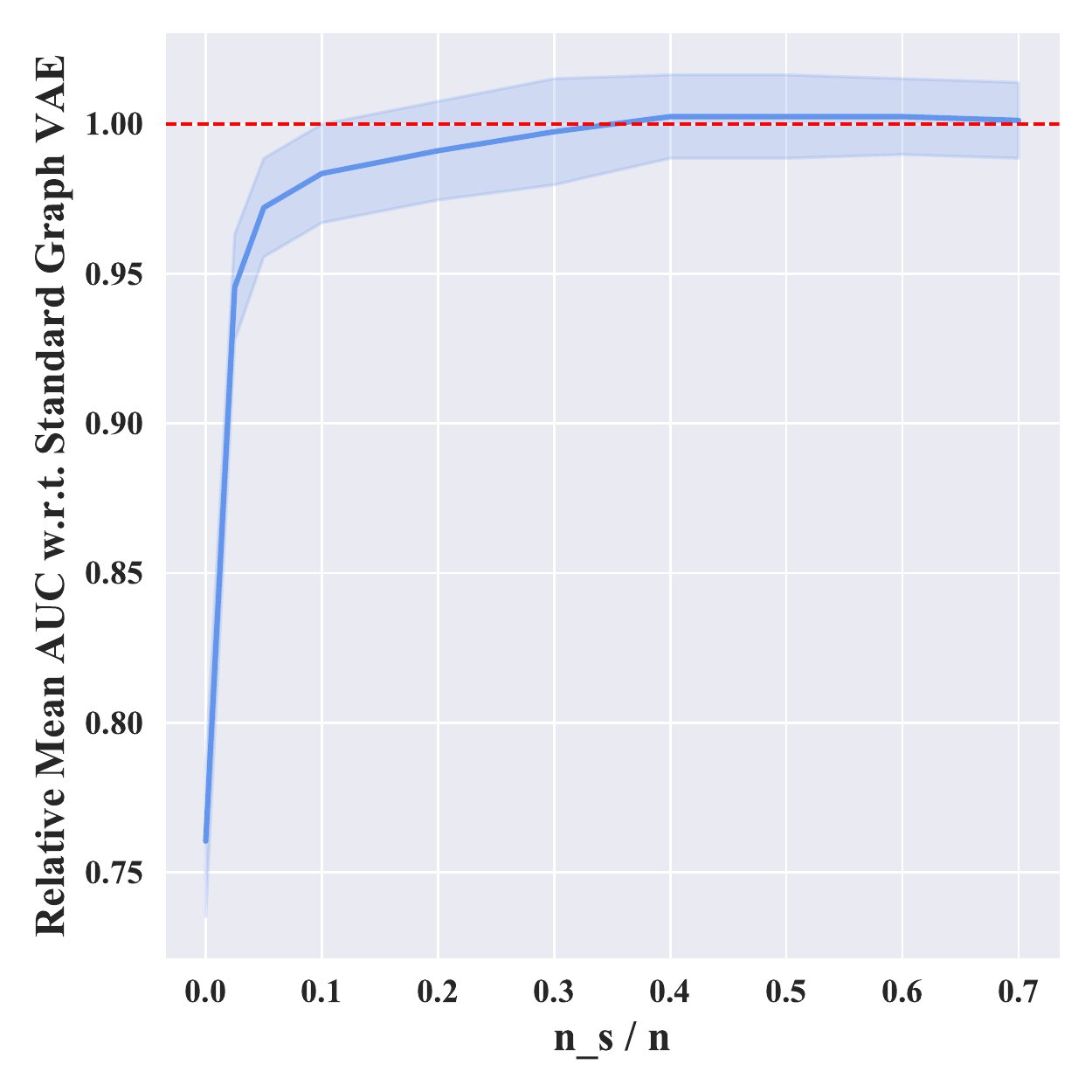}}}\subfigure[Pubmed]{
  \scalebox{0.35}{\includegraphics{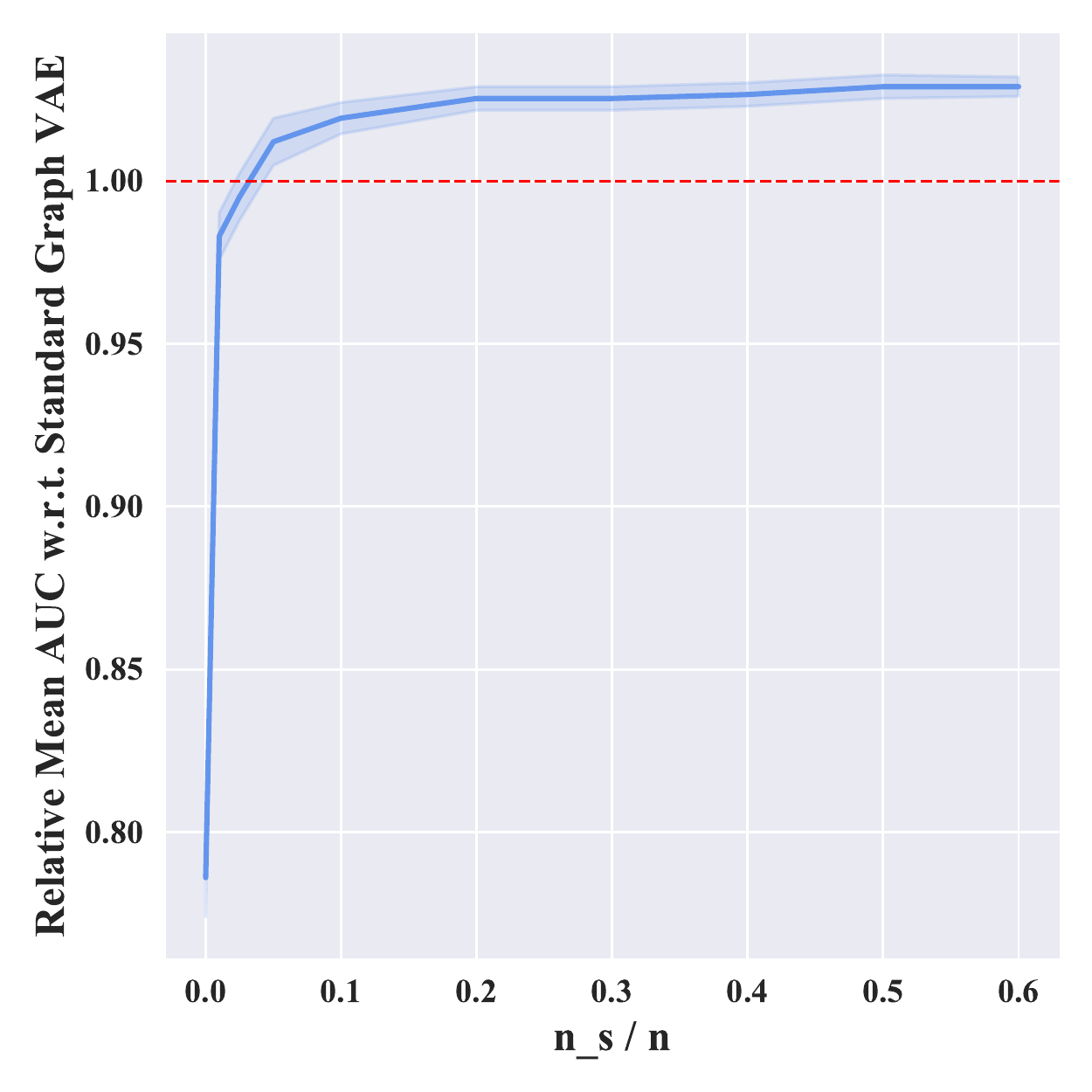}}}
   \subfigure[Cora (with features)]{
  \scalebox{0.35}{\includegraphics{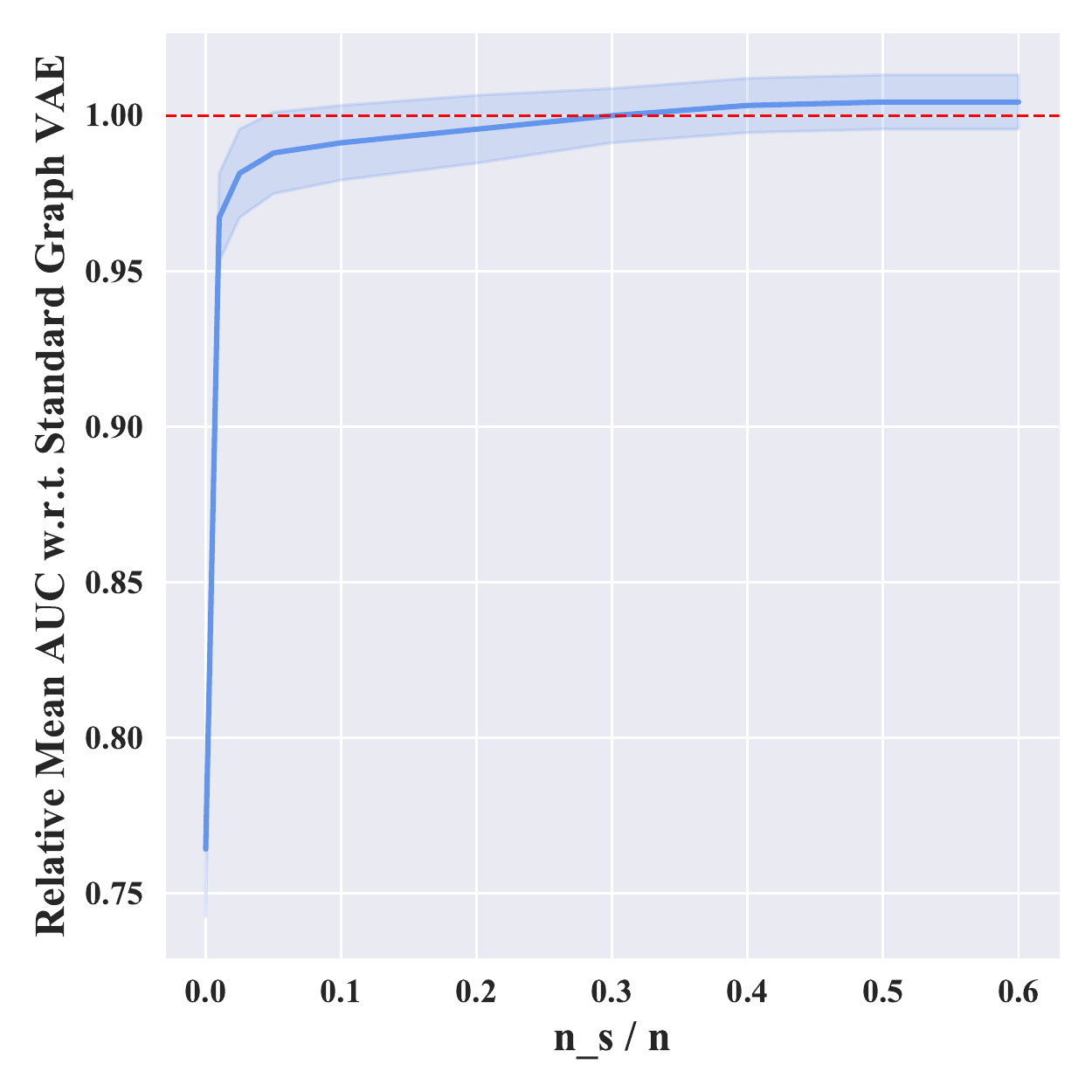}}}\subfigure[Citeseer (with features)]{
  \scalebox{0.35}{\includegraphics{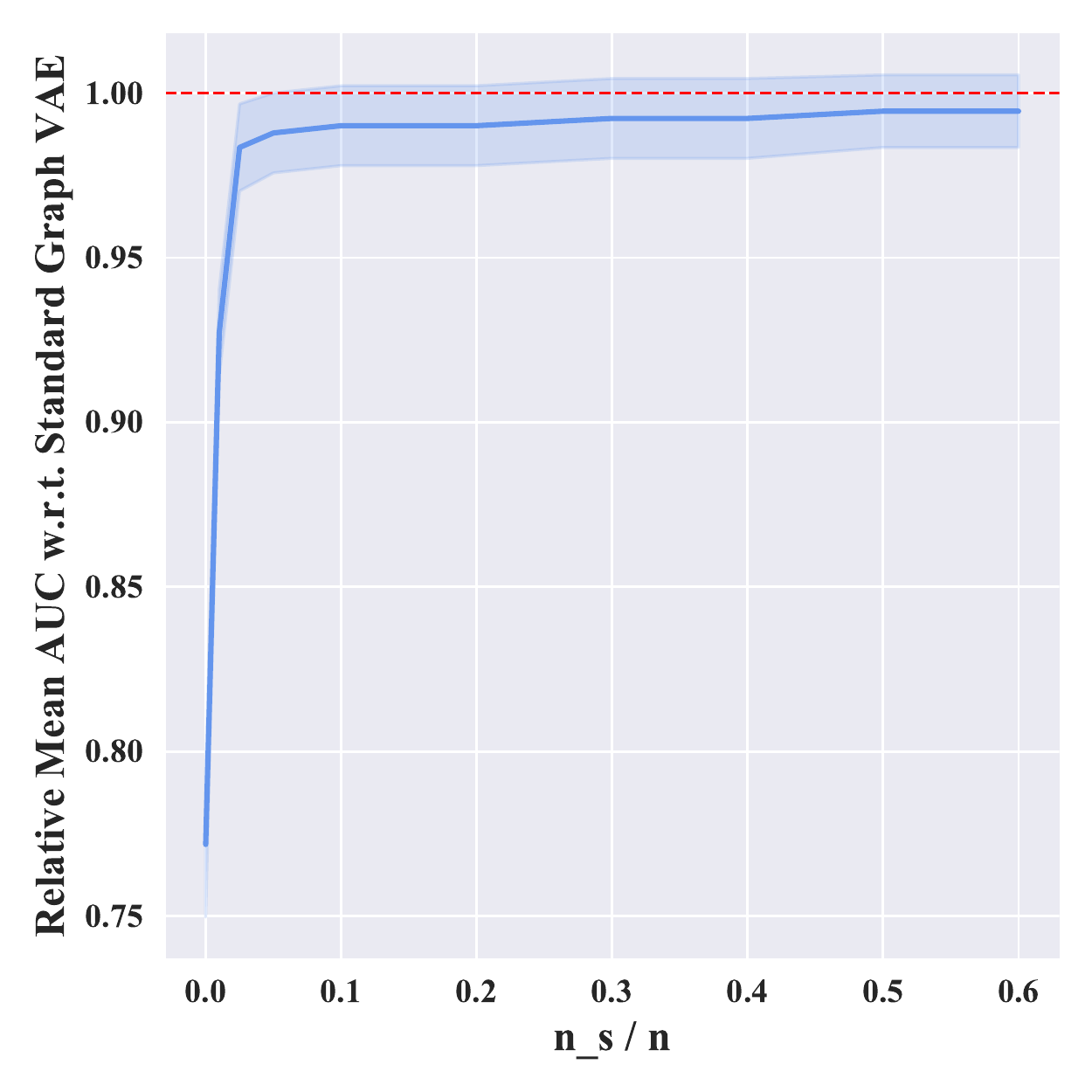}}}\subfigure[Pubmed (with features)]{
  \scalebox{0.35}{\includegraphics{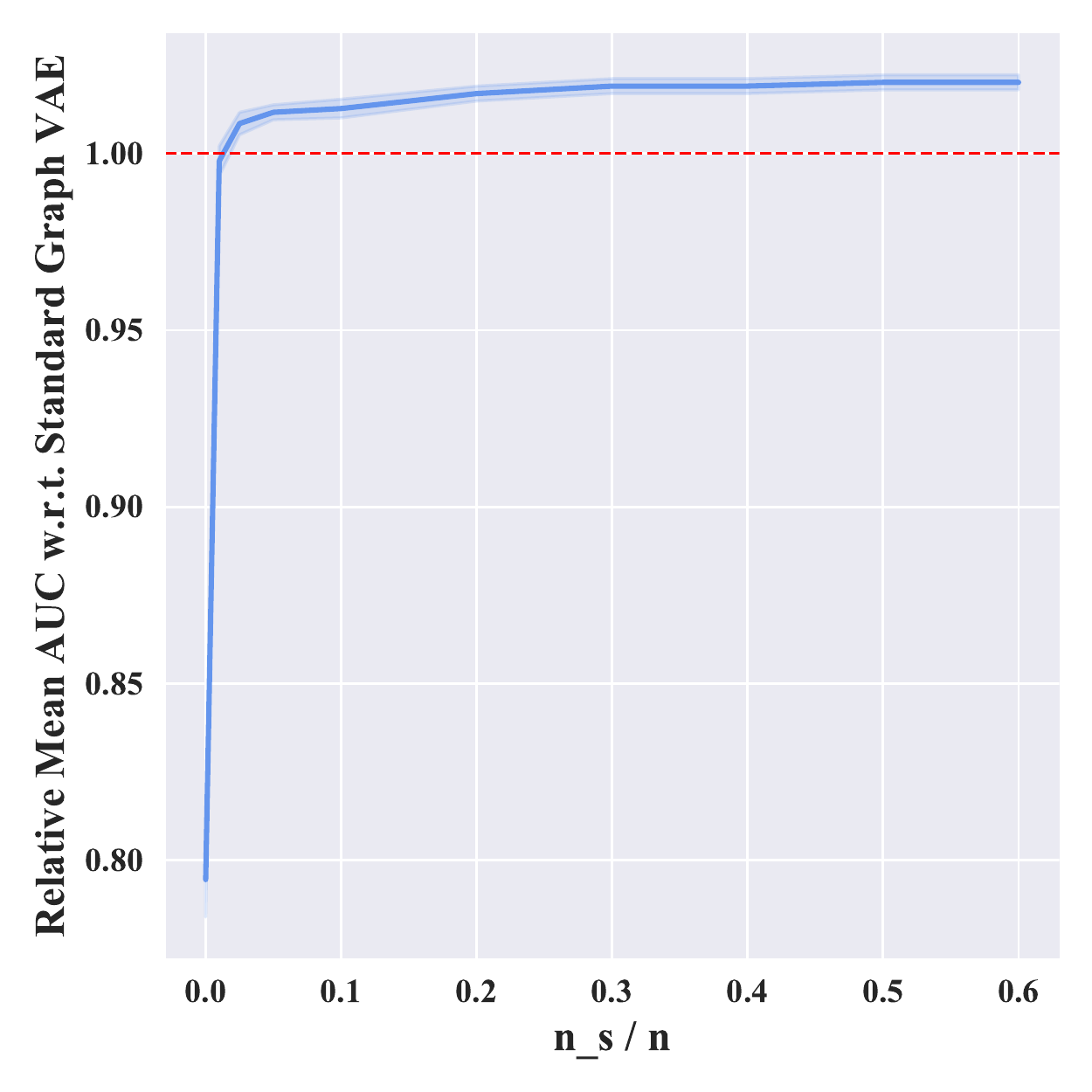}}}
  \caption{Summarized results for link prediction on the medium-size graphs Cora, Citeseer and Pubmed: relative mean AUC scores of degree-based Variational FastGAE models w.r.t. standard Graph VAE depending on the proportion of sampled nodes $n_{(S)}/n$ in decoders. We observe that, even for relatively low $n_{(S)}/n$ proportions, Variational FastGAE achieves comparable or even slightly better performances w.r.t. standard Graph VAE (results above the red line).}
\end{figure*}

\paragraph{Comparison of Uniform, Core-based and Degree-based FastGAE} In all our experiments, we observe that FastGAE with core and degree sampling both outperform FastGAE (and variational FastGAE) with uniform sampling. Furthermore, core and degree sampling also return more stable scores, i.e. with lower standard errors, especially when the number of samples $n_{(S)}$ is relatively small. Such results confirm the empirical superiority of strategies that leverage the graph structure w.r.t. pure random strategies.

\paragraph{FastGAE vs Baselines} In Table 2, Table 3 and Table 4, these models also outperform the other few existing methods to scale graph AE and VAE, usually by a wide margin. For instance, in Table 2, we show that, to achieve (almost) comparable link prediction performances w.r.t. FastGAE on Pubmed, \textit{Core-GAE} \cite{salha2019-1} requires longer running times (see \textit{Core-GAE} with $k=2$), and that faster variants significantly underperform (almost -20 AUC points for \textit{Core-GAE} with $k=9$ w.r.t. FastGAE with degree sampling). FastGAE is also conceptually simpler than Core-GAE, which we consider to be another advantage of our approach.

Besides, FastGAE-based models are faster and more effective than the ones leveraging negative sampling~\cite{pytorchgeometric} (e.g. +3.63 AUC points for FastGAE with degree sampling and $n_{(S)} =$ 20000 w.r.t. \textit{Negative Sampling GAE} in Table 2). This performance gain might be explained by the more systematic inclusion of \textit{unconnected pairs of important nodes}\footnote{Indeed, when performing negative sampling for graph AE, we only reconstruct a few random unconnected node pairs, and ignore the others. However, reconstructing some of these neglected pairs might actually be crucial. Let us consider two nodes with high core number or centrality: knowing that these two important nodes are \textit{not} connected is critical to learn meaningful embeddings. The FastGAE sampling scheme ensures a more systematic inclusion of these important "negative pairs" in the decoding step than negative sampling.} in the losses of FastGAE-based models.

\begin{figure*}[t]
\centering
  \subfigure[Cora]{
  \scalebox{0.35}{\includegraphics{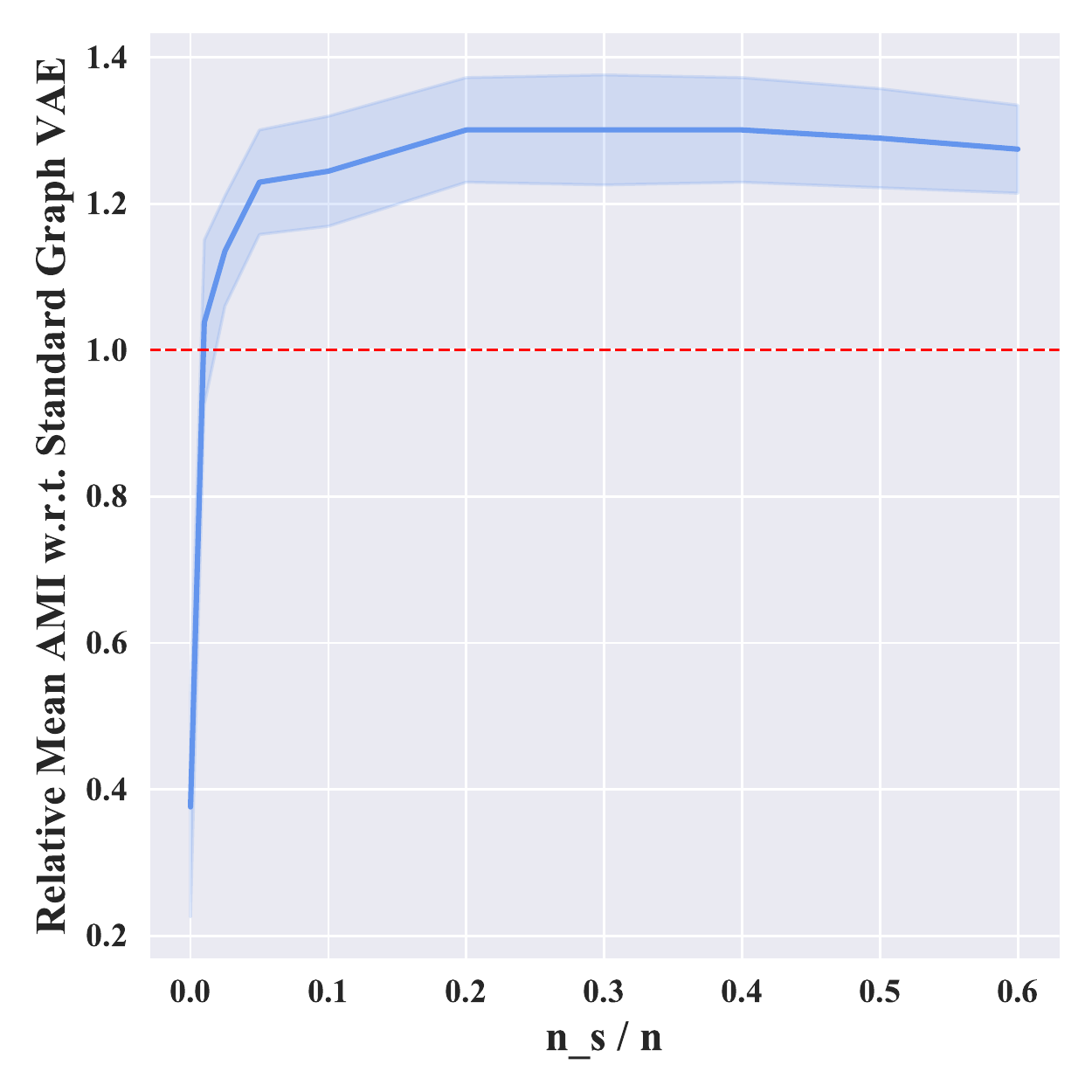}}}\subfigure[Citeseer]{
  \scalebox{0.35}{\includegraphics{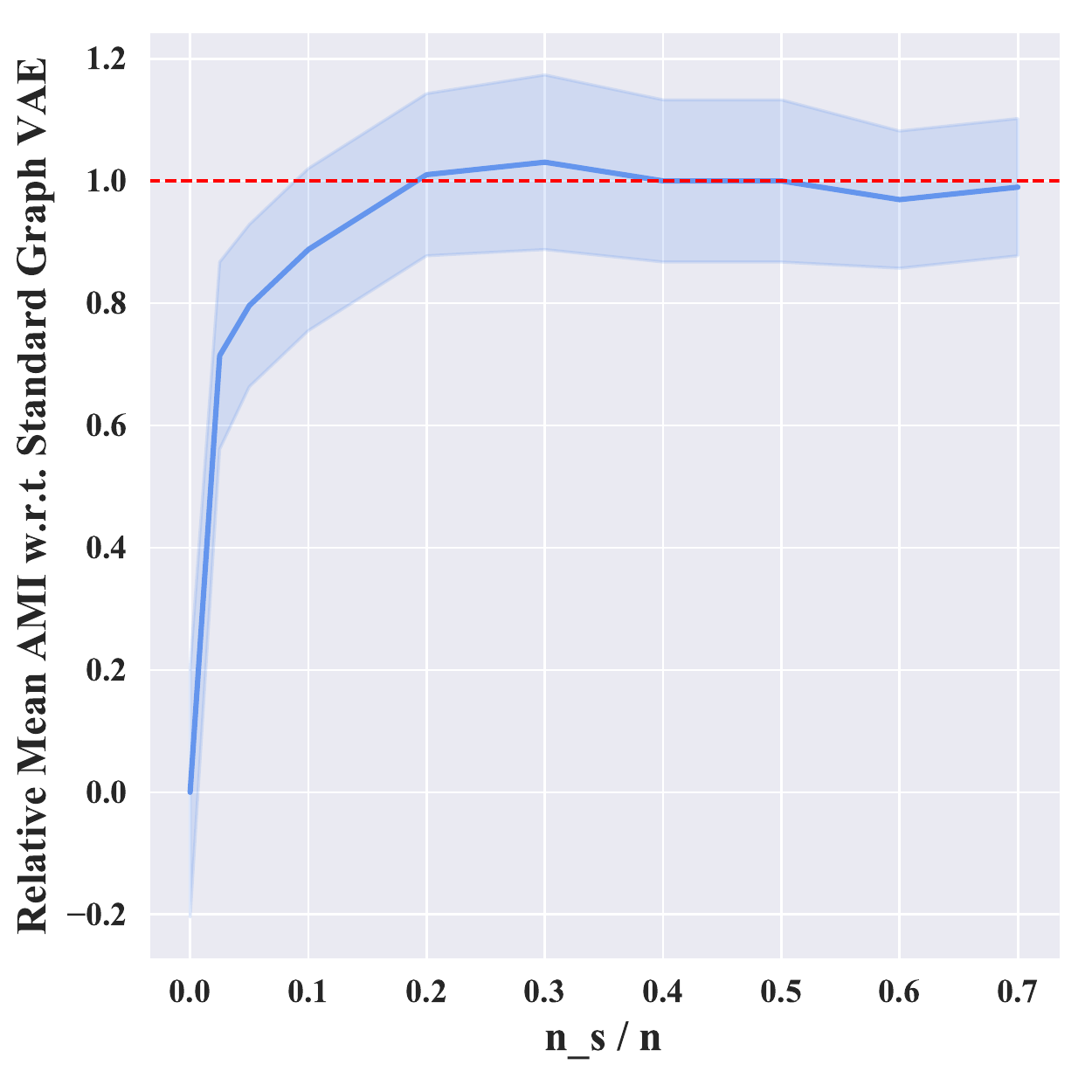}}}\subfigure[Pubmed]{
  \scalebox{0.35}{\includegraphics{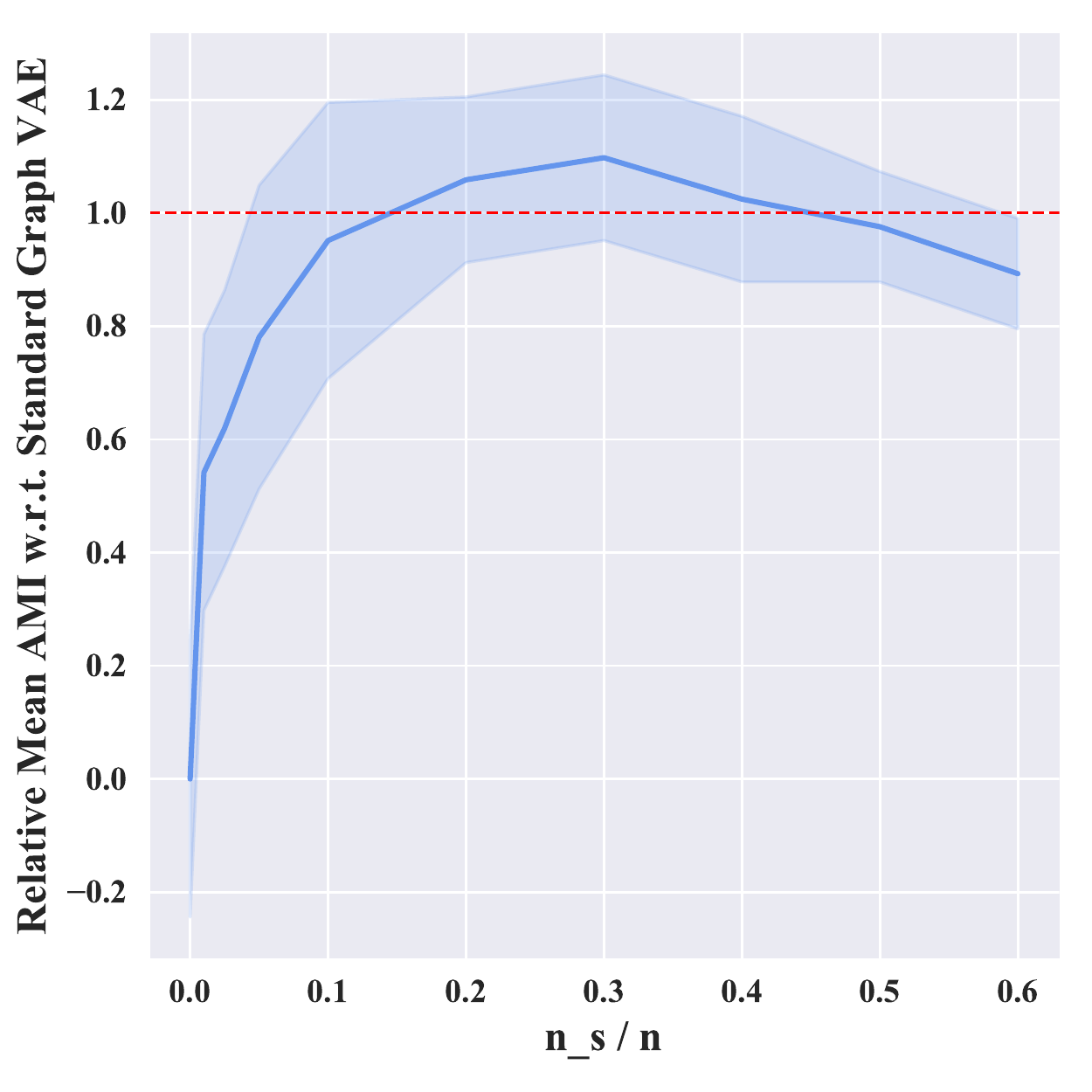}}}
  \subfigure[Cora (with features)]{
  \scalebox{0.35}{\includegraphics{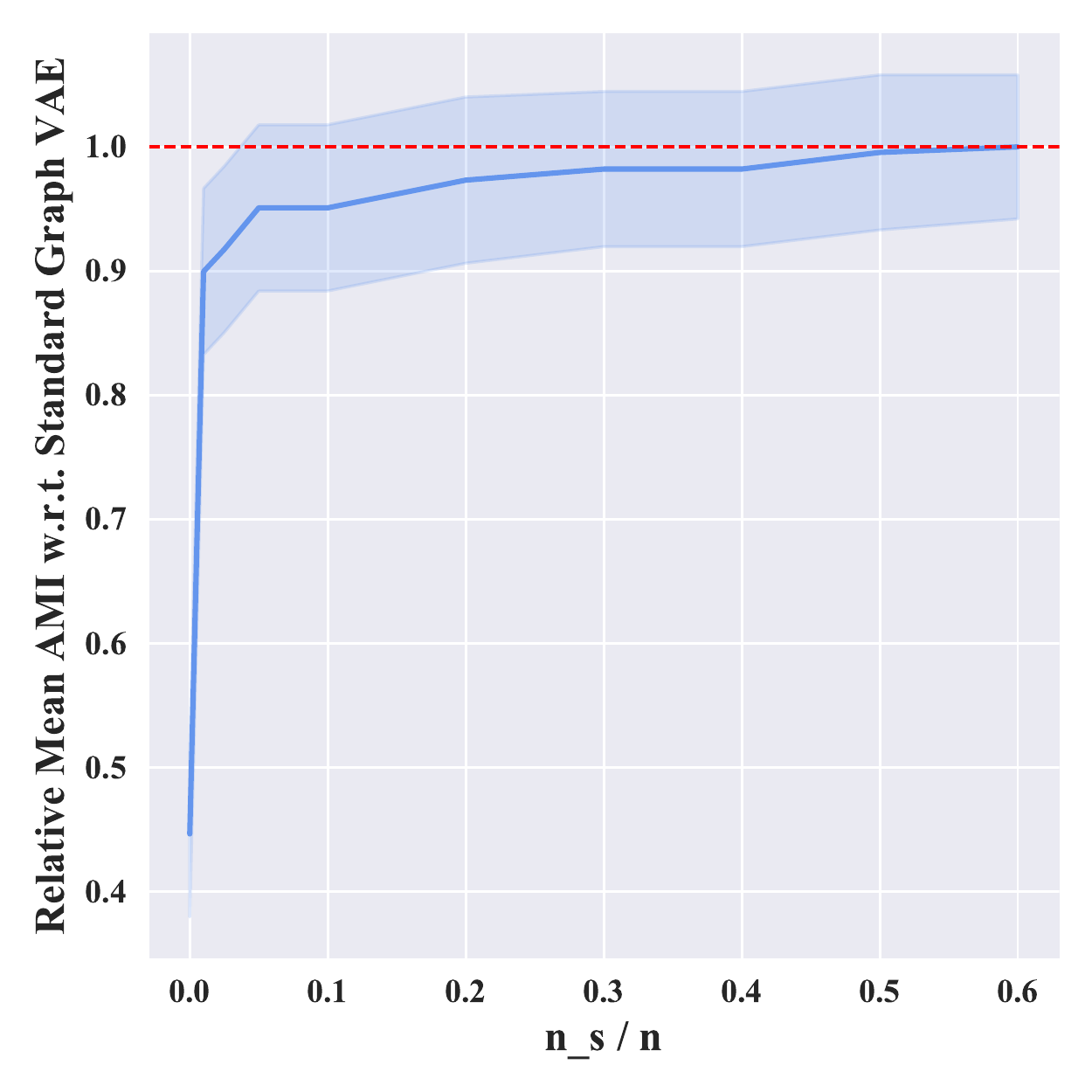}}}\subfigure[Citeseer (with features)]{
  \scalebox{0.35}{\includegraphics{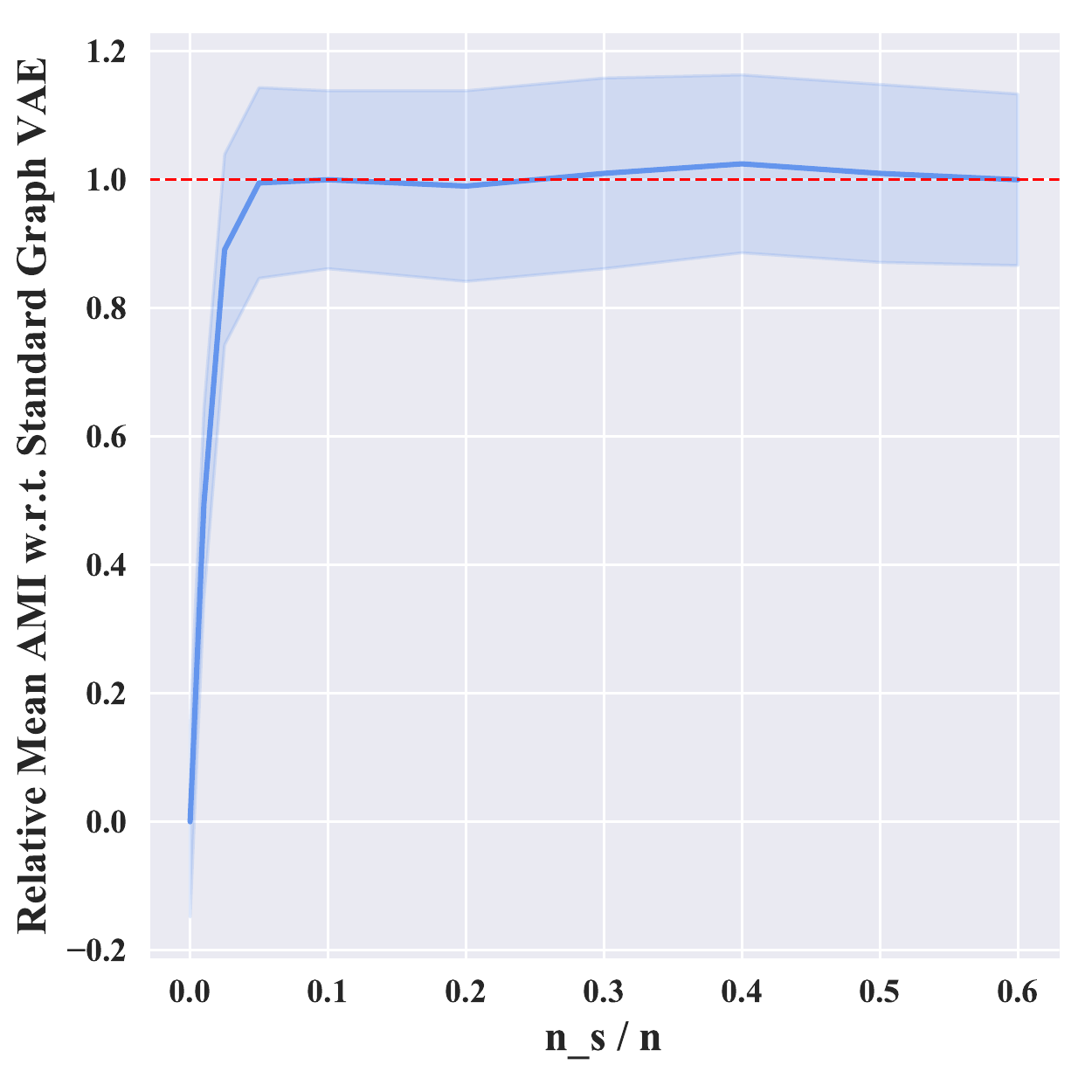}}}\subfigure[Pubmed (with features)]{
  \scalebox{0.35}{\includegraphics{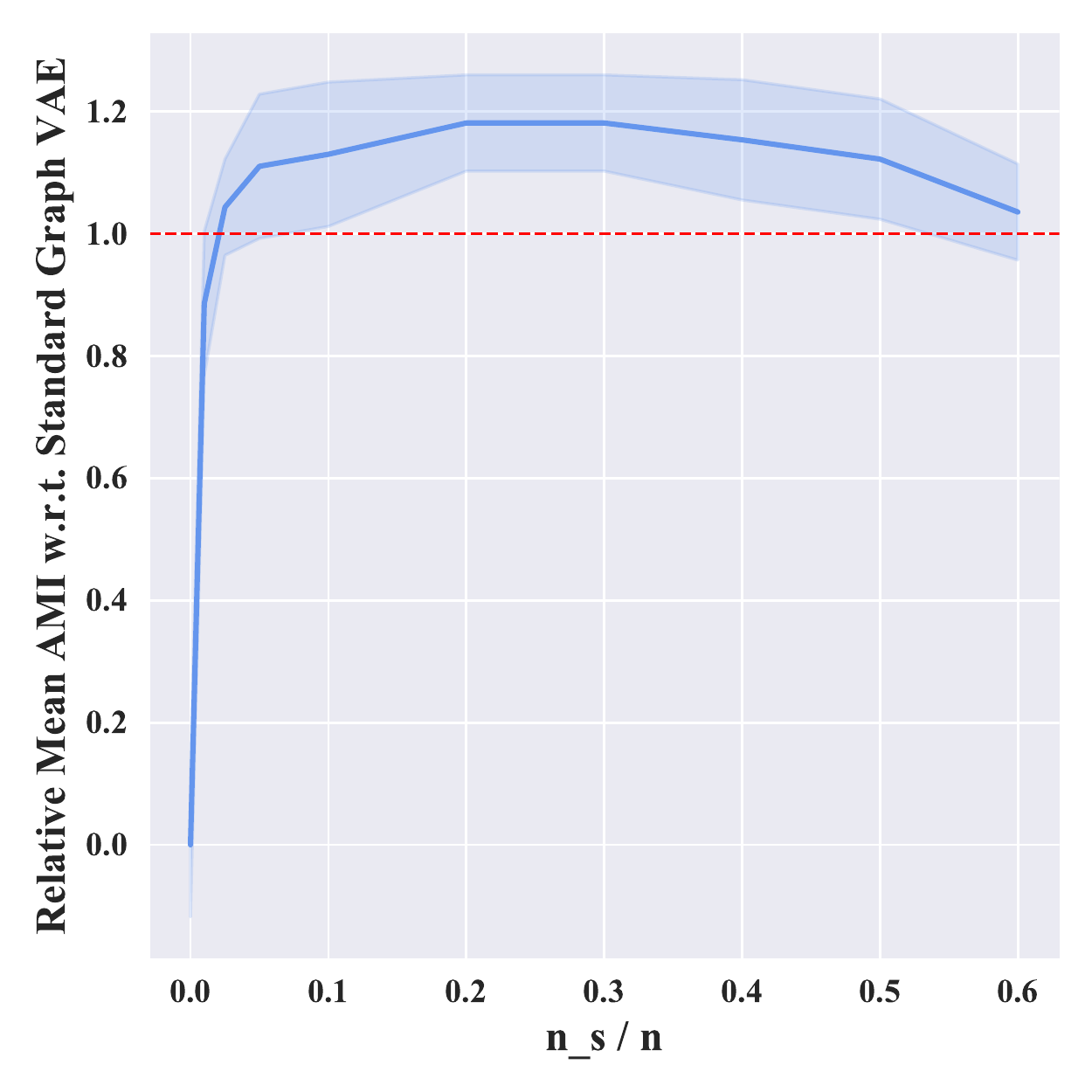}}}
  \caption{Summarized results for node clustering on the medium-size graphs Cora, Citeseer and Pubmed: relative mean AMI scores of degree-based Variational FastGAE models w.r.t. standard Graph VAE depending on the proportion of sampled nodes $n_{(S)}/n$ in decoders. We observe that, even for relatively low $n_{(S)}/n$ proportions, Variational FastGAE achieves comparable or even slightly better performances w.r.t. standard Graph VAE (results above the red line).}
\end{figure*}

Last, but not least, our proposed framework is also competitive w.r.t. the popular non AE/VAE-based baselines in most cases. The only exception concerns the node clustering experiments on Cora and Citeseer (see Table 4) where the Louvain baseline~\cite{blondel2008louvain} outperfoms AE/VAE models, which we will further discuss in Section 4.2.3.

\paragraph{On the hyperparameter $\alpha$} In Appendix C, we report optimal values of $\alpha$ for all graphs. We recall that $\alpha \in \mathbb{R}^+$ is the hyperparameter introduced in equation 9, that helps balancing important and "less important" nodes during sampling. Setting $\alpha = 0$, leads to the uniform sampling setting where all nodes are sampled with an equal probability. On the contrary, by setting $\alpha \rightarrow \infty$ we would always sample the most important nodes and ignore others. Experiments from Figure C.6 from Appendix C show that these two extreme cases are usually sub-optimal, and that a careful tuning of $\alpha$ (e.g. $\alpha = 2$ for core sampling in Table~2) improves performances.

\paragraph{On the threshold $n^*_{(S)}$} In Section 3.3, we introduced a theoretically-grounded threshold $n^*_{(S)} = C \sqrt{n}$ to select the subgraph size. Overall, in all our experiments (see Table 3, Table 4 and Table 5), selecting the proposed $n^*_{(S)}$ provided interesting performance/speed trade-offs, leading to fairly competitive results w.r.t. standard graph AE/VAE and best baselines, while being significantly faster.

\subsubsection{FastGAE for Large Graphs}

\begin{table*}[t]
\centering
\begin{tiny}
\begin{tabular}{c|c|cc|ccc}
\toprule
\textbf{Model}  & \textbf{Subgraphs} & \multicolumn{2}{c}{\textbf{Average Perf. on Test Set}} & \multicolumn{3}{c}{\textbf{Average Running Times (in seconds)}}\\
(Dimension $d=16$) & \textbf{size $n_{(S)}$} & \tiny \textbf{AUC (in \%)} & \tiny \textbf{AP (in \%)} & \tiny Compute  & \tiny Train& \tiny \textbf{Total} \\ 
 &  & \tiny & \tiny \ & \tiny $p_i$ & \tiny model &  \\ 
\midrule
\midrule
Standard Graph AE  & - &  \multicolumn{2}{c|}{\textit{(intractable)}}  & \multicolumn{3}{c}{\textit{(intractable)}} \\
\midrule
FastGAE with & 20000 & \textbf{92.91} $\pm$ \textbf{0.22} & \textbf{93.35} $\pm$ \textbf{0.21} & 0.30 & 4401.67 & 4401.97 (1h13)  \\
\textbf{degree} sampling   & 16425$^{*}$ & \textbf{93.02} $\pm$ \textbf{0.23} & \textbf{93.39} $\pm$ \textbf{0.23} & 0.30 & 3693.32 & 3693.62 (1h02)  \\
($\alpha = 2$)& 10000 & 91.76 $\pm$ 0.23 & 91.74 $\pm$ 0.21 & 0.30 & 1164.22 & 1164.52 (19 min) \\
& 2500 & 87.53 $\pm$ 0.50 & 87.42 $\pm$ 0.51 & 0.30 & 537.99 & 538.29 (9 min)  \\
& 1000 & 85.55 $\pm$ 0.62 & 85.96 $\pm$ 0.55 & 0.30 & 500.12 & 500.42 (8 min) \\
\midrule
FastGAE with & 20000 & 90.71 $\pm$ 0.21 & 91.70 $\pm$ 0.19 & 668.05 & 4800.58 & 5468.63 (1h31)  \\
\textbf{core} sampling & 16425$^{*}$ & 90.48 $\pm$ 0.21 & 90.85 $\pm$ 0.23 & 668.05 & 4027.90 & 4695.95 (1h18) \\
($\alpha = 2$) & 10000 & 89.08 $\pm$ 0.25 & 88.65 $\pm$ 0.24 & 668.05 & 1232.03 & 1900.08 (32 min)  \\
& 2500 & 82.50 $\pm$ 0.51 & 81.42 $\pm$ 0.60 & 668.05 & 544.64 & 1222.69 (20 min)  \\
 & 1000 & 73.99 $\pm$ 0.70 & 75.24 $\pm$ 0.74 & 668.05 & 503.88 & 1171.93 (19 min)  \\
\midrule
FastGAE with & 20000 & 85.97 $\pm$ 0.26 & 87.71 $\pm$ 0.25 & - & 4397.89 & 4387.89 (1h13)  \\
\textbf{uniform} sampling & 16425$^{*}$ & 84.40 $\pm$ 0.25 & 86.11 $\pm$ 0.25 & - & 3602.66 & 3602.66 (1h) \\
 & 10000 & 83.77 $\pm$ 0.28 & 83.37 $\pm$ 0.26 & - & 1106.01 & 1106.01 (18 min)  \\
& 2500 & 70.66 $\pm$ 0.35 & 71.16 $\pm$ 0.38 & - & 485.03 & 485.03 (8 min)  \\
& 1000 & 59.34 $\pm$ 0.83 & 58.83 $\pm$ 1.30 & - & 438.02 & \textbf{438.02 (7 min)} \\
\midrule
\midrule
Core-GAE, $k=14$ (best choice) & - & 88.06 $\pm$ 0.27 & 88.94 $\pm$ 0.23 & - & 4805.11 & 4805.11 (1h20) \\
Core-GAE, $k=21$ (fastest choice) & - & 86.94 $\pm$ 0.69 & 87.23 $\pm$ 0.71 & - & 619.01 & 619.01 (10 min)  \\
Negative Sampling GAE & - & 86.11 $\pm$ 0.48 & 86.70 $\pm$ 0.49 & - & 2392.96 & 2392.96 (40 min)  \\
node2vec & - & \textbf{92.96} $\pm$ \textbf{0.23} & \textbf{93.43} $\pm$ \textbf{0.17} & - & 25851.39 & 25851.39 (7h11)  \\
Spectral Embedding & - & \multicolumn{2}{c|}{\textit{(intractable)}}   & \multicolumn{3}{c}{\textit{(intractable)}} \\
\bottomrule
\end{tabular}
\end{tiny}
\caption{Link prediction on Patent ($n =$ 3774768, $m =$ 16518948), using FastGAE with degree, core and uniform sampling, and baselines. Standard Graph AE is intractable. For degree and core sampling, values of the hyperparameter $\alpha$ (as defined in equation 9) were tuned, as described in Figure C.6. All models learn embedding vectors of dimension $d=16$. Scores are averaged over 10 runs with different and random train/validation/test sets. Bold numbers correspond to the best performance (several numbers are bold when scores are comparable, in a $\pm 1$ standard deviation range) and best running time. Subgraphs sizes annotated with $^*$ correspond to the $n^*_{(S)}$ threshold, as introduced in equation 19.}
\end{table*}

After studying medium-size graphs, we now report in this section the evaluation of FastGAE and variational FastGAE on the four \textit{large graphs} from our experiments: SBM, Google, Youtube and Patent. The above Table 5 details mean AUC  and  AP scores and standard errors over 10 runs with different train/test splits for link prediction on the Patent graph with FastGAE. We also report more summarized results (for the sake of brevity) for link prediction on SBM, Google, Youtube and Patent in Figure 5 and in Table~6, and summarized results for node clustering on SBM in Figure 5 and in the previous Table 4. As in Table 5, all scores are averaged over 10 runs with different train/test splits.

\begin{figure*}[!t]
\centering
     \subfigure[Google]{
  \scalebox{0.35}{\includegraphics{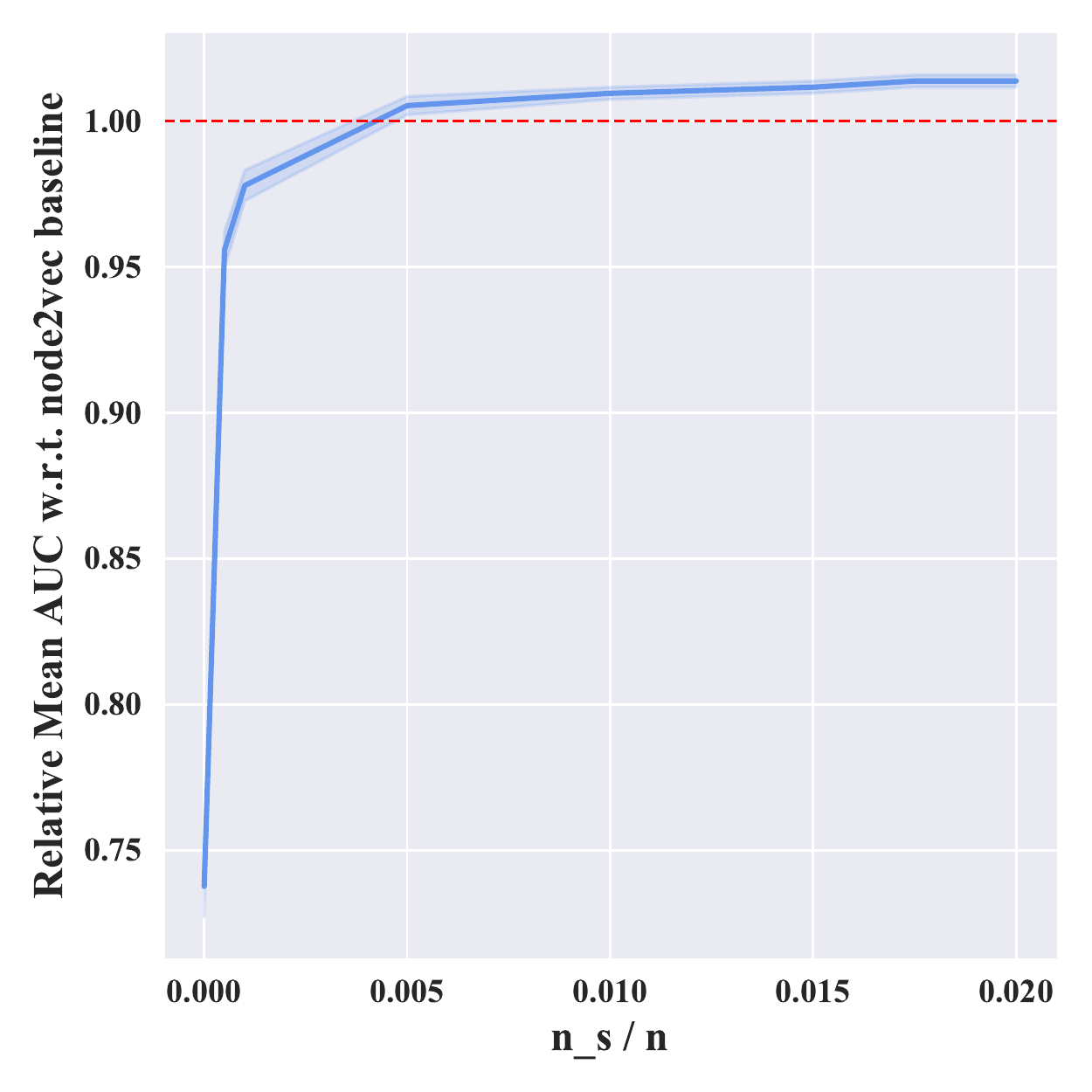}}}\subfigure[Youtube]{
  \scalebox{0.35}{\includegraphics{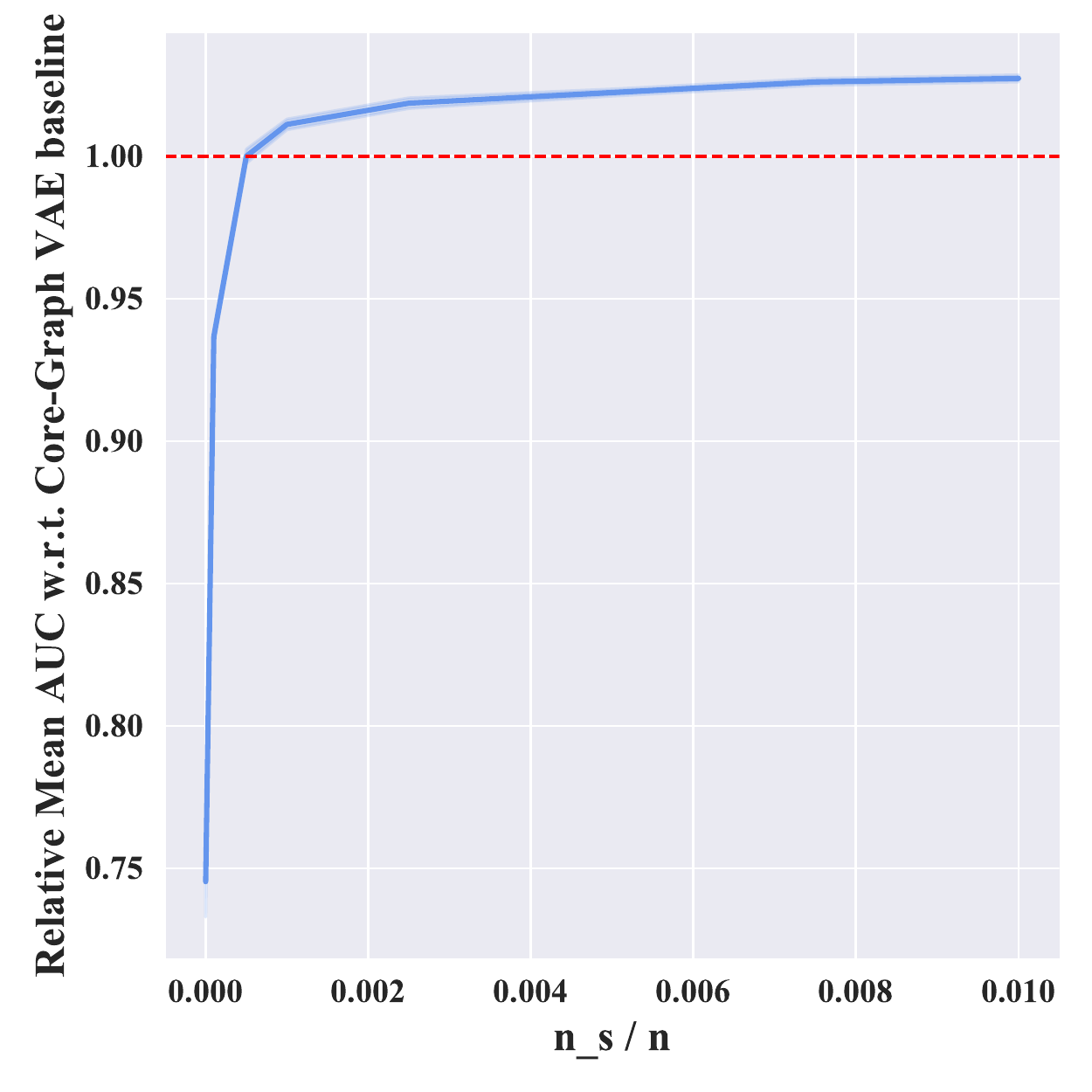}}}\subfigure[Patent]{
  \scalebox{0.35}{\includegraphics{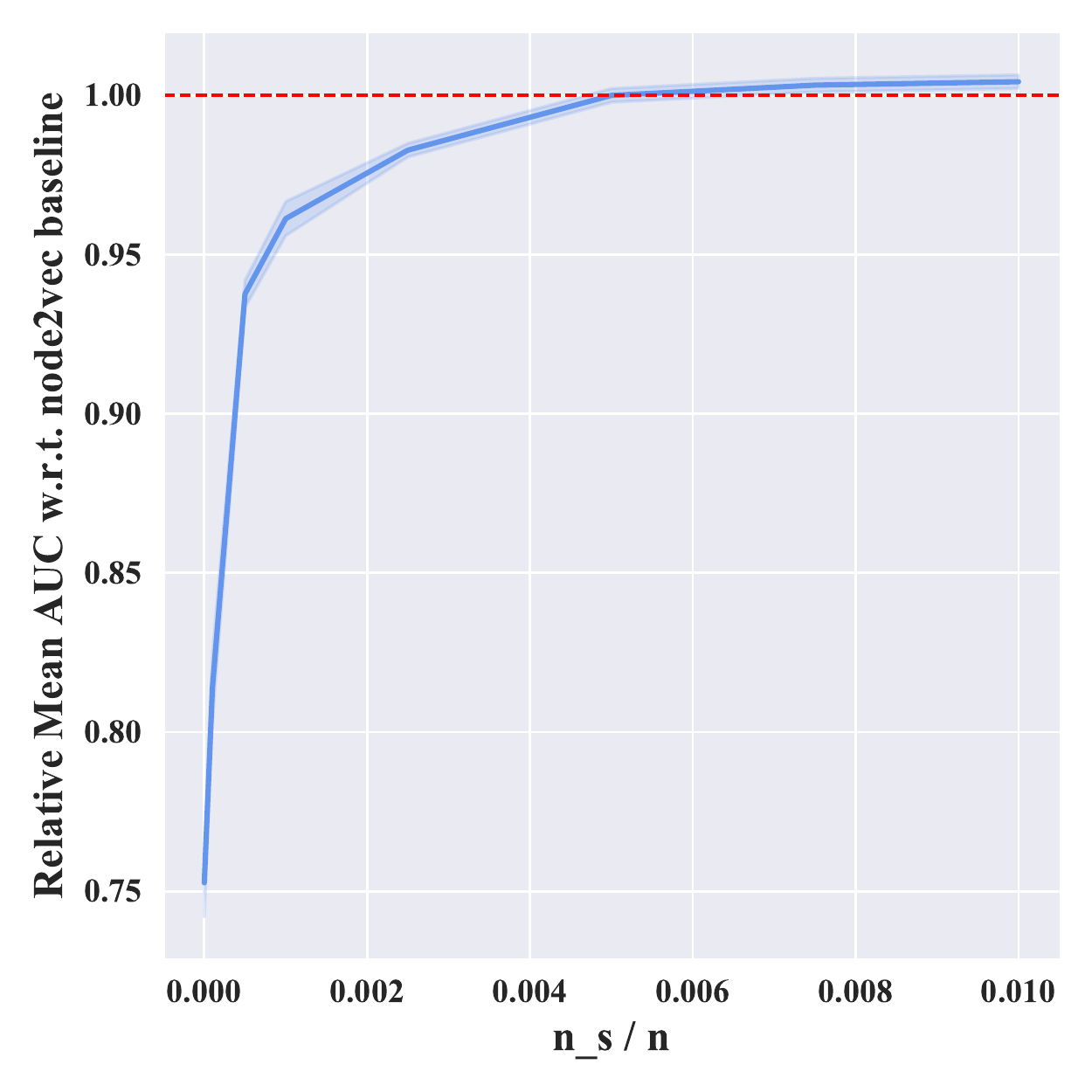}}}
  \subfigure[SBM]{
  \scalebox{0.35}{\includegraphics{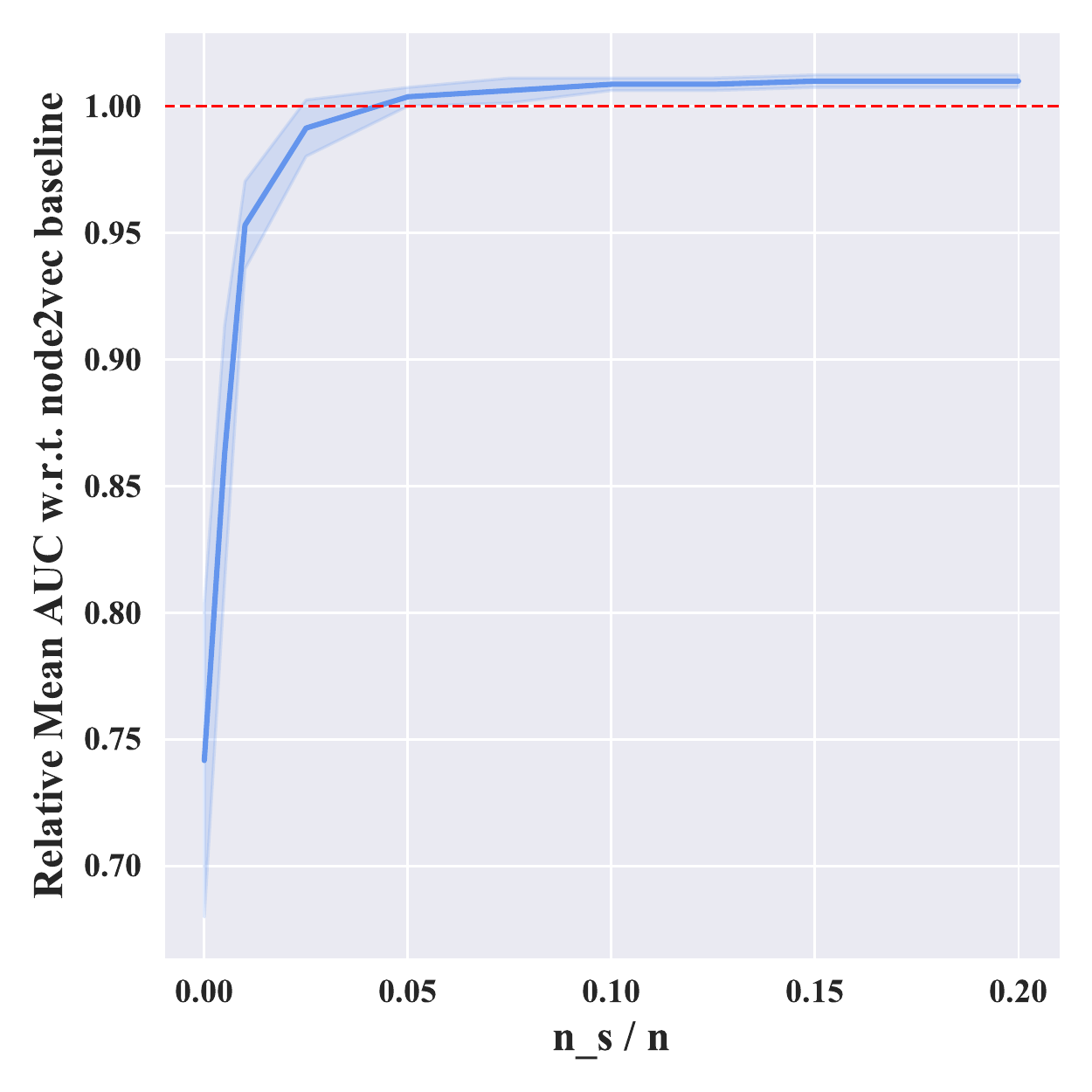}}}\subfigure[SBM (Clustering)]{\scalebox{0.35}{\includegraphics{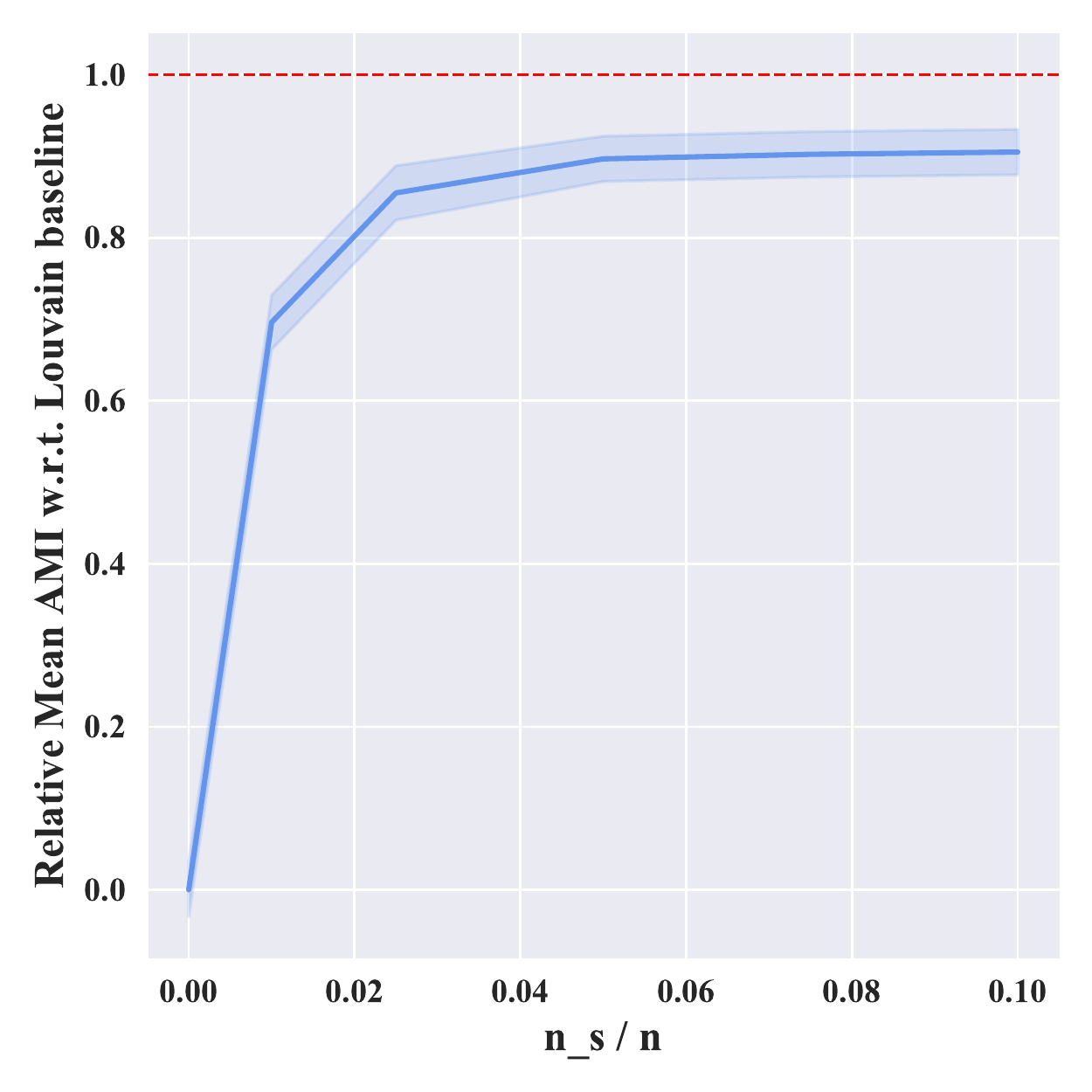}}}
  \caption{Summarized results for link prediction on the four large graphs SBM, Google, Youtube and Patent (subfigures a, b, c and d) and for node clustering on SBM (subfigure e): relative mean AUC scores (or mean AMI scores for subfigure e) of degree-based Variational FastGAE models w.r.t. the best scalable baseline, depending on the proportion of sampled nodes $n_{(S)}/n$ in decoders.}
\end{figure*}

\paragraph{FastGAE vs Scalable AE/VAE Baselines} On large graphs, direct comparison with standard graph AE and VAE is unfortunately impossible. However, our FastGAE and variational FastGAE models almost always outperform the other existing approaches to scale graph AE and VAE, usually by a wide margin. For instance, for link prediction on Patent (Table~5), degree-based and core-based FastGAE models with $n_{(S)} =$ 20000, 16425 and 10000 all outperform the best \textit{Core-GAE} by up to roughly 5 AUC points (for degree-based FastGAE with $n_{(S)} =$ 20000) and with comparable or better running times.

Regarding the \textit{Core-GAE} baseline \cite{salha2019-1}, we also point out that, in one of our large graphs, namely on the SBM one, this method was even \textit{intractable} due to the lack of \textit{size decreasing} core structure on this graph. Indeed, the $21$-core of SBM includes 95200 nodes, which is too large to train a graph AE or VAE on our machines, and the $22$-core is empty. Requiring a size decreasing core structure is a drawback of \textit{Core-GAE} w.r.t. the more flexible FastGAE approach.

Moreover, as for medium-size graphs, we also observe that core-based and degree-based FastGAE tend to significantly outperform negative sampling (e.g. up to +6.8 AUC points for link prediction on Patent in Table 5; also, \textit{Negative Sampling GAE} never appears as the best baseline in Table 4 nor in Table 6), consolidating our previous conclusions.

Besides, as before, the proposed $n^*_{(S)}$ provides quite effective performance/speed trade-offs and will constitute an interesting heuristic to help future FastGAE users selecting subgraph sizes.

\paragraph{Comparison of Uniform, Core-based and Degree-based FastGAE} As for medium-size graphs, core-based sampling and degree-based sampling is empirically more effective than uniform sampling (e.g. in Table 5, +6.94 AUC points for FastGAE with degree sampling on Patent, with $n_{(S)} =$ 20000), and associated to lower standard errors. We observe that computing the $p_i$ probabilities through core-based sampling is longer on large graphs, but bring no empirical benefit w.r.t. degree-based sampling: we therefore recommend using degree-based sampling for large graphs.

\paragraph{FastGAE vs non-AE/VAE baselines, and the case of Node Clustering} For the link prediction task, best FastGAE models usually reach competitive results w.r.t. node2vec while being significantly faster (see e.g. the last column of Table 6). However, regarding node clustering, we observe in Table 4 and in Figure 5 that the \textit{Louvain} baseline outperfoms AE/VAE models on SBM, a phenomenon that we also noted on the Cora and Citeseer graphs (section 4.2.2). We conjecture that current graph AE and VAE models might be suboptimal to effectively reconstruct communities in graph data ; this claim is consistent with recent experiments on these datasets \cite{salha2019-1,salha2020simple}.  As our objective, in this paper, was to scale existing graph AE/VAE, but not to ensure nor to claim their superiority over all other methods for node clustering, we do not further investigate this limit here. Nonetheless, future works on more effective cluster reconstruction from AE/VAE embeddings could definitely lead towards the improvement of these models.


\begin{table*}[!t]
\centering
\begin{tiny}
\begin{tabular}{c|l|cc|ccc}
\toprule
\textbf{Dataset}  & \textbf{Model}  & \multicolumn{2}{c}{\textbf{Average Perf. on Test Set}} & \multicolumn{3}{c}{\textbf{Average Running Times (in sec.)}} \\
& (Dimension $d=16$) & \tiny \textbf{AUC (in \%)} & \tiny \textbf{AP (in \%)} & \tiny Compute & \tiny Train & \tiny \textbf{Total} \\ 
&  &  & & \tiny $p_i$ & \tiny model & \\ 
\midrule
\midrule

 & Standard Graph VAE  & \multicolumn{2}{c|}{\textit{(intractable)}}  & \multicolumn{3}{c}{\textit{(intractable)}}  \\
 & \underline{Var. FastGAE (degree, $\alpha = 2$)}&  &  &  &  & \\
\textbf{SBM} & - with $n_{(S)} =$ 2000 & 79.37 $\pm$ 0.52 & 80.68 $\pm$ 0.84 & \textbf{0.03} & \textbf{27.36} & \textbf{27.39} \\
 & - with $n_{(S)} = n^*_{(S)} =$ 2673 & 80.96 $\pm$ 0.35 & 83.69 $\pm$ 0.60 & 0.03 & 30.66 & 30.69\\
 & - with $n_{(S)} =$ 5000 & \textbf{81.45} $\pm$ \textbf{0.39} & \textbf{84.30} $\pm$ \textbf{0.82} & 0.03 & 43.86 & 43.89 \\
& \underline{Best baseline} &   &  &  &  & \\
& node2vec & 80.89 $\pm$ 0.32 & 83.51 $\pm$ 0.29 & - & 1328.82 & 1328.82 (22 min) \\
\midrule

 & Standard Graph AE  & \multicolumn{2}{c|}{\textit{(intractable)}}  & \multicolumn{3}{c}{\textit{(intractable)}}\\
 & \underline{FastGAE (degree, $\alpha = 1$)}& &  &   &  & \\
\textbf{Google} & - with $n_{(S)} =$ 2500 & 94.52 $\pm$ 0.26 & 95.50 $\pm$ 0.11 & \textbf{0.14} & \textbf{122.53} & \textbf{122.67} \\
 & - with $n_{(S)} = n^*_{(S)} =$ 7911 & \textbf{95.75} $\pm$ \textbf{0.24} & \textbf{96.62} $\pm$ \textbf{0.09} & 0.14  & 158.63 & 158.77 \\
 & - with $n_{(S)} =$ 10000 & \textbf{95.91} $\pm$ \textbf{0.19} & \textbf{96.64} $\pm$ \textbf{0.12} & 0.14 & 168.10 & 168.24  \\
& \underline{Best baseline} &   &   &  &  & \\
& node2vec  & 94.89 $\pm$ 0.63 & \textbf{96.82} $\pm$ \textbf{0.72} & - & 14762.78 & 14762.78 (4h06)  \\
\midrule

 & Standard Graph VAE  & \multicolumn{2}{c|}{\textit{(intractable)}}  & \multicolumn{3}{c}{\textit{(intractable)}} \\
& \underline{Var. FastGAE (degree, $\alpha = 5$)} & &  &  &  & \\
\textbf{Youtube} & - with $n_{(S)} =$ 3000 & 81.14 $\pm$ 0.19 & 86.61 $\pm$ 0.16 & \textbf{0.28} & \textbf{453.22} & \textbf{453.50 (8min)}\\
 & - with $n_{(S)} = n^*_{(S)} =$ 15179 & 81.83 $\pm$ 0.15 & \textbf{87.21} $\pm$ \textbf{0.15} & 0.28 & 2964.51 & 2964.79 (49min) \\
 & - with $n_{(S)} =$ 20000 & \textbf{82.31} $\pm$ \textbf{0.18} & \textbf{87.36} $\pm$ \textbf{0.15} & 0.28 & 3596.03 & 3596.31 (1h) \\
& \underline{Best baseline} & &  &  &  & \\
& Core-Graph VAE, $k=40$  & 80.53 $\pm$ 0.23 & 82.45 $\pm$ 0.20 & - & 12433.51 & 12433.51 (3h27) \\
\midrule

 & Standard Graph AE  & \multicolumn{2}{c|}{\textit{(intractable)}}  & \multicolumn{3}{c}{\textit{(intractable)}}\\
 & \underline{FastGAE with (degree, $\alpha = 2$)} &  &  &  & \\
\textbf{Patent} & - with $n_{(S)} =$ 5000 & 90.66 $\pm$ 0.25 & 90.76 $\pm$ 0.22 & \textbf{0.30} & \textbf{605.75} & \textbf{606.05 (10min)}\\
 & - with $n_{(S)} = n^*_{(S)} =$ 16425 &  \textbf{93.02} $\pm$ \textbf{0.23} & \textbf{93.39} $\pm$ \textbf{0.23} & 0.30 & 3693.32 & 3693.62 (1h02)\\
 & - with $n_{(S)} =$ 20000 & \textbf{92.91} $\pm$ \textbf{0.22} & \textbf{93.35} $\pm$ \textbf{0.21} & 0.30 & 4401.67 & 4401.67 (1h13) \\
& \underline{Best baseline} &  &  &  & \\
& node2vec  & \textbf{92.96} $\pm$ \textbf{0.23} & \textbf{93.43} $\pm$ \textbf{0.17} & - & 25851.39 & 25851.39 (7h11) \\
\bottomrule
\end{tabular}
\end{tiny}
\caption{Summarized results for link prediction on all large graphs. For each graph, for brevity, we only report a few representative degree-based FastGAE \textbf{or} Variational FastGAE models, and the best baseline (among Core-Graph AE/VAE, Negative Sampling Graph AE/VAE and node2vec). Scores are averaged over 10 runs with different and random train/validation/test sets. Standard Graph AE and VAE are intractable. Scores are averaged over 10 runs with different and random train/validation/test sets. For degree sampling, values of the hyperparameter $\alpha$ (as defined in equation 9) were tuned, as described in Figure C.6. All models learn embedding vectors of dimension $d=16$. Bold numbers correspond to the best performance (several numbers are bold when scores are comparable, in a $\pm 1$ standard deviation range) and best running time.}
\end{table*}

\subsubsection{On the embedding dimension $d$} Our tables present results for a fixed embedding dimension of $d=16$, for \textit{all} models (all variants of AE/VAE and other baselines), even for large graphs. Nonetheless, we reached similar conclusions for $d = 32$, $64$ and $128$: although performances sometimes slightly improved by increasing $d$, the \textit{ranking} of the different models remained unchanged. We also considered optimizing $d$ individually for each model (to cover potential cases where the impact of $d$ on the performance of each model would have been different) but, again, it did not modify the ranking of models in terms of AUC, AP and AMI scores.

\subsubsection{On the number of training iterations} 
As detailed in Section 4.1.3, all graph AE and VAE models, with or without our FastGAE framework, were trained for 200 iterations (resp. 300) for graphs with $n <$ 100000 (resp. $n \geq$ 100000). We thoroughly checked the convergence of all models, by assessing the stabilization of performances in terms of AUC scores on validation sets. Using a fixed number of iterations is common in recent research on graph AE and VAE \cite{kipf2016-2,semiimplicit2019,salha2019-1,berg2018matrixcomp}. We nonetheless think that \textit{early-stopping} \cite{girosi1995regularization} would also be a relevant alternative strategy, that could lead to additional speed-ups, and might deserve further investigations in future works. Besides, we observed that, for very small values of $n_{(S)}$, increasing the number of training iterations did not significantly improved our results: to improve scores on such settings, increasing the sampling size $n_{(S)}$ was overall more effective than increasing the number of training iterations.

\section{Conclusion and Discussion}

In this paper, we introduced and released a general framework to scale graph AE and VAE models. We demonstrated its effectiveness on large graphs with up to millions of nodes and edges, both in terms of speed, of scalability and of performance. We outperformed the few existing approaches to scale graph AE and VAE, usually by a wide margin. FastGAE is also conceptually simpler than these alternative approaches \cite{salha2019-1}, and we believe that simple solutions often have the most impact.

Futhermore, FastGAE is a flexible framework that easily extends to AE/VAE models with alternative GNN \textit{encoders}. In our experiments, the GCN encoders of standard graph AE/VAE models and of FastGAE-based models could easily be replaced by any alternative architecture learning the embedding matrix $Z$ in another way, e.g. by a FastGCN \cite{chen2018fastgcn}, a Cluster-GCN \cite{chiang2019cluster}, a GCN with simple graph convolutions (SGC) \cite{wu2019simplifying} or a GraphSAGE \cite{hamilton2017inductive} model. Besides, FastGAE easily extends to graph AE or VAE with alternative \textit{decoders}. For instance, one could replace the symmetric inner-product decoder from our experiments by the asymmetric decoder recently proposed by \citet{salha2019-2}, which would extend FastGAE to \textit{directed} graphs. 

Last, but not least, we also identify possible future research directions for improvements. Apart from the aforementioned limit (section 4.2.3) of current graph AE and VAE models on the node clustering task (that, however, concerns all graph AE/VAE from our experiments and is not specific to FastGAE), we underline that the proposed FastGAE  method could underperform on very sparse graphs. Indeed, in such scenario, the subgraphs to reconstruct might include a large proportion of isolated nodes, which would negatively impact learning. Moreover, in the case of large graphs with a lot of sparsely connected components, we recommend applying FastGAE separately on each component. Also, in this paper we always assumed that the graph was fixed, which might sometimes be a limit, that could initiate future interesting studies on extensions of FastGAE for scalable \textit{dynamic} graph embeddings, potentially with a dynamic selection of $n_{(S)}$.

\clearpage

\appendix

\section*{Appendices}

\setcounter{theorem}{0}

This supplementary material provides all proofs from our theoretical analyses as well as an additional figure. It is organized as follows:
\begin{itemize}
    \item In Appendix A, we report the proofs of Propositions 1, 2 and 3 from our Section 3.3.2 on approximated losses.
    \item In Appendix B, we report the proofs of Propositions 4 and 5 from our Section 3.3.3 on the computation of the threshold subgraph size $n^*_{(S)}$.
    \item The figure of Appendix C presents optimal values of the hyperparameter $\alpha$.
\end{itemize}

\section{On Approximate Losses}

\begin{theorem}
Let $\mathcal{G}_{(S)} = (\mathcal{V}_{(S)},\mathcal{E}_{(S)})$ be a subgraph of $\mathcal{G}$ obtained from sampling $n_{(S)}$ nodes \textbf{with} replacement using the node sampling strategy of FastGAE. Let $i$ and $j$ denote two distinct nodes from the original graph $\mathcal{G}$: $(i,j) \in \mathcal{V}^2$. Then:
\begin{equation*}\mathbb{P}\Big(i \in \mathcal{V}_{(S)}\Big) = 1 - (1 - p_i)^{n_{(S)}}.\end{equation*}
Also:
\begin{align*}
\mathbb{P}\Big((i,j) \in \mathcal{V}_{(S)}^2\Big) &= 1 - \Big[(1 - p_i)^{n_{(S)}} + (1 - p_j)^{n_{(S)}} \nonumber \\ 
&- (1 - p_i - p_j)^{n_{(S)}}\Big].
\end{align*}
\end{theorem}

\begin{proof}
In this setting, sampling probabilities are independent of previous sampling steps, and remain fixed to $p_i$.  Therefore, for node $i \in \mathcal{V}$, we have:
$$\mathbb{P}\Big(i \notin \mathcal{V}_{(S)}\Big) = (1 - p_i)^{n_{(S)}}.$$
Indeed, for $i$ \textit{not} to belong to $\mathcal{V}_{(S)}$, it must not be selected at any of the $n_{(S)}$ draws, which happens with probability $1 - p_i$ for each draw. Therefore:
$$\mathbb{P}\Big(i \in \mathcal{V}_{(S)}\Big) = 1 - (1 - p_i)^{n_{(S)}}.$$

Moreover, let $i$ and $j$ denote two distinct nodes from the original graph $\mathcal{G}$: $(i,j) \in \mathcal{V}^2$. We have:
\begin{align*}
&\mathbb{P}\Big((i,j) \notin \mathcal{V}_{(S)}^2\Big) =
\mathbb{P}\Big(i \notin \mathcal{V}_{(S)} \text{ or } j \notin \mathcal{V}_{(S)}\Big) \\
&= \mathbb{P}\Big(i \notin \mathcal{V}_{(S)}\Big) + \mathbb{P}\Big(j \notin \mathcal{V}_{(S)}\Big) - \mathbb{P}\Big(i \notin \mathcal{V}_{(S)}, j \notin \mathcal{V}_{(S)}\Big)
\end{align*}
with, using the previous result, $\mathbb{P}(i \notin \mathcal{V}_{(S)}) = (1 - p_i)^{n_{(S)}}$ and $\mathbb{P}(j \notin \mathcal{V}_{(S)}) = (1 - p_j)^{n_{(S)}}$. Using a similar argument, we also obtain:
$$\mathbb{P}\Big(i \notin \mathcal{V}_{(S)}, j \notin \mathcal{V}_{(S)}\Big) = \Big(1 - (p_i + p_j)\Big)^{n_{(S)}}.$$
Therefore:
\begin{align*}
\mathbb{P}\Big((i,j) \notin \mathcal{V}_{(S)}^2\Big) &= \Big[(1 - p_i)^{n_{(S)}} + (1 - p_j)^{n_{(S)}} \\
&- (1 - p_i - p_j)^{n_{(S)}}\Big].
\end{align*}
And:
\begin{align*}
\begin{split}
\mathbb{P}\Big((i,j) \in \mathcal{V}_{(S)}^2\Big) &= 1 - \mathbb{P}\Big((i,j) \notin \mathcal{V}_{(S)}^2\Big) \\ &= 1 - \Big[(1 - p_i)^{n_{(S)}} \\
&+ (1 - p_j)^{n_{(S)}} \\
&- (1 - p_i - p_j)^{n_{(S)}}\Big].
\end{split}
\end{align*}
Last, for self-loops:
$$\mathbb{P}\Big((i,i) \in \mathcal{V}_{(S)}^2\Big) = \mathbb{P}\Big(i \in \mathcal{V}_{(S)}\Big) = 1 - (1 - p_i)^{n_{(S)}}.$$
\end{proof}

\begin{theorem}
Let $\mathcal{G}_{(S)} = (\mathcal{V}_{(S)},\mathcal{E}_{(S)})$ be a subgraph of $\mathcal{G}$ obtained from sampling $n_{(S)}$ nodes \textbf{without} replacement using the node sampling strategy of FastGAE. Let $i$ and $j$ denote two distinct nodes from $\mathcal{G}$: $(i,j) \in \mathcal{V}^2$. Then:
\begin{equation*}\mathbb{P}\Big(i \in \mathcal{V}_{(S)}\Big) = \sum_{\mathcal{U} \in \mathcal{\textbf{U}}(i)} p_{u_1} \prod_{k=2}^{n_{(S)}} \frac{p_{u_k}}{1 - \sum_{k'=1}^{k-1} p_{u_{k'}}},\end{equation*}
where $\mathcal{\textbf{U}}(i) = \{\mathcal{U} \subset \mathcal{V}, |\mathcal{U}| = n_{(S)} \text{ and } i \in \mathcal{U}\}$ is the set of all \textbf{ordered} subsets of $n_{(S)}$ distinct nodes including node $i$. For a given set $\mathcal{U} \in \mathcal{\textbf{U}}(i)$, we denote by $(u_1, u_2,...,u_{n_{(S)}})$ its ordered elements. Also,
\begin{equation*}\mathbb{P}\Big((i,j) \in \mathcal{V}_{(S)}^2\Big) = 
\sum_{\mathcal{U} \in \mathcal{\textbf{U}}(i) \cap \textbf{U}(j)} p_{u_1} \prod_{k=2}^{n_{(S)}} \frac{p_{u_k}}{1 - \sum_{k'=1}^{k-1} p_{u_{k'}}}.\end{equation*}
\end{theorem}

\begin{proof}
We are looking for the probability that a node $i  \in \mathcal{V}$ from the graph belongs to a drawn subset $\mathcal{V}_{(S)}$, that contains $n_{(S)}$ distinct nodes. For $\mathcal{V}_{(S)}$ to include $i$, $\mathcal{V}_{(S)}$ should match any of the possible ordered subsets of $n_{(S)}$ nodes that include node $i$. In this setting where we sample without replacement, the probability to draw node $i$ \textit{depends on nodes previously drawn}. All possible orders of sampling the nodes should be considered. Let:
$$\mathcal{\textbf{U}}(i) = \Big\{\mathcal{U} \subset \mathcal{V}, |\mathcal{U}| = n_{(S)} \text{ and } i \in \mathcal{U}\Big\}$$
denote the set of all \textbf{ordered} subsets of $n_{(S)}$ distinct nodes that include node $i$. With such notations:
$$\mathbb{P}\Big(i \in \mathcal{V}_{(S)}\Big) = \mathbb{P}\Big(\mathcal{V}_{(S)} \in \mathcal{\textbf{U}}(i) \Big) = \sum_{\mathcal{U} \in \mathcal{\textbf{U}}(i)} \mathbb{P}\Big(\mathcal{V}_{(S)} = \mathcal{U}\Big).$$

The summation comes from the fact that events are \textit{disjoint} ($\mathcal{V}_{(S)}$ can not match two of these ordered subsets simultaneously.

Now, for a given set $\mathcal{U} \in \mathcal{\textbf{U}}(i)$, let us denote by $(u_1, u_2,...,u_{n_{(S)}})$ its \textbf{ordered} elements. Also, let $(\mathcal{V}_{(S)1}, \mathcal{V}_{(S) 2},...,\mathcal{V}_{(S) n_{(S)}})$ be the $n_{(S)}$ ordered nodes of set $\mathcal{G}_{(S)}$ (i.e. $\mathcal{V}_{(S) 1}$ is the first drawn node, $\mathcal{V}_{(S) 2}$ is the second one, etc). We have:

\begin{align*}
&\mathbb{P}\Big(\mathcal{V}_{(S)} = \mathcal{U} \Big) = \mathbb{P}\Big(\mathcal{V}_{(S) 1} = u_1, \mathcal{V}_{(S)2} = u_2,..., \mathcal{V}_{(S) n_{(S)}} = u_{n_{(S)}} \Big)\\
&= \mathbb{P}(\mathcal{V}_{(S) 1} = u_1) \prod_{k=2}^{n_{(S)}}  \mathbb{P}(\mathcal{V}_{(S) k} = u_k | \mathcal{V}_{(S) k-1} = u_{k-1},..., \mathcal{V}_{(S) 1} = u_1) \\
&= p_{u_1} \prod_{k=2}^{n_{(S)}} \frac{p_{u_k}}{1 - \sum_{k'=1}^{k-1} p_{u_{k'}}}.
\end{align*}
Therefore, by summing elements to come back to $\mathbb{P}(i \in \mathcal{V}_{(S)})$:
$$\mathbb{P}\Big(i \in \mathcal{V}_{(S)}\Big) = \sum_{\mathcal{U} \in \mathcal{\textbf{U}}(i)} p_{u_1} \prod_{k=2}^{n_{(S)}} \frac{p_{u_k}}{1 - \sum_{k'=1}^{k-1} p_{u_{k'}}}.$$

Moreover, let $i$ and $j$ denote two distinct nodes from the original graph $\mathcal{G}$: $(i,j) \in \mathcal{V}^2$. Using similar notations and reasoning, we get:
\begin{align*}
\mathbb{P}\Big((i,j) \in \mathcal{V}_{(S)}^2\Big) &= \mathbb{P}\Big(i \in \mathcal{V}_{(S)},j \in \mathcal{V}_{(S)} \Big) \\
&= \sum_{\mathcal{U} \in \mathcal{\textbf{U}}(i) \cap \textbf{U}(j)} \mathbb{P}\Big(\mathcal{V}_{(S)} = \mathcal{U} \Big).
\end{align*}
Therefore:
$$\mathbb{P}\Big((i,j) \in \mathcal{V}_{(S)}^2\Big) = 
\sum_{\mathcal{U} \in \mathcal{\textbf{U}}(i) \cap \textbf{U}(j)} p_{u_1} \prod_{k=2}^{n_{(S)}} \frac{p_{u_k}}{1 - \sum_{k'=1}^{k-1} p_{u_{k'}}}.$$
And, for self-loops, $\mathbb{P}((i,i) \in \mathcal{V}_{(S)}^2) = \mathbb{P}(i \in \mathcal{V}_{(S)})$.
\end{proof}

\begin{theorem}
Using the expressions of Proposition 1 (with replacement) or Proposition 2 (without replacement):
\begin{equation*}\mathbb{E}\Big[\mathcal{L}^{\text{\tiny FastGAE}}\Big] = \frac{1}{n_{(S)}^2}\sum_{(i,j) \in \mathcal{V}^2} \mathbb{P}\Big((i,j) \in \mathcal{V}_{(S)}^2\Big) \mathcal{L}_{ij}(A_{ij},\hat{A}_{ij}).\end{equation*}
\end{theorem}

\begin{proof}
We have:
\begin{align*}
\mathbb{E}\Big[\mathcal{L}^{\text{\tiny FastGAE}}\Big] &= \mathbb{E}\Big[\frac{1}{n_{(S)}^2}\sum_{(i,j) \in \mathcal{V}^{2}} \mathds{1}_{((i,j) \in \mathcal{V}_{(S)}^2)}\mathcal{L}_{ij}(A_{ij},\hat{A}_{ij})\Big] \\
&= \frac{1}{n_{(S)}^2}\sum_{(i,j) \in \mathcal{V}^{2}} \mathbb{E}\Big[\mathds{1}_{((i,j) \in \mathcal{V}_{(S)}^2)}\Big]\mathcal{L}_{ij}(A_{ij},\hat{A}_{ij}) \\
&= \frac{1}{n_{(S)}^2}\sum_{(i,j) \in \mathcal{V}^2} \mathbb{P}\Big((i,j) \in \mathcal{V}_{(S)}^2\Big) \mathcal{L}_{ij}(A_{ij},\hat{A}_{ij}).
\end{align*}
By replacing $\mathbb{P}((i,j) \in \mathcal{V}_{(S)}^2)$ by the expressions of Proposition 1 (with replacement) or Proposition 2 (without replacement), we obtain an explicit formulation for $\mathbb{E}\Big[\mathcal{L}^{\text{\tiny FastGAE}}\Big]$.
\end{proof}

\section{On the Selection of $n_{(S)}$}

\begin{theorem}
Let us consider a training iteration of the FastGAE framework, a sampled subgraph $\mathcal{G}_{(S)} = (\mathcal{V}_{(S)},\mathcal{E}_{(S)})$, with $|\mathcal{V}_{(S)}| = n_{(S)} < n$ nodes sampled without replacement, and the corresponding node-level approximate reconstruction computed for a given node $i$:
$$\mathcal{L}^{\text{\tiny FastGAE}}(i) =  \frac{1}{n_{(S)}} \sum_{j \in \mathcal{V}} \mathds{1}_{(j \in \mathcal{V}_{(S)})} \mathcal{L}_{ij}(A_{ij},\hat{A}_{ij}),$$
with the random variable $\mathds{1}_{(j \in \mathcal{V}_{(S)})} = 1$ if node $j \in \mathcal{V}_{(S)}$ and $0$ otherwise, with $A_{ij} \in \{0,1\}$ for all $(i,j) \in \mathcal{V}^2$ and with:
$$\mathcal{L}_{ij}(A_{ij},\hat{A}_{ij}) = - [A_{ij}\log(\hat{A}_{ij}) + (1-A_{ij})\log(1 - \hat{A}_{ij})].$$
Then, under Assumption 1 from Section 3.3.3, for any $\gamma \geq 0$, we have:
\begin{small}
\begin{equation*}
\mathbb{P}(|\mathcal{L}^{\text{\tiny FastGAE}}(i) - \mathbb{E}[\mathcal{L}^{\text{\tiny FastGAE}}(i)]| \geq \gamma) \leq 2~\text{exp}\Big( - 2(\frac{\gamma}{\log(\varepsilon)})^2 \frac{n_{(S)}^2}{n}\Big).
\end{equation*}
\end{small}
\end{theorem}
We note that the right hand side term tends to 0 exponentially fast w.r.t. the deviation magnitude $\gamma$ and w.r.t. the subgraph size $n_{(S)}$.

\begin{proof}
As a preliminary, let us recall Hoeffding's inequality \cite{hoeffding1963probability}. Let $X_1, X_2..., X_n$ be real independent random variables verifying, for some $(a_k)_{1 \leq k \leq n}$ and $(b_k)_{1 \leq k \leq n}$ with $a_k < b_k$: $\forall k, \mathbb{P}(a_k \leq X_k \leq b_k) = 1.$ Let $S_n = \sum_{i=1}^n X_i.$ Then, for all $\gamma>0$, Hoeffding's inequality states that:
$$\mathbb{P}(|S_n - \mathbb{E}(S_n)| \geq t) \leq 2~\text{exp}\Big( - \frac{2\gamma^2}{\sum_{i=1}^n (b_i - a_i)^2} \Big).$$
\cite{hoeffding1963probability} also proves that the above inequality holds when the $X_i$ are samples without replacement from a finite population (and therefore not independent). In the setting of Proposition 4, that falls into this second case due to the node-level sampling scheme of FastGAE, we have:
$$\mathcal{L}^{\text{\tiny FastGAE}}(i) = \sum_{j \in \mathcal{V}} X_{ij},$$
where, under Assumption 1:
\begin{small}
\begin{align*}
X_{ij} &= \frac{1}{n_{(S)}} \mathds{1}_{(j \in \mathcal{V}_{(S)})}\mathcal{L}_{ij}(A_{ij},\hat{A}_{ij}) \\
 &=  \underbrace{\mathds{1}_{(j \in \mathcal{V}_{(S)})}}_{\in \{0,1\}} \underbrace{\frac{-1}{n_{(S)}} \underbrace{[A_{ij}\log(\hat{A}_{ij}) + (1-A_{ij})\log(1 - \hat{A}_{ij})]}_{\in [\log(\varepsilon),\log(1-\varepsilon)]}}_{\in [-\log(1-\varepsilon)/n_{(S)},-\log(\varepsilon)/n_{(S)}]} \\
 &\in \Big[0, \frac{-\log (\varepsilon)}{n_{(S)}}\Big].
\end{align*}
\end{small}
We note that $\frac{-\log (\varepsilon)}{n_{(S)}} >0$, as $0 < \varepsilon < 1$. Applying Hoeffding's inequality, at each sampling step and for all $\gamma >0$:
\begin{small}
\begin{align*}
    \mathbb{P}(|\mathcal{L}^{\text{\tiny FastGAE}}(i) - \mathbb{E}[\mathcal{L}^{\text{\tiny FastGAE}}(i)]| \geq \gamma) &\leq 2~\text{exp}\Big( - \frac{2\gamma^2}{\sum_{j \in \mathcal{V}} (\frac{-\log (\varepsilon)}{n_{(S)}})^2} \Big) \\
    &= 2~\text{exp}\Big( \frac{- 2\gamma^2}{n \frac{(-\log (\varepsilon))^2}{n_{(S)}^2}}\Big) \\
    &= 2~\text{exp}\Big( - 2(\frac{\gamma}{\log(\varepsilon)})^2 \frac{n_{(S)}^2}{n}\Big)
\end{align*}
\end{small}
We note that it exhibits the link between the deviation of the loss and the $\frac{n_{(S)}^2}{n}$ ratio.
\end{proof}

\begin{theorem}
For any confidence level $\alpha \in ]0,1[$ and node $i \in \mathcal{V}$, selecting a subgraph size $n_{(S)}$ such that
\begin{equation}
n_{(S)} \geq n^*_{(S)} = \sqrt{n} \underbrace{\sqrt{\frac{-\log (\frac{\alpha}{2})\log(\varepsilon)^2}{ 2\gamma^2}}}_{\text{denoted $C$ in eq. 10}}
\end{equation}
guarantees that
$$\mathbb{P}(|\mathcal{L}^{\text{\tiny FastGAE}}(i) - \mathbb{E}[\mathcal{L}^{\text{\tiny FastGAE}}(i)]| \geq \gamma) \leq \alpha.$$
\end{theorem}

\begin{proof}
This is a corollary of Proposition 4, from which we derive that, for any $\alpha \in ]0,1[$:
\begin{align*}
2~\text{exp}\Big( - 2(\frac{\gamma}{\log(\varepsilon)})^2 \frac{n_{(S)}^2}{n}\Big) &\leq \alpha \\
\Rightarrow \mathbb{P}(|\mathcal{L}^{\text{\tiny FastGAE}} - \mathbb{E}[\mathcal{L}^{\text{\tiny FastGAE}}]| \geq \gamma) &\leq \alpha.
\end{align*}
Then:
\begin{align*}
&2~\text{exp}\Big( - 2(\frac{\gamma}{\log(\varepsilon)})^2 \frac{n_{(S)}^2}{n}\Big) \leq \alpha \\
\Leftrightarrow~ &- 2(\frac{\gamma}{\log(\varepsilon)})^2 \frac{n_{(S)}^2}{n} \leq \log(\frac{\alpha}{2}) \\
\Leftrightarrow~ &n_{(S)} \geq \sqrt{n} \sqrt{\frac{-\log (\frac{\alpha}{2})\log(\varepsilon)^2}{ 2\gamma^2}}
\end{align*}
\end{proof}

\section{On the hyperparameter $\alpha$}
In this last appendix, we report the additional Figure C.6, presenting the optimal values of the hyperparameter $\alpha$, for all graphs, and for both core-based and degree-based sampling.

\begin{figure*}[!ht]
\centering
  \subfigure[Cora - Degree Sampling]{
  \scalebox{0.3}{\includegraphics{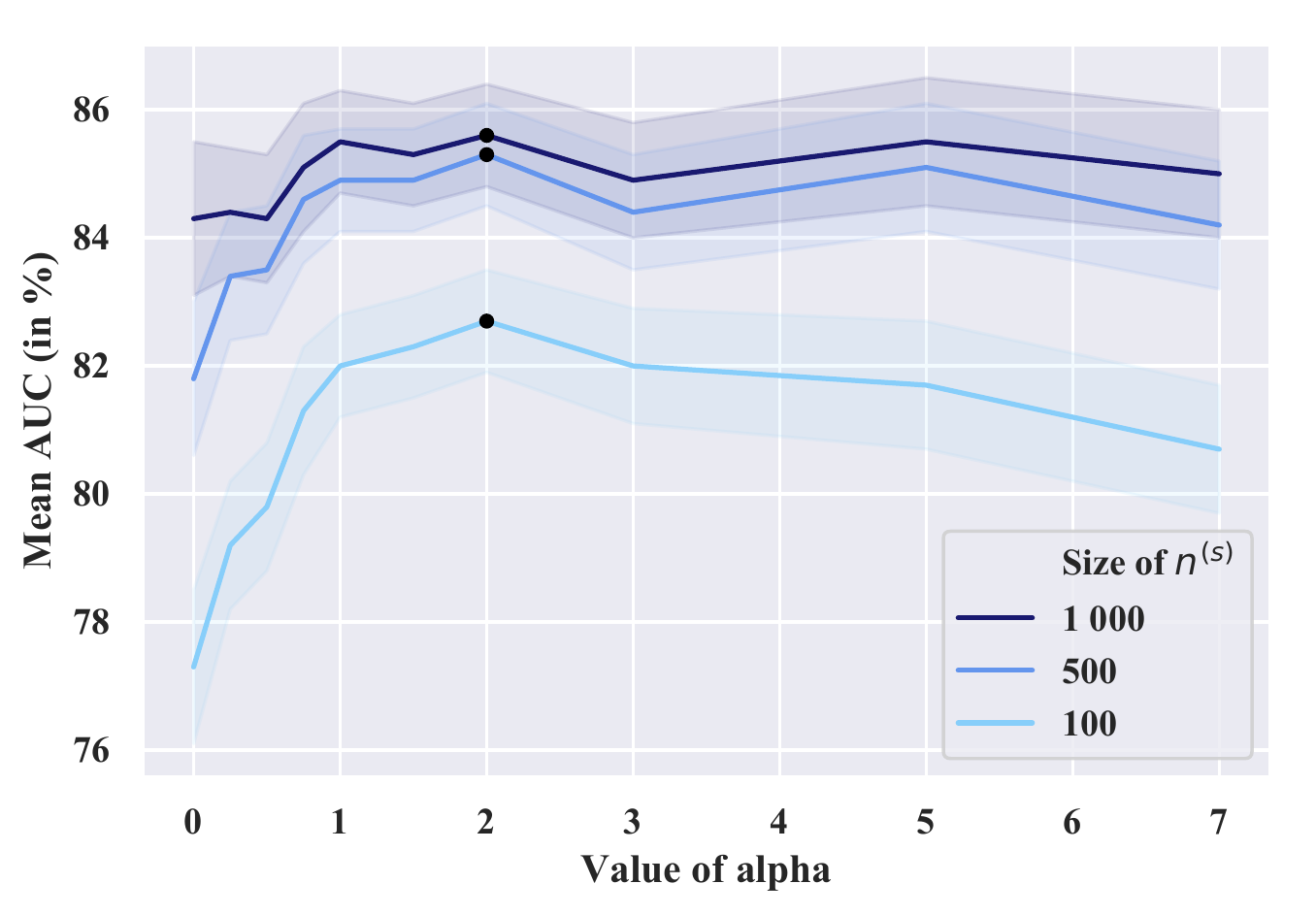}}}\subfigure[Citeseer - Degree Sampling]{
  \scalebox{0.3}{\includegraphics{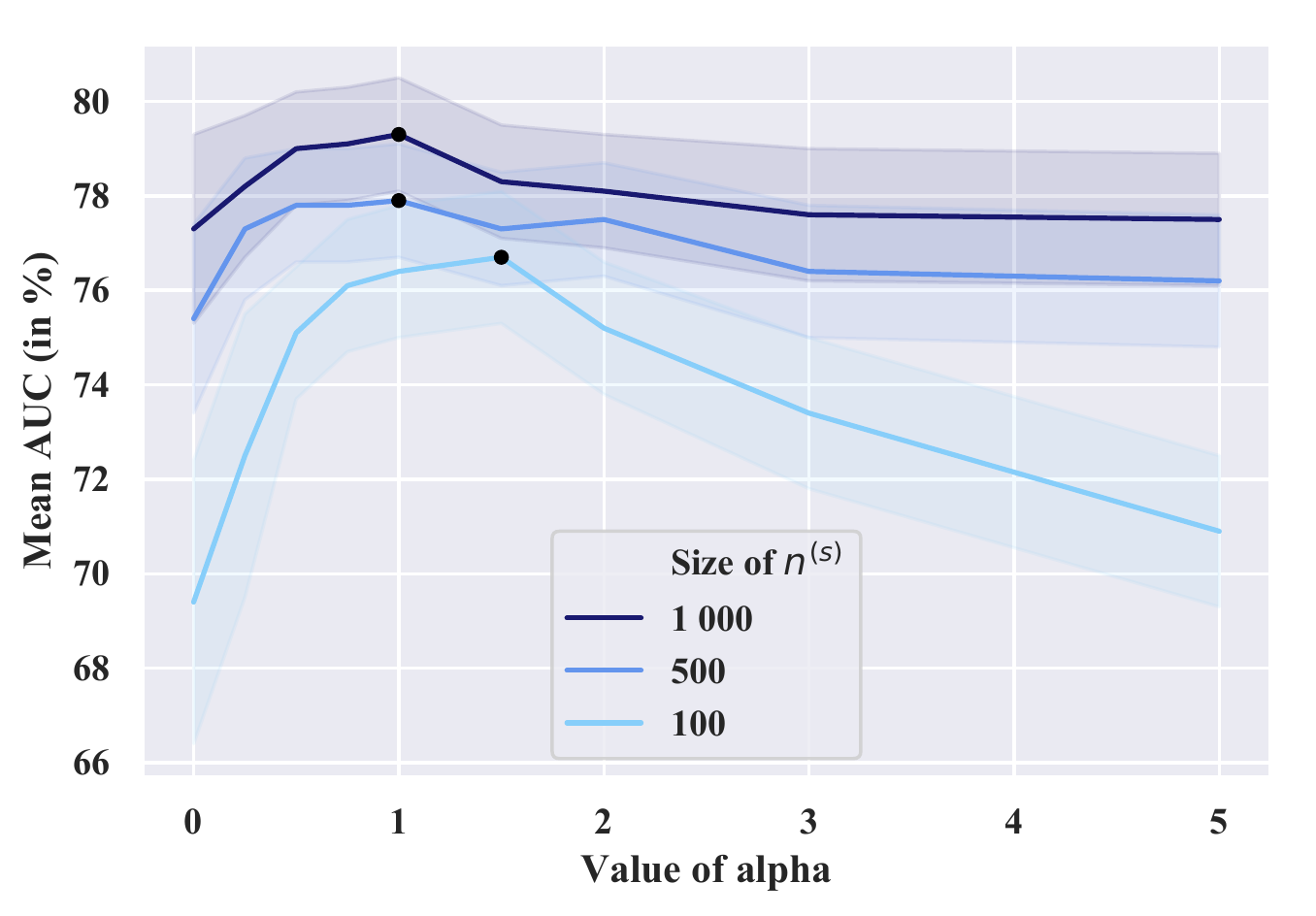}}}\subfigure[Pubmed - Degree Sampling]{
  \scalebox{0.3}{\includegraphics{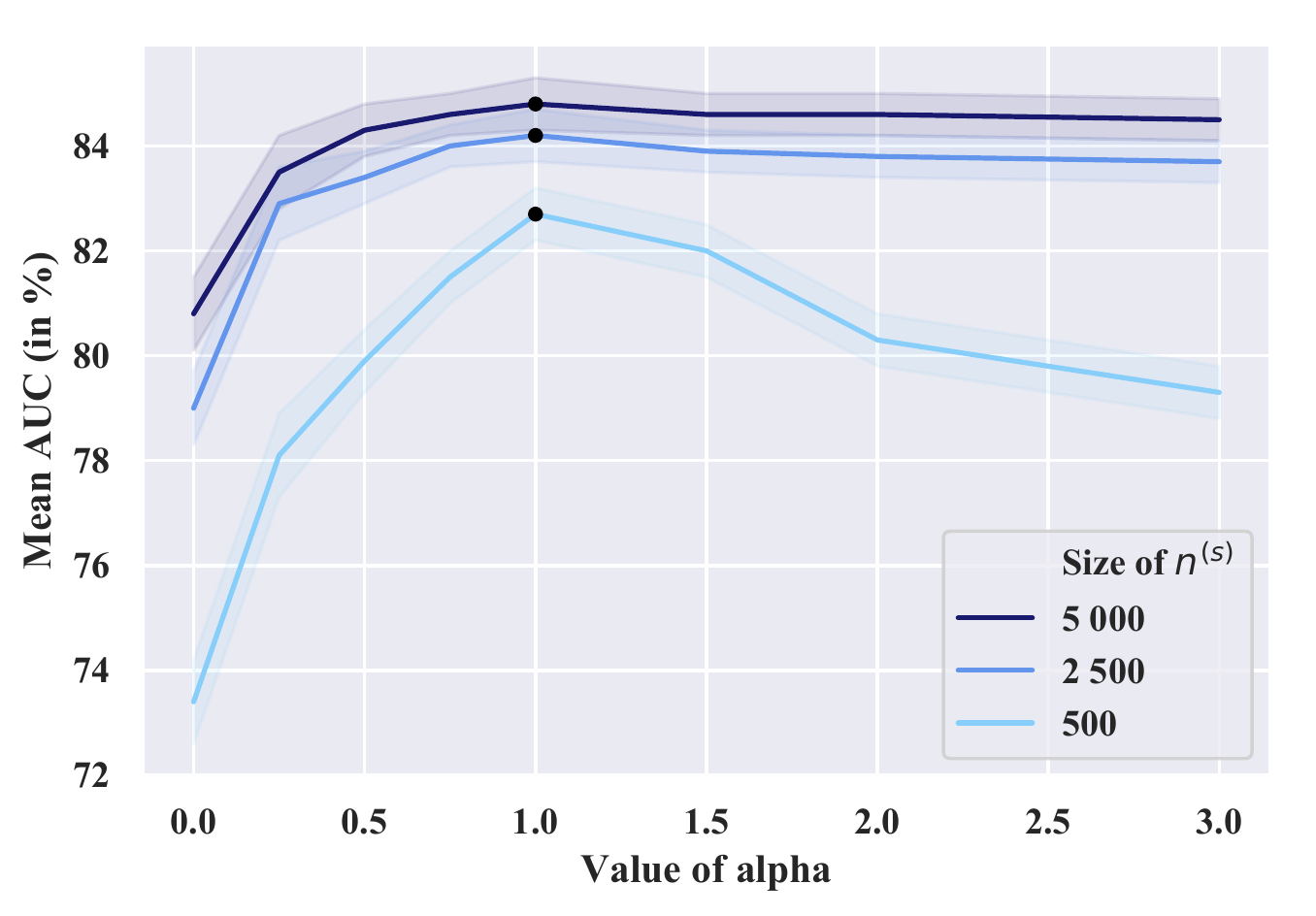}}}
   \subfigure[Cora - Core Sampling]{
  \scalebox{0.3}{\includegraphics{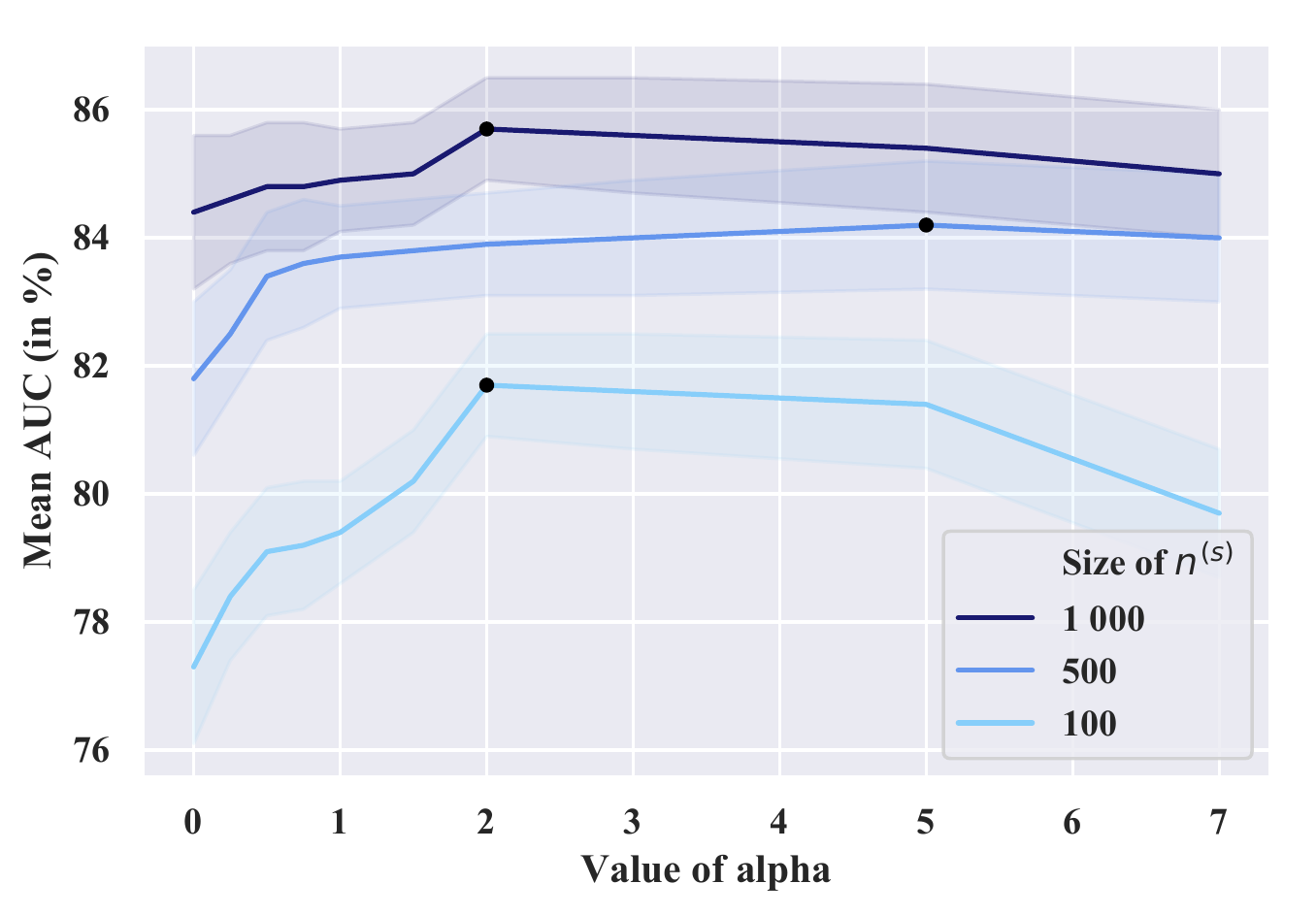}}}\subfigure[Citeseer - Core Sampling]{
  \scalebox{0.3}{\includegraphics{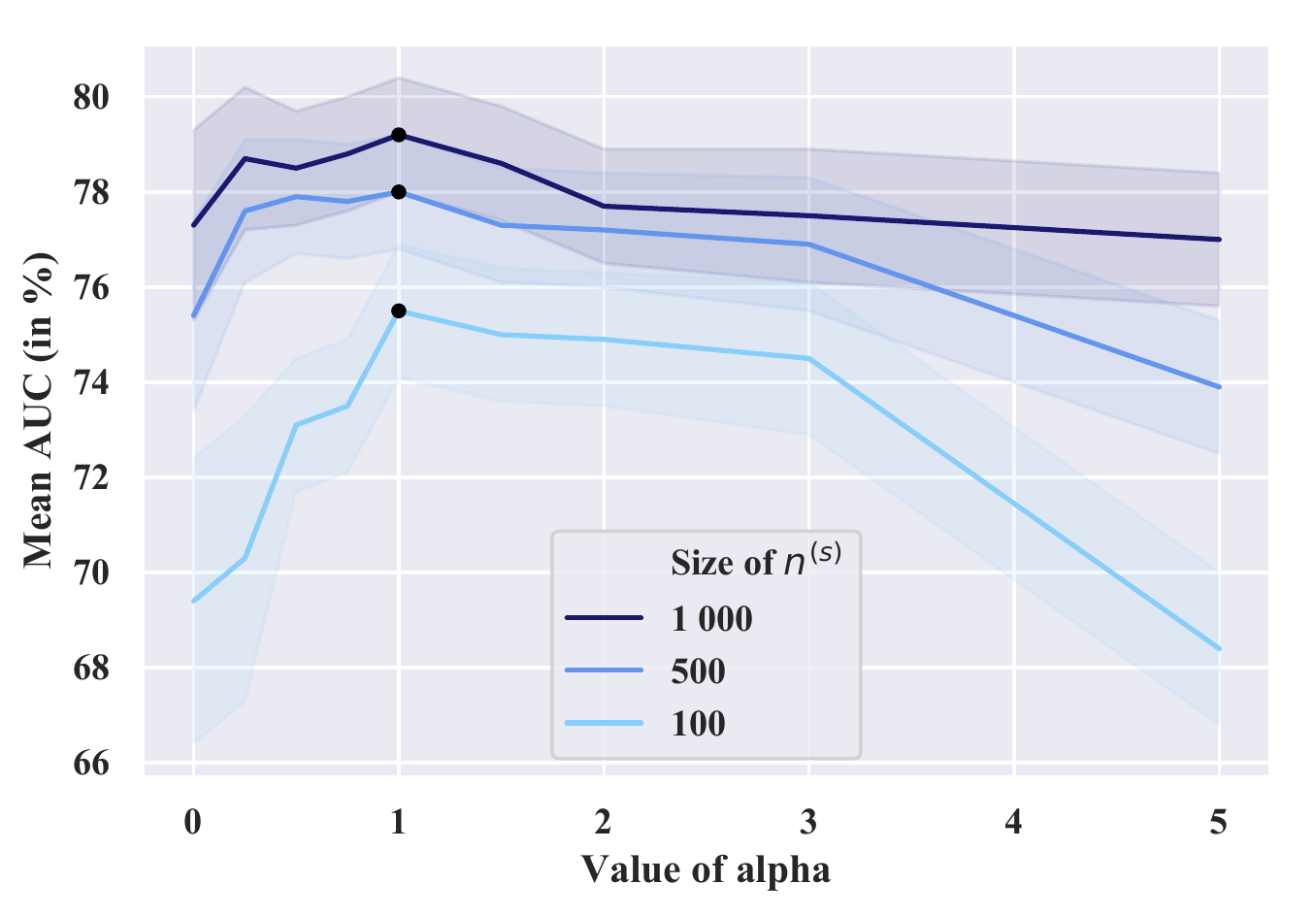}}}\subfigure[Pubmed - Core Sampling]{
  \scalebox{0.3}{\includegraphics{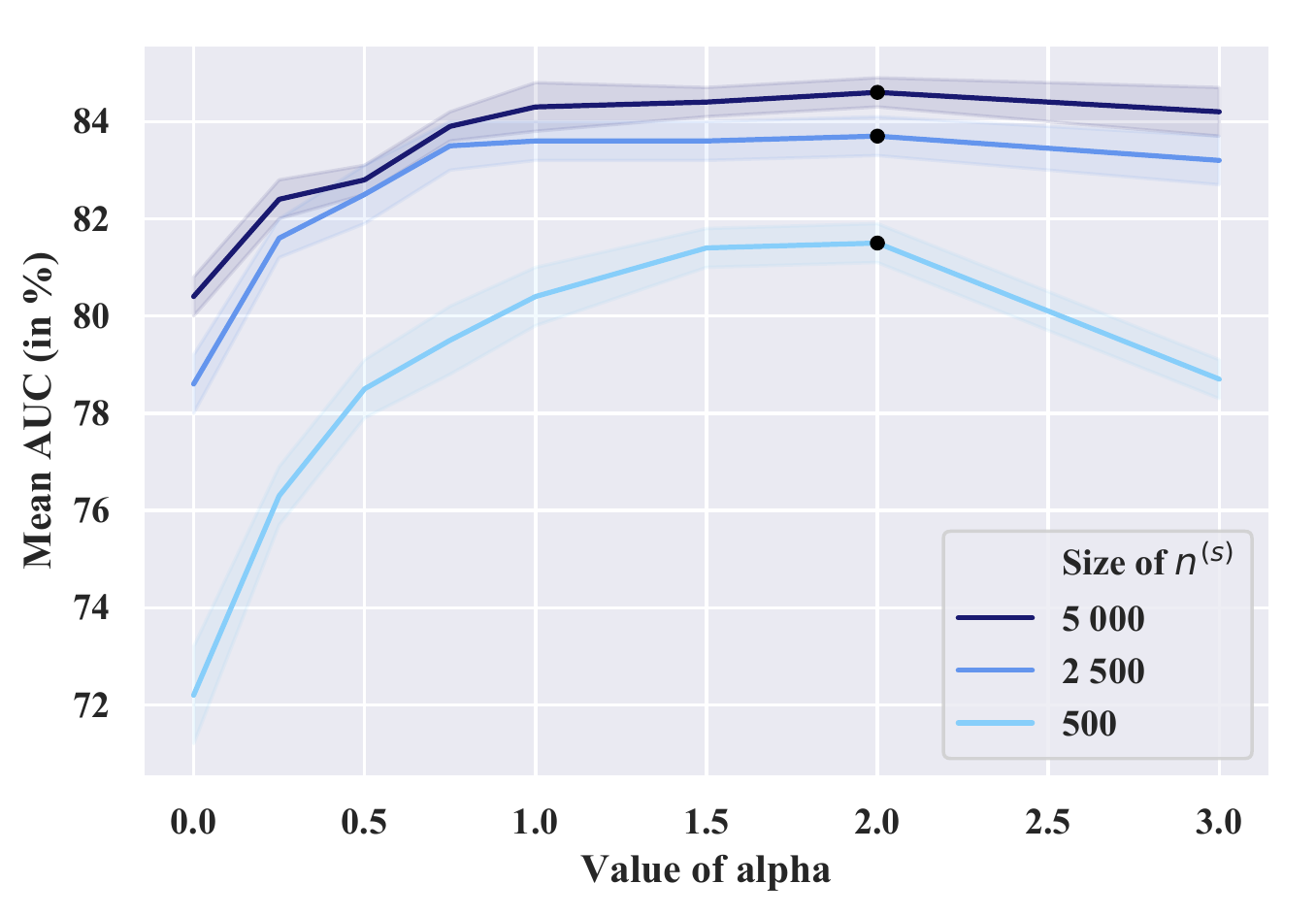}}}
     \subfigure[SBM - Degree Sampling]{
  \scalebox{0.3}{\includegraphics{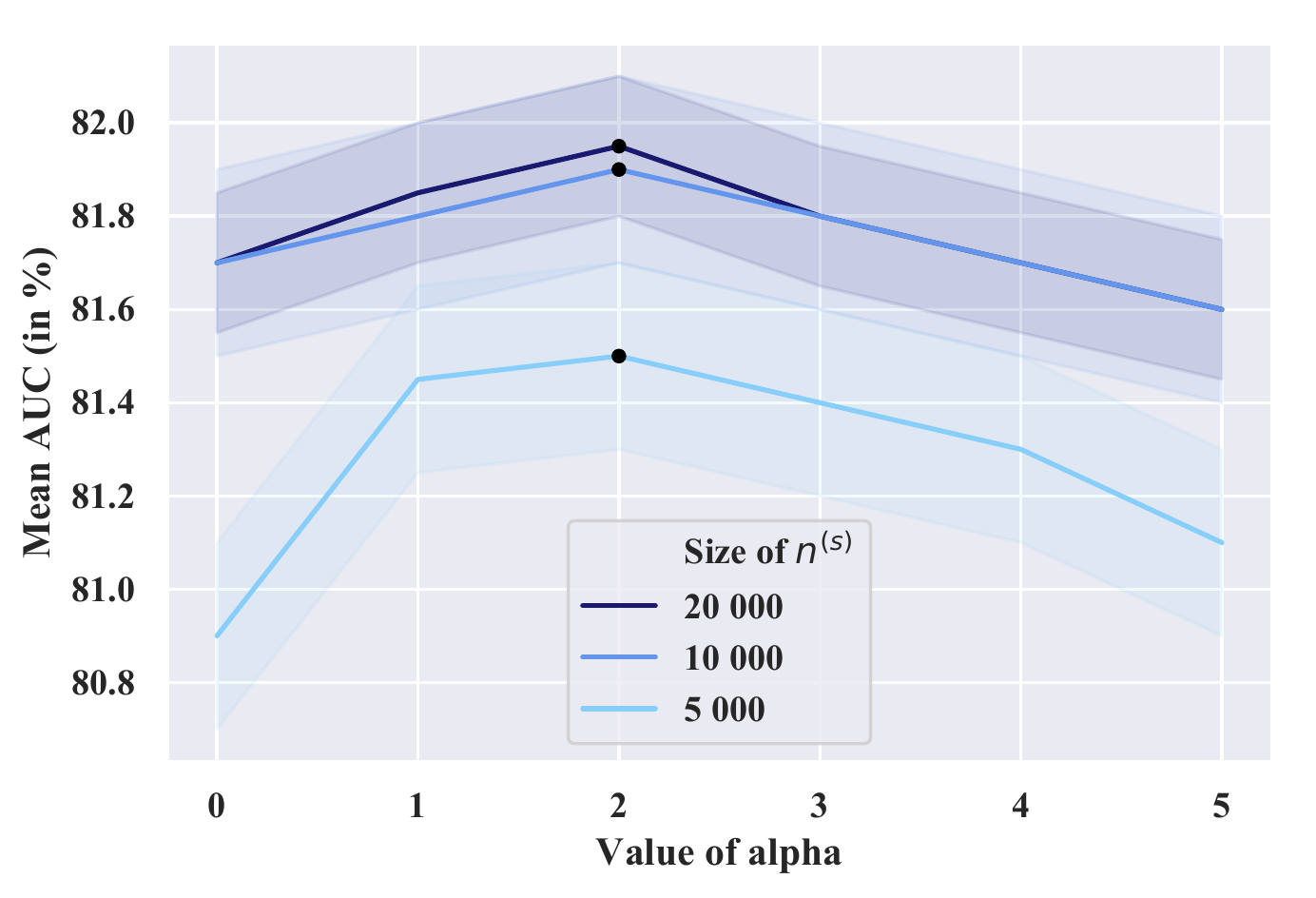}}}\subfigure[Google - Degree Sampling]{
  \scalebox{0.3}{\includegraphics{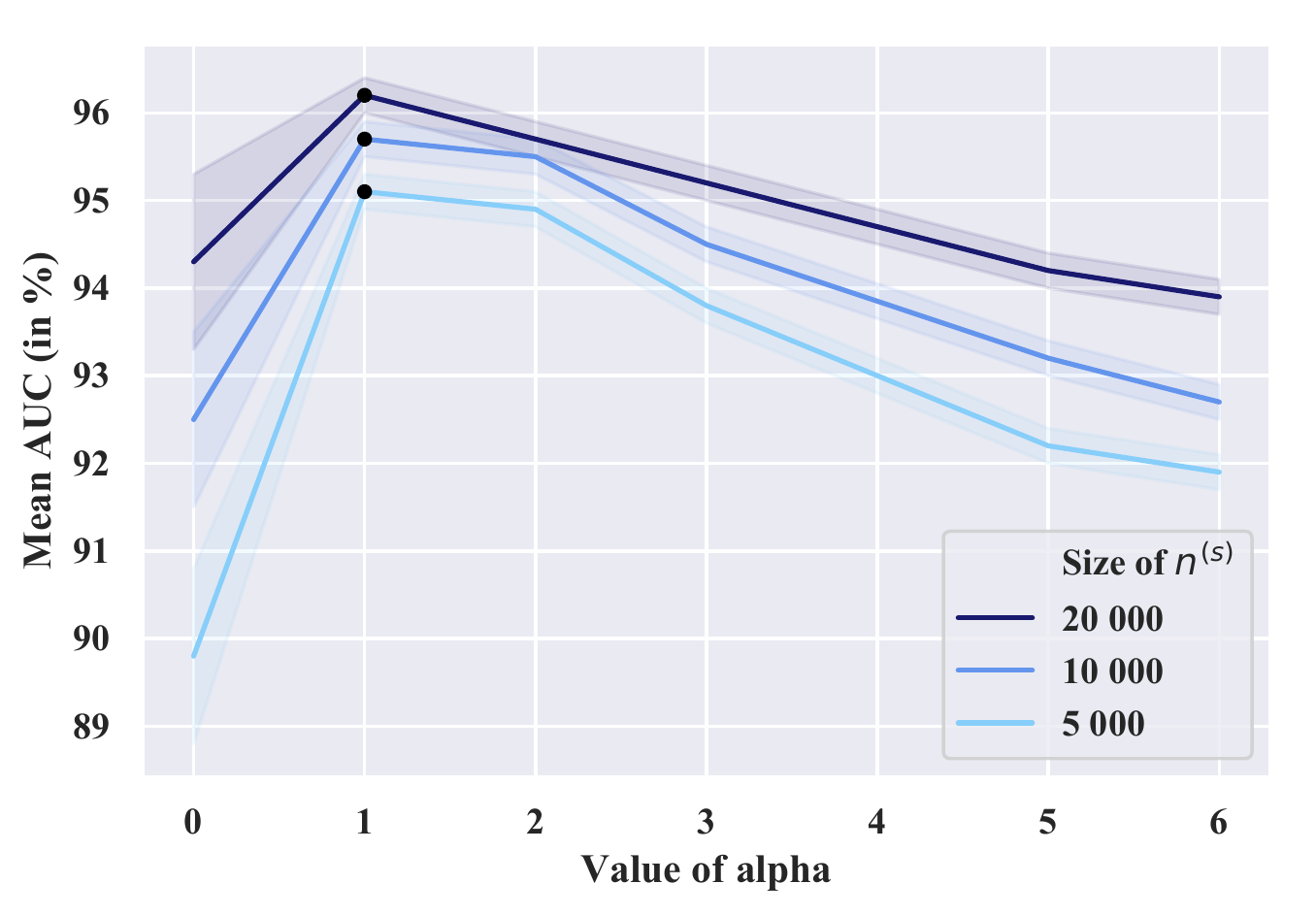}}}\subfigure[Youtube - Degree Sampling]{
  \scalebox{0.3}{\includegraphics{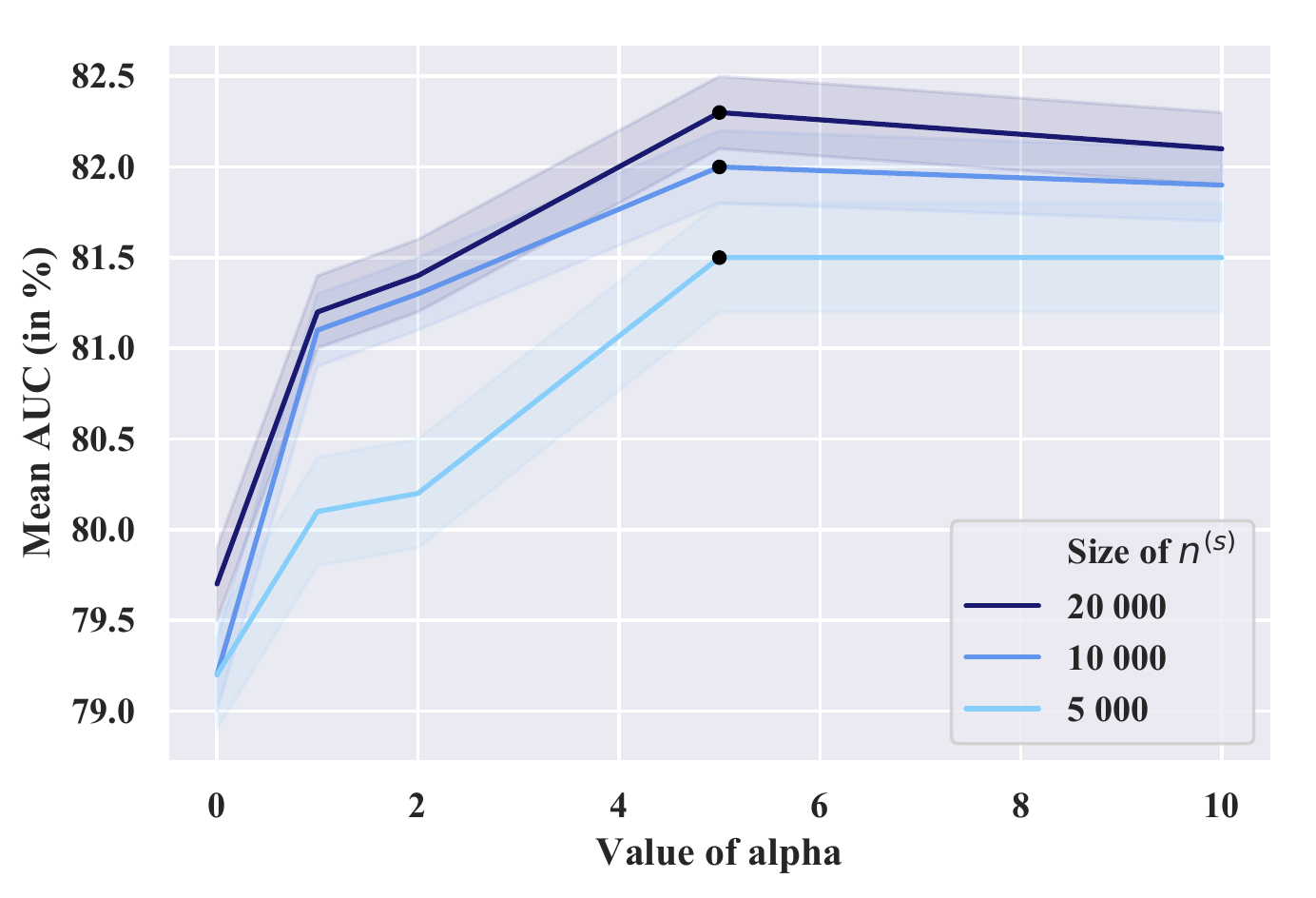}}}
     \subfigure[SBM - Core Sampling]{
  \scalebox{0.3}{\includegraphics{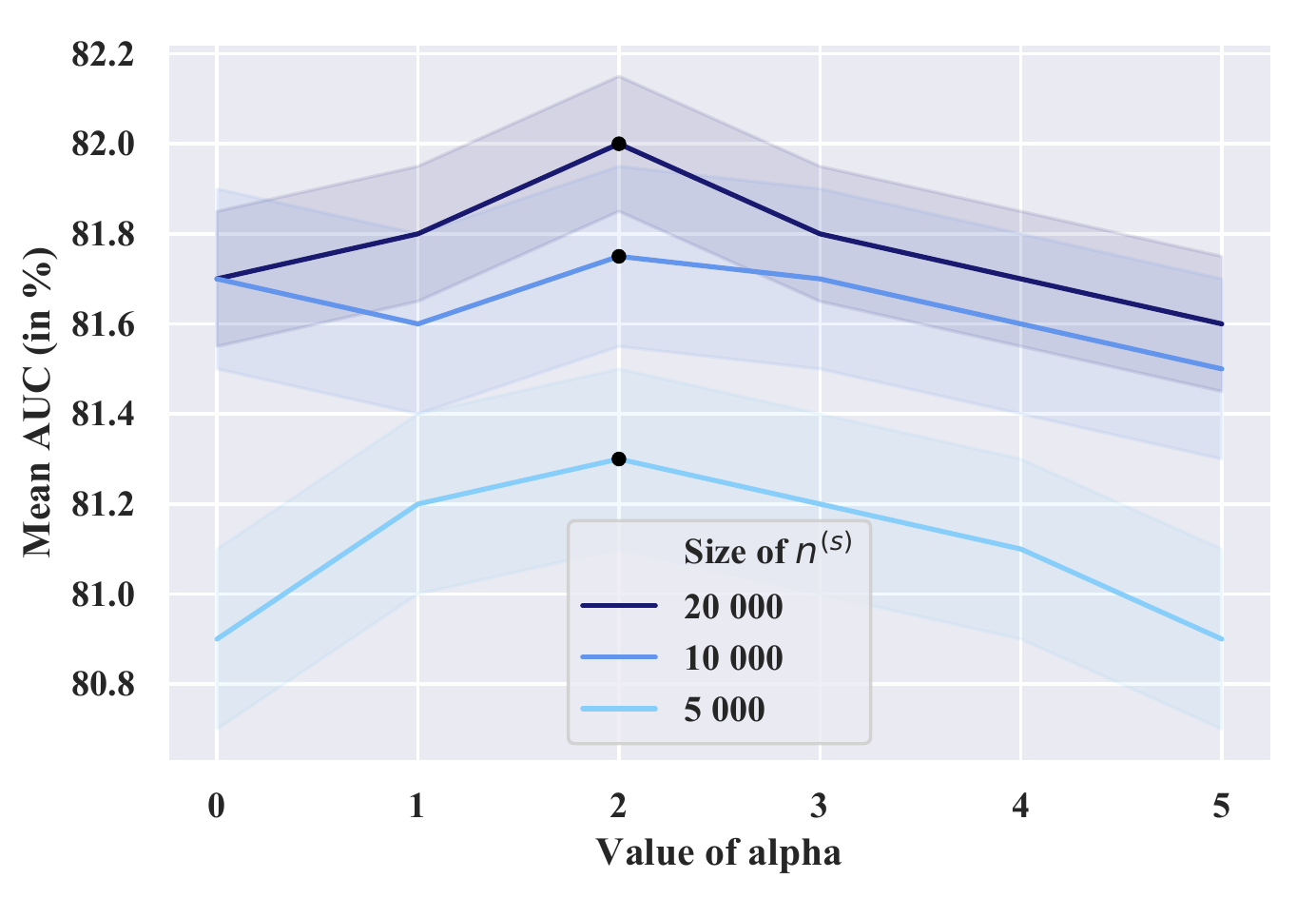}}}\subfigure[Google - Core Sampling]{
  \scalebox{0.3}{\includegraphics{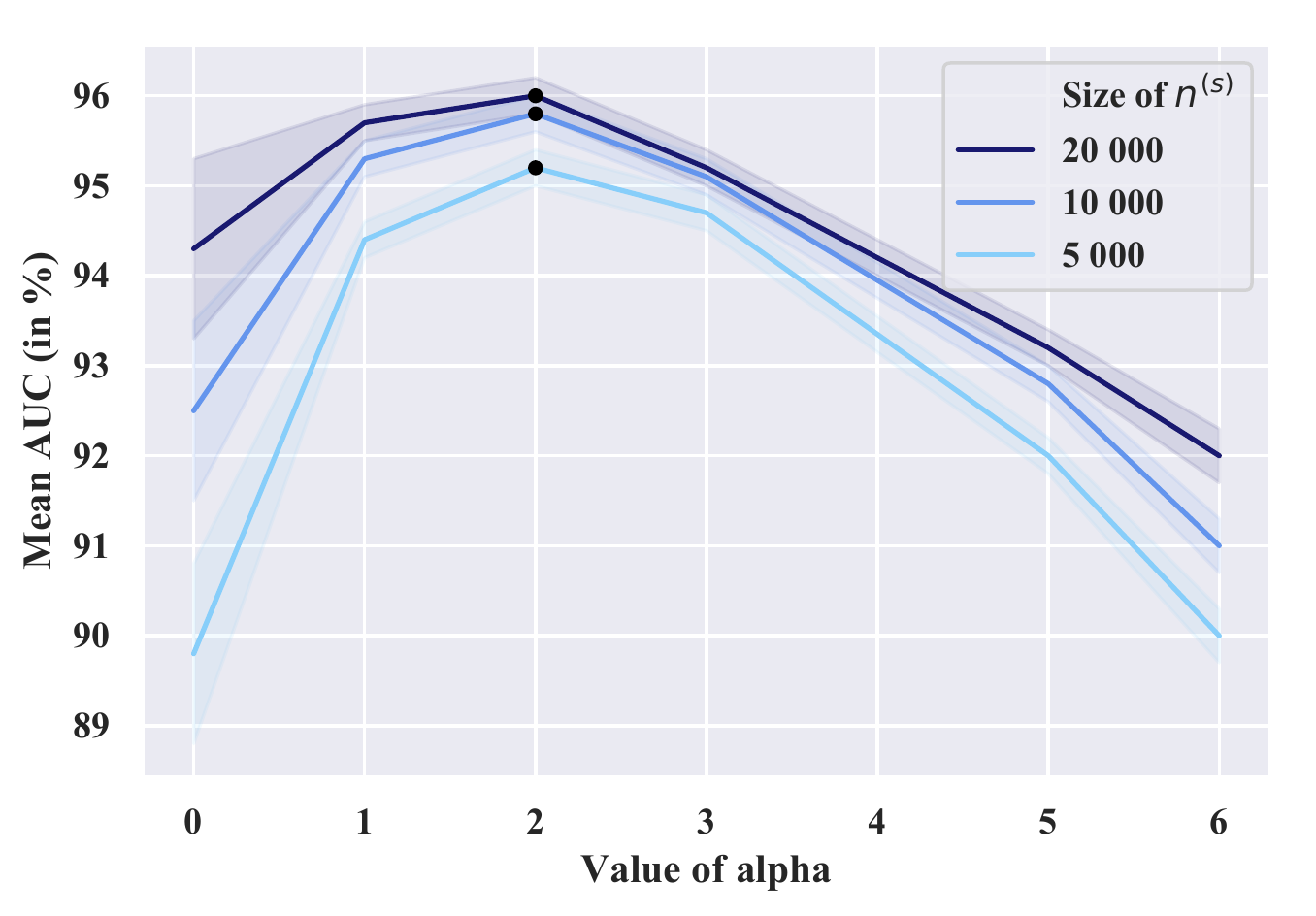}}}\subfigure[Youtube - Core Sampling]{
  \scalebox{0.3}{\includegraphics{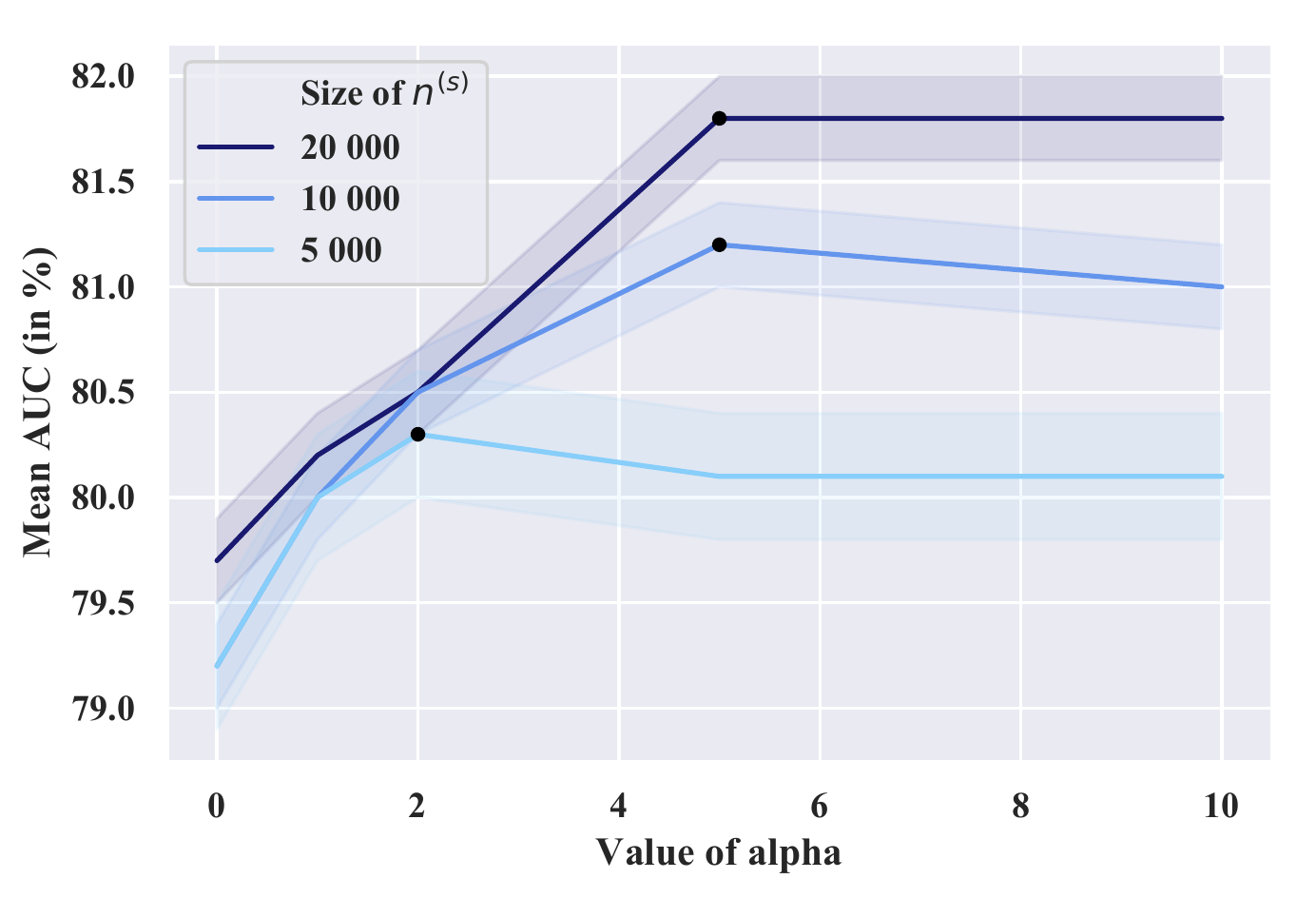}}}
    \subfigure[Patent - Degree Sampling]{
  \scalebox{0.3}{\includegraphics{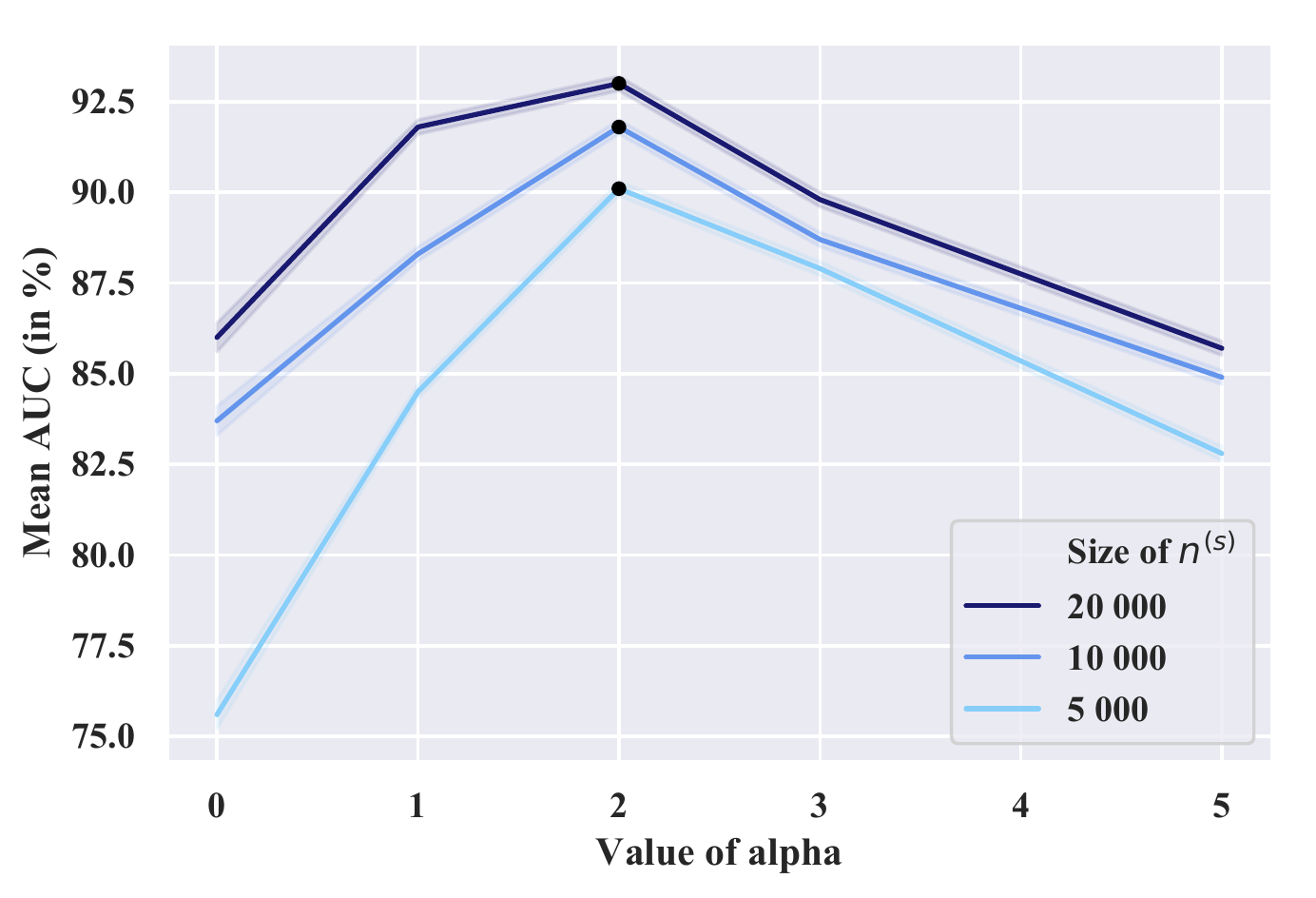}}}\subfigure[Patent - Core Sampling]{
  \scalebox{0.3}{\includegraphics{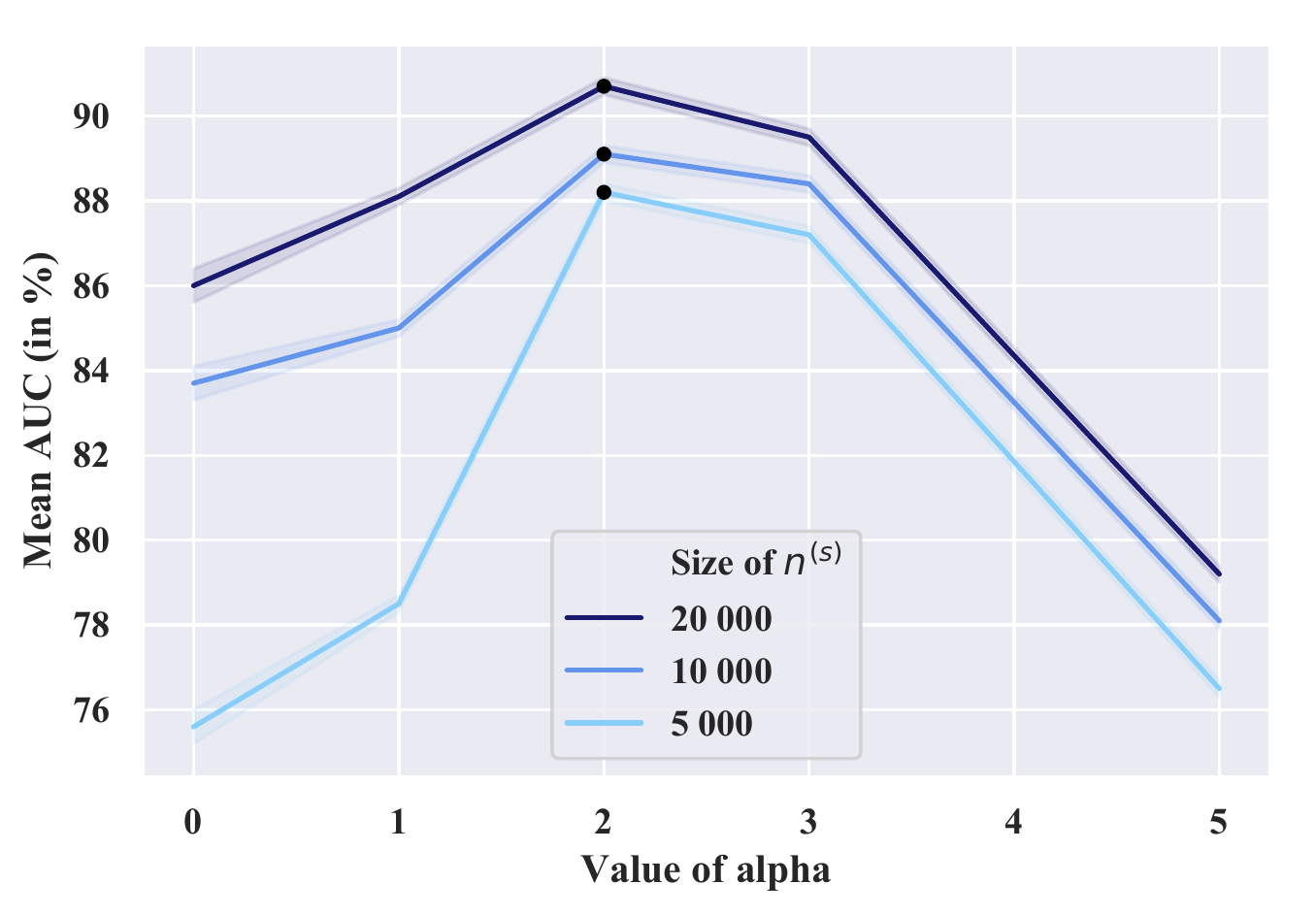}}}
  \caption{Optimal values of hyperparameter $\alpha$ for degree-based and core-based node sampling w.r.t. mean AUC scores on validation sets, for Variational FastGAE models and for all graphs.}
\end{figure*}

\clearpage



\bibliography{mybibfile}

\end{document}